%% file: dl_arxiv.tex
\documentclass[11pt]{article} 
\usepackage[top=1in, bottom=1in, left=1in, right=1in]{geometry}

\RequirePackage{natbib}
\setcitestyle{authoryear,round,citesep={;},aysep={,},yysep={;}}

\usepackage{graphicx,amsmath,amsfonts,amsthm,amscd,amssymb,bm,url,color,latexsym}

\renewenvironment{proof}[1][\proofname]{\noindent {\bfseries #1.}  }{\qed}
\usepackage{subcaption}
\usepackage{bbm}
\usepackage{microtype}
\usepackage{algorithm}
\usepackage{algorithmicx}
\usepackage{algpseudocode}
\usepackage{verbatim}
\usepackage{framed}
\usepackage{comment}
\usepackage{tikz}
\usepackage{fge}
\usepackage{float}
\usepackage{mathtools}
\usepackage{enumerate}
\usepackage{breqn}

\allowdisplaybreaks

\usepackage{hyperref}
\hypersetup{
    colorlinks=true,%
    citecolor=blue,%
    filecolor=blue,%
    linkcolor=blue,%
    urlcolor=blue
}
\usepackage[toc, page]{appendix}

\usepackage[capitalize,nameinlink]{cleveref}

\newtheorem{theorem}{Theorem}[section]
\newtheorem{lemma}[theorem]{Lemma}

\newtheorem{proposition}[theorem]{Proposition}
\newtheorem{definition}[theorem]{Definition}

\newtheorem{remark}[theorem]{Remark}

\newtheorem{fact}[theorem]{Fact}

\renewcommand{\mathbf}{\boldsymbol}

\newcommand{\mb}{\mathbf}
\newcommand{\mc}{\mathcal}

\newcommand{\bb}{\mathbb}

\newcommand{\set}[1]{\left\{ #1 \right\}}

\newcommand{\reals}{\bb R}

\newcommand{\downto}{\,\searrow\,}

\newcommand{\eps}{\varepsilon}
\newcommand{\R}{\reals}

\newcommand{\indicator}{\indic}

\newcommand{ \brac }[1]{\left[ #1 \right]}
\newcommand{ \paren }[1]{ \left( #1 \right) }

\newcommand{\objscale}{\sqrt{\frac{\pi}{2}}\cdot}

\newcommand{\what}{\widehat}
\newcommand{\rgrad}{\partial_R}

\DeclareMathOperator{\hd}{d_H}

\DeclareMathOperator{\supp}{supp}
\DeclareMathOperator{\conv}{conv}

\DeclareMathOperator{\sign}{sign}
\DeclareMathOperator{\grad}{grad}

\DeclareMathOperator{\mini}{minimize}

\DeclareMathOperator{\st}{subject\; to}
\DeclareMathOperator*{\argmax}{arg\,max}

\DeclareMathOperator{\radi}{rad}
\DeclareMathOperator{\dexp}{d_{\bb E}}
\DeclareMathOperator{\d2}{d_2}

\newcommand{\wt}{\widetilde}
\newcommand{\ol}{\overline}

\newcommand{\norm}[2]{\left\| #1 \right\|_{#2}}
\newcommand{\abs}[1]{\left| #1 \right|}
\newcommand{\innerprod}[2]{\left\langle #1,  #2 \right\rangle}
\newcommand{\prob}[1]{\bb P\left[ #1 \right]}
\newcommand{\expect}[1]{\bb E\left[ #1 \right]}

\newcommand{\iid}{iid}

\newcommand{\indic}[1]{\mathbbm 1\left\{#1\right\}} 
\newcommand{\normal}{\mathsf{N}}  
\newcommand{\wtilde}[1]{\widetilde{#1}}

\newcommand{\EOmega}{\bb E_{\Omega}}
\newcommand{\bgt}{\mathrm{BG}(\theta)}
\newcommand{\<}{\left\langle}
\renewcommand{\>}{\right\rangle}
\renewcommand{\dh}{\mathrm{d}}

\newcommand{\floor}[1]{\left\lfloor #1\right\rfloor}
\newcommand{\ceil}[1]{\left\lceil #1\right\rceil}

\numberwithin{equation}{section}

\title{Subgradient Descent Learns Orthogonal Dictionaries}

\author{Yu Bai\footnote{Department of Statistics, Stanford
    University. \texttt{yub@stanford.edu}.}
  \and
  Qijia Jiang\footnote{Department of Electrical Engineering, Stanford
    University. \texttt{qjiang2@stanford.edu}.}
  \and
  Ju Sun\footnote{Department of Mathematics, Stanford
    University. \texttt{sunju@stanford.edu}.}
}

\begin{document}
\maketitle
\input{Sections/abstract.tex}
\input{Sections/intro.tex}
\input{Sections/prelim.tex}
\input{Sections/result.tex}
\input{Sections/proof-highlight.tex}
\input{Sections/experiment.tex}
\input{Sections/discuss.tex}

\subsection*{Acknowledgement}
We thank Emmanuel Cand{\`e}s for helpful discussion. We thank Ruixue Wen for pointing out a minor bug in our initial argument, Zhihui Zhu and R{\' e}mi Gribonval for pointing us to a couple of missing references. JS thanks Yang Wang (HKUST) for motivating him to study nonsmooth nonconvex optimization problems while visiting ICERM in 2017.

\appendix
\input{Sections/tools.tex}
\input{Sections/proof-pop-geometry.tex}
\input{Sections/proof-emp-geometry.tex}
\input{Sections/proof-optimization.tex}
\input{Sections/proof-recovery.tex}
\input{Sections/aux.tex}

{\small
\bibliography{DL,NCVX,NSO,addref}
\bibliographystyle{iclr2019_conference}
}

\end{document}

%% file: Sections/abstract.tex
\begin{abstract}
  This paper concerns dictionary learning, i.e., sparse coding, a fundamental representation learning problem. We show that a subgradient descent algorithm, with random initialization, can provably recover orthogonal dictionaries on a natural nonsmooth, nonconvex $\ell_1$ minimization formulation of the problem, under mild statistical assumptions on the data. This is in contrast to previous provable methods that require either expensive computation or delicate initialization schemes. Our analysis develops several tools for characterizing landscapes of nonsmooth functions, which might be of independent interest for provable training of deep networks with nonsmooth activations (e.g., ReLU), among numerous other applications. Preliminary experiments corroborate our analysis and show that our algorithm works well empirically in recovering orthogonal dictionaries.
\end{abstract}

%% file: Sections/intro.tex
\section{Introduction}
Dictionary learning (DL), i.e., sparse coding, concerns the problem of learning compact representations: given data $\mb Y$, one tries to find a basis $\mb A$ and coefficients $\mb X$, so that $\mb Y \approx \mb A \mb X$ and $\mb X$ is most sparse. DL has found a variety of applications, especially in image processing and computer vision~\citep{MairalEtAl2014Sparse}. When posed in analytical form, DL seeks a transformation $\mb Q$ such that $\mb Q \mb Y$ is sparse; in this sense DL can be considered as an (extremely!) primitive ``deep'' network~\citep{RavishankarBresler2013Learning}.

Many heuristic algorithms have been proposed to solve DL since the seminal work of~\citet{OlshausenField1996Emergence}, and most of them are surprisingly effective in practice~\citep{MairalEtAl2014Sparse,SunEtAl2015Complete}. However, understandings on when and how DL can be solved with guarantees have started to emerge only recently. 
Under appropriate generating models on $\mb A$ and $\mb X$, \citet{SpielmanEtAl2012Exact} showed that complete (i.e., square, invertible) $\mb A$ can be recovered from $\mb Y$, provided that $\mb X$ is ultra-sparse. Subsequent works~\citep{AgarwalEtAl2017Clustering,AroraEtAl2014New,AroraEtAl2015Simple,ChatterjiBartlett2017Alternating,AwasthiVijayaraghavan2018Towards} provided similar guarantees for overcomplete (i.e. fat) $\mb A$, again in the ultra-sparse regime. The latter methods are invariably based on nonconvex optimization with model-dependent initialization, rendering their practicality on real data questionable.

The ensuing developments have focused on breaking the sparsity barrier and addressing the practicality issue. Convex relaxations based on the sum-of-squares (SOS) SDP hierarchy can recover overcomplete $\mb A$ when $\mb X$ has linear sparsity~\citep{BarakEtAl2015Dictionary,MaEtAl2016Polynomial,SchrammSteurer2017Fast}, while incurring expensive computation (solving large-scale SDP's or large-scale tensor decomposition problems). By contrast, \citet{SunEtAl2015Complete} showed that complete $\mb A$ can be recovered in the linear sparsity regime by solving a certain nonconvex problem with arbitrary initialization. However, the second-order optimization method proposed there is still expensive. This problem is partially addressed by~\citet{GilboaEtAl2018Efficient}, which proved that a first-order gradient descent algorithm with random initialization enjoys a similar performance guarantee.

A standing barrier toward practicality is dealing with nonsmooth functions. To promote sparsity, the $\ell_1$ norm is the function of choice in practical DL, as is common in modern signal processing and machine learning~\citep{Candes2014Mathematics}: despite its nonsmoothness, this choice often facilitates highly scalable numerical methods, such as proximal gradient method and alternating direction method~\citep{MairalEtAl2014Sparse}. The analyses in~\citet{SunEtAl2015Complete,GilboaEtAl2018Efficient}, however, focused on characterizing the function landscape of an alternative formulation of DL, which takes a smooth surrogate to $\ell_1$ to get around the nonsmoothness---due to the need of low-order derivatives. The tactic smoothing introduced substantial analysis difficulties, and broke the practical advantage of computing with the simple $\ell_1$ function.

\subsection{Our contribution}
In this paper, we show that working directly with a natural $\ell_1$
norm formulation results in neat analysis and a practical
algorithm. We focus on the problem of learning orthogonal
dictionaries: given data $\set{\mb y_i}_{i \in [m]}$ generated as
$\mb y_i = \mb A \mb x_i$, where $\mb A \in \R^{n \times n}$ is a
fixed unknown orthogonal matrix and each $\mb x_i \in \R^n$ is an iid
Bernoulli-Gaussian random vector with parameter $\theta \in (0, 1)$,
recover $\mb A$. This statistical model is the same as in previous
works~\citep{SpielmanEtAl2012Exact,SunEtAl2015Complete}.

Write $\mb Y \doteq [\mb y_1, \dots, \mb y_m]$ and similarly $\mb X \doteq [\mb x_1, \dots, \mb x_m]$. We propose recovering $\mb A$ by solving the following nonconvex (due to the constraint), nonsmooth (due to the objective) optimization problem:
\begin{align}
  \label{problem:nonsmooth-dl}
\mini_{\mb q \in \R^n} \; f(\mb q) \doteq \frac{1}{m}\norm{\mb q^\top
  \mb Y}{1} =
  \frac{1}{m}\sum_{i=1}^m |\mb q^\top \mb y_i| \quad \st \;
  \norm{\mb q}{2} = 1, 
\end{align}
which was first proposed in~\citet{jaillet2010l1}. Based on the above statistical model, $\mb q^\top \mb Y = \mb q^\top \mb A \mb X$ has the highest sparsity when $\mb q$ is a column of $\mb A$ (up to sign) so that $\mb q^\top \mb A$ is $1$-sparse. \citet{SpielmanEtAl2012Exact} formalized this intuition and optimized the same objective as~\cref{problem:nonsmooth-dl} with a $\norm{\mb q}{\infty} = 1$ constraint, which only works when $\theta \sim O(1/\sqrt{n})$. \citet{SunEtAl2015Complete, GilboaEtAl2018Efficient} worked with the same sphere constraint but replaced the $\ell_1$ objective with a smooth surrogate, entailing considerable analytical and computational deficiencies as alluded to above.

In contrast, we show that with sufficiently many samples, the optimization landscape of formulation (\ref{problem:nonsmooth-dl}) is benign with high probability (over the randomness of $\mb X$), and a simple Riemannian subgradient descent algorithm (see~\cref{section:prelim} and~\cref{section:result} for details on the Riemannian subgradients $\partial_R f\paren{\cdot}$)
\begin{align}
  \label{algorithm:rgd}
  \mb q^{(k+1)} = \frac{\mb q^{(k)} - \eta^{(k)} \mb v^{(k)} }{\norm{\mb q^{(k)} - \eta^{(k)} \mb v^{(k)}}{}}, \quad \text{for}\; \mb v^{(k)} \in \partial_R f\paren{\mb q^{(k)}}
\end{align}
can provably recover $\mb A$ in polynomial time.
\begin{theorem}[Main result, informal version of~\cref{theorem:main}]
  \label{theorem:informal}
  Assume $\theta \in [1/n, 1/2]$. For $m\ge \Omega(\theta^{-2} n^4 \log^3 n)$, the following holds with high probability:
  there exists a ${\rm poly}(m,\eps^{-1})$-time algorithm, which runs Riemannian subgradient descent
  on formulation (\ref{problem:nonsmooth-dl}) from at most $O(n\log n)$ independent, uniformly random
  initial points, and outputs a set of vectors
  $\set{\what{\mb a}_1,\dots,\what{\mb a}_n}$ such that up to
  permutation and sign change,
  $\norm{\what{\mb a_i} - \mb a_i}{2}\le \eps$ for all $i\in[n]$.
\end{theorem}
In words, our algorithm works also in the linear sparsity regime, the same as established in~\citet{SunEtAl2015Complete,GilboaEtAl2018Efficient}, but at a lower sample complexity $O(n^4)$ in contrast to the existing $O(n^{5.5})$ in~\citet{SunEtAl2015Complete}.\footnote{The sample complexity in~\citet{GilboaEtAl2018Efficient} is not explicitly stated.} As for the landscape, we show that (\cref{theorem:pop-geometry,theorem:emp-grad-inward}) each of the desired solutions $\set{\pm \mb a_i}_{i \in [n]}$ is a local minimizer of formulation (\ref{problem:nonsmooth-dl}) with a sufficiently large basin of attraction so that a random initialization will land into one of the basins with at least constant probability.

To obtain the result, we need to tame the nonconvexity and nonsmoothness in formulation (\ref{problem:nonsmooth-dl}). Whereas the general framework for analyzing this kind of problems is highly technical~\citep{RockafellarWets1998Variational}, our problem is much structured and lies in the locally Lipschitz function family. Calculus tools for studying first-order geometry (i.e., through first-order derivatives) of this family are mature and the technical level matches that of conventional smooth functions \citep{Clarke1990Optimization,Aubin1998Optima,BagirovEtAl2014Introduction,LedyaevZhu2007Nonsmooth,HosseiniPouryayevali2011Generalized}---which suffices for our purposes. Calculus for nonsmooth analysis requires set-valued analysis~\citep{AubinFrankowska2009Set}, and randomness in the data calls for random set theory~\citep{Molchanov2017Theory}. We show that despite the technical sophistication, analyzing locally Lipschitz functions only entails a minimal set of tools from these  domains, and the integration of these tools is almost painless.  Overall, we integrate and develop elements from nonsmooth analysis (on Riemannian manifolds), set-valued analysis, and random set theory for locally Lipschitz functions, and establish uniform convergence of subdifferential sets to their expectations for our problem, based on a novel covering argument (highlighted in~\cref{prop:egrad_concentrate_uniform}).  We believe that the relevant tools and ideas will be valuable to studying other nonconvex, nonsmooth optimization problems we encounter in practice---most of them are in fact locally Lipschitz continuous.

We formally present the problem setup and some preliminaries in~\cref{section:prelim}. The technical components of our main result is presented in~\cref{section:result}, including analyses of the population objective~(\cref{section:population-geometry}), the empirical objective~(\cref{section:empirical-geometry}), optimization guarantee for finding one basis~(\cref{section:opt-one-basis}), and recovery of the entire dictionary~(\cref{section:recovery}). In~\cref{section:proof-highlight}, we elaborate on the proof of concentration of subdifferential sets~(\cref{prop:egrad_concentrate_uniform}), which is our key technical challenge. We demonstrate our theory through simluations in~\cref{section:experiment}. Key technical tools from nonsmooth analysis, set-valued analysis, and random set theory are summarized in~\cref{sec:technical_tools}.

\subsection{Related work}
\label{sec:related_work}
\paragraph{Dictionary learning}
Besides the relevance to the many results sampled above, we highlight similarities of our result and analysis to~\citet{GilboaEtAl2018Efficient}. Both propose first-order optimization methods with random initialization, and several key quantities in the proofs are similar. A defining difference is that we work with the nonsmooth $\ell_1$ objective directly, while~\citet{GilboaEtAl2018Efficient} built on the smoothed objective from~\citet{SunEtAl2015Complete}. We put considerable emphasis on practicality: the subgradient of the nonsmooth objective is considerably cheaper to evaluate than that of the smooth objective in~\citet{SunEtAl2015Complete}, and in the algorithm we use Euclidean projection rather than exponential mapping to remain feasible---again, the former is much lighter for computation.

\paragraph{General nonsmooth analysis}
While by now nonsmooth analytic tools such as subdifferentials for convex functions are well received in the machine learning and relevant communities, that for general functions are much less so. The Clarke subdifferential and relevant calculus developed for locally Lipschitz functions are particularly relevant for problems in these areas, and cover several families of functions of interest, such as convex functions, continuously differentiable functions, and many forms of composition \citep{Clarke1990Optimization,Aubin1998Optima,BagirovEtAl2014Introduction,Schirotzek2007Nonsmooth}. Remarkably, majority of the tools and results can be generalized to locally Lipschitz functions on Riemannnian manifolds~\citep{LedyaevZhu2007Nonsmooth,HosseiniPouryayevali2011Generalized}. Our formulation (\ref{problem:nonsmooth-dl}) is exactly optimization of a locally Lipschitz function on a Riemannian manifold (i.e., the sphere). For simplicity, we try to avoid the full manifold language, nonetheless.

\paragraph{Nonsmooth optimization on Riemannian manifolds or with constraints}
Equally remarkable is many of the smooth optimization techniques and convergence results can be naturally adapted to optimization of locally Lipschitz functions on Riemannian manifolds~\citep{GrohsHosseini2015Nonsmooth,Hosseini2015Optimality,HosseiniUschmajew2017Riemannian,GrohsHosseini2016Subgradient,de2018newton,chen2012smoothing}. New optimization methods such as gradient sampling and variants have been invented to solve general nonsmooth problems~\citep{BurkeEtAl2005robust,BurkeEtAl2018Gradient,BagirovEtAl2014Introduction,CurtisQue2015quasi,CurtisEtAl2017BFGS}. Almost all available convergence results in these lines pertain to only global convergence, which is too weak for our purposes. Our specific convergence analysis gives us a precise local convergence result with a rate estimate (\cref{theorem:opt-one-basis}).

\paragraph{Provable solutions to nonsmooth, nonconvex problems}
In modern machine learning and high-dimensional data analysis, nonsmooth functions are often used to promote structures (sparsity, low-rankness, quantization) or achieve robustness. Provable solutions to nonsmooth, nonconvex formulations of a number of problems exist, including empirical risk minimization with nonsmooth regularizers~\citep{Loh2015Regularized}, structured element pursuit~\citep{QuEtAl2014Finding}, generalized phase retrieval~\citep{WangEtAl2016Solving,ZhangEtAl2016Reshaped,Soltanolkotabi2017Structured,DuchiRuan2017Solving, DavisEtAl2017nonsmooth}, convolutional phase retrieval~\citep{QuEtAl2017Convolutional}, robust subspace estimation
\citep{MaunuEtAl2017Well,ZhangYang2017Robust,zhu2018dual}, robust matrix recovery
\citep{GeEtAl2017No,LiEtAl2018Nonconvexa}, and estimation under deep generative models~\citep{HandVoroninski2017Global,HeckelEtAl2018Deep}. In deep network training, when the activation functions are nonsmooth, e.g., the popular ReLU~\citep{GlorotEtAl2011Deep}, the resulting optimization problems are also nonsmooth and nonconvex. There have been intensive research efforts toward provable training of ReLU neural networks with one hidden layer
 \citep{Soltanolkotabi2017Learning,LiYuan2017Convergence,ZhongEtAl2017Recovery,DuGoel2018Improved,Oymak2018,DuEtAl2018Gradient,LiLiang2018Learning,JagatapHegde2018Learning,ZhangEtAl2018Learning,DuEtAl2017Gradient,ZhongEtAl2017Learning,ZhongEtAl2017Recovery,ZhouLiang2017Critical,BrutzkusEtAl2017SGD,DuEtAl2017When,BrutzkusGloberson2017Globally,WangEtAl2018Learning,YunEtAl2018Critical,LaurentBrecht2017Multilinear,MeiEtAl2018Mean,GeEtAl2017Learning}---too many for us to be exhaustive here.

Most of the above results depend on either problem-specific initialization plus local function landscape characterization, or algorithm-specific analysis. In comparison, for the orthogonal dictionary learning problem, our result provides an algorithm-independent characterization of the landscape of the nonsmooth, nonconvex formulation (\ref{problem:nonsmooth-dl}), and is  ``almost global'' in the sense that the region we characterize has a constant measure over the sphere $\bb S^{n-1}$ so that uniformly random initialization lands into the region with a constant probability.

With few exceptions, the majority of the above works are technically vague about nonsmooth points. They either prescribe a
subgradient element for non-differentiable points, or ignore the non-differentiable points altogether. The former is problematic, as non-differentiable points often cannot be reliably tested numerically (e.g., $\mb 0$), leading to inconsistency between theory and practice. The latter could be fatal, as non-differentiable points could be points of interest,\footnote{...although it is true that for locally Lipschitz functions on a bounded set in $\R^d$, the (Lebesgue) measure of non-differentiable points is zero. } e.g., local minimizers---see the case made by~\citet{LaurentBrecht2017Multilinear} for two-layer ReLU networks. Also, assuming the usual calculus rules of smooth functions for nonsmooth ones results in uncertainty in computation~\citep{kakade2018provably}. Our result is grounded on the rich and established toolkit of nonsmooth analysis (see~\cref{sec:technical_tools}), and we concur with~\citet{LaurentBrecht2017Multilinear} that nonsmooth analysis is an indispensable technical framework for solving nonsmooth, nonconvex problems.  test test 

%% file: Sections/prelim.tex
\section{Preliminaries}
\label{section:prelim}
\paragraph{Problem setup}
Given an unknown orthogonal dictionary $\mb A=[\mb a_1,\dots,\mb a_n]\in\R^{n\times n}$, we wish to recover $\mb A$ through $m$ observations of the form
\begin{equation}
  \mb y_i = \mb A\mb x_i,
\end{equation}
or $\mb Y=\mb A\mb X$ in matrix form, where $\mb X=[\mb x_1,\dots,\mb x_m]$ and $\mb Y=[\mb y_1,\dots,\mb y_m]$.

The coefficient vectors $\mb x_i$ are sampled from the iid Bernoulli-Gaussian distribution with parameter $\theta\in(0,1)$: each entry $x_{ij}$ is independently drawn from a standard Gaussian with probability $\theta$ and zero otherwise; we write $\mb x_i \sim_{iid} \bgt$. The Bernoulli-Gaussian model is a good prototype distribution for sparse vectors, as $\mb x_i$ will be on average $\theta$-sparse.

We assume that $n\ge 3$ and $\theta\in[1/n,1/2]$. In particular, $\theta\ge 1/n$ is to require that each $\mb x_i$ has at least one non-zero entry on average.

\paragraph{First-order geometry}
We will focus on the first-order geometry of the non-smooth objective~\cref{problem:nonsmooth-dl}: $f(\mb q)=\frac{1}{m}\sum_{i=1}^m|\mb q^\top\mb y_i|$. In the whole Euclidean space $\R^n$, $f$ is convex with subdifferential set
\begin{equation}
  \partial f(\mb q) = \frac{1}{m}\sum_{i=1}^m \sign(\mb q^\top \mb y_i)\mb y_i,
\end{equation}
where $\sign(\cdot)$ is the set-valued sign function (i.e. $\sign(0)=[-1,1]$).
As we minimize $f$ subject to the constraint $\norm{\mb q}{2}=1$, our problem is no longer convex. Obviously, $f$ is locally Lipschitz on $\bb S^{n-1}$ wrt the angle metric. By~\cref{def:clarke_subd_manifold}, the Riemannian subdifferential of $f$ on $\bb S^{n-1}$ is~\citep{HosseiniUschmajew2017Riemannian}:
\begin{equation}
  \rgrad f(\mb q) = (\mb I - \mb q\mb q^\top)\partial f(\mb q).
\end{equation}
We will use the term subgradient (or Riemannian subgradient) to denote an arbitrary element in the subdifferential (or the Riemannian subdifferential).

A point $\mb q$ is stationary for problem~\cref{problem:nonsmooth-dl} if $\mb 0\in\rgrad f(\mb q)$. We will not distinguish between local maximizers and saddle points---we call a stationary point $\mb q$ a saddle point if there is a descent direction (i.e. direction along which the function is locally maximized at $\mb q$). Therefore, any stationary point is either a local minimizer or a saddle point.

\paragraph{Set-valued analysis}
As the subdifferential is a set-valued mapping, analyzing it requires some set-valued analysis. The addition of two sets is defined as the Minkowski summation: $X+Y \doteq \set{x+y:x\in X,y\in Y}$. Due to the randomness in the data, the subdifferentials in our problem are random sets, and our analysis involves deriving expectations of random sets and studying behaviors of iid Minkowski summations of random sets. The concrete definition and properties of the expectation of random sets are presented in~\cref{appendix:set-expectation}. We use the Hausdorff distance
\begin{align}
 \hd\paren{X_1, X_2} \doteq \max\set{\sup_{\mb x_1 \in X_1} \dh\paren{\mb x_1, X_2}, \sup_{\mb x_2 \in X_2} \dh\paren{\mb x_2, X_1}}.
\end{align}
to measure distance between sets. Basic properties of the Hausdorff distance are provided in~\cref{appendix:hausdorff}.

\paragraph{Notation}
Bold small letters (e.g., $\mb x$) are vectors and bold capitals are matrices (e.g., $\mb X$). The dotted equality $\doteq$ is for definition. For any positive integer $k$, $[k] \doteq \set{1, \dots, k}$. By default, $\norm{\cdot}{}$ is the $\ell_2$ norm if applied to a vector, and the operator norm if applied to a matrix. The vector $\mb x_{-i}$ is $\mb x$ with the $i$-th coordinate removed. The letters $C$ and $c$ or any indexed versions are reserved for universal constants that may change from line to line. From now on, we use $\Omega$ exclusively to denote the generic support set of iid $\mathrm{Ber}\paren{\theta}$ law, the dimension of which should be clear from context. 

%% file: Sections/result.tex
\section{Main Result}
\label{section:result}
We now state our main result, the recovery guarantee for learning orthogonal dictionary  by solving formulation (\ref{problem:nonsmooth-dl}).
\begin{theorem}[Recovering orthogonal dictionary via subgradient descent]
  \label{theorem:main}
  Suppose we observe
  \begin{equation}
    m \ge Cn^4 \theta^{-2} \log^3 n
  \end{equation}
  samples in the dictionary learning problem and we desire an accuracy $\eps \in (0, 1)$ for recovering the dictionary. With probability at least $1-\exp\paren{-cm \theta^3 n^{-3} \log^{-3} m}- \exp\paren{-c'R/n}$, if we run the Riemannian subgradient descent~\cref{algorithm:rgd} with step sizes $\eta^{(k)} = k^{-3/8}/\paren{100\sqrt{n}}$ and repeat $R\ge C'n\log n$ times with independent random initializations on $\bb S^{n-1}$, we will obtain a set of vectors $\set{\what{\mb a}_1,\dots,\what{\mb a}_n}$ such that up to permutation and sign change, $\norm{\what{\mb a_i} - \mb a_i}{2}\le \eps$ for all $i\in[n]$. The total number of subgradient descent iterations is bounded by
  \begin{equation}
  C''R \theta^{-16/3} \eps^{-8/3} n^4 \log^{8/3} n.
  \end{equation}
  Here $C, C', C'', c, c' > 0$ are all universal constants.
\end{theorem}
At a high level, the proof of~\cref{theorem:main} consists of the following steps, which we elaborate throughout the rest of this section.
\begin{enumerate}
\item Partition the sphere into $2n$ symmetric ``good sets'' and show certain directional subgradient is strong on population objective $\expect{f}$ inside the good sets (\cref{section:population-geometry}).
\item Show that the same geometric properties carry over to the empirical objective $f$ with high probability. This involves proving the uniform convergence of the subdifferential set $\partial f$ to $\expect{\partial f}$ (\cref{section:empirical-geometry}).
\item Under the benign geometry, establish the convergence of Riemannian subgradient descent to one of $\set{\pm \mb a_i:i\in[n]}$ when initialized in the corresponding ``good set'' (\cref{section:opt-one-basis}).
\item Calling the randomly initialized optimization procedure $O(n\log n)$ times will recover all of $\set{\mb a_1,\dots, \mb a_n}$ with high probability, by a coupon collector's argument (\cref{section:recovery}).
\end{enumerate}

\paragraph{Scaling and rotating to identity}
Throughout the rest of this paper, we are going to assume wlog that the dictionary is the identity matrix, i.e., $\mb A=\mb I_n$, so that $\mb Y=\mb X$, $f(\mb q)=\norm{\mb q^\top \mb X}{1}$, and the goal is to find the standard basis vectors $\set{\pm\mb e_1,\dots,\pm\mb e_n}$. The case of a general orthogonal $\mb A$ can be reduced to this special case via rotating by $\mb A^\top$: $\mb q^\top \mb Y = \mb q^\top \mb A\mb X=(\mb q')^\top \mb X$ where $\mb q'=\mb A^\top\mb q$, and applying the result on $\mb q'$, as feasibility remains intact, i.e., $\norm{\mb q}{} = \norm{\mb q'}{} = 1$. We also scale the objective by $\sqrt{\pi/2}$ for convenience of later analysis.

\subsection{Properties of the population objective}
\label{section:population-geometry}
We begin by characterizing the geometry of the expected objective $\expect{f}$. Recall that we have rotated $A$ to be identity, so that we have
\begin{equation}
  f(\mb q) = \objscale \frac{1}{m}\norm{\mb q^\top\mb X}{1} = \objscale\frac{1}{m}\sum_{i=1}^m \abs{\mb q^\top\mb x_i},~~~\partial f(\mb q) = \objscale \frac{1}{m}\sum_{i=1}^m\sign\paren{\mb q^\top\mb x_i}\mb x_i.
\end{equation}
\paragraph{Minimizers and saddles of the population objective}
We begin by computing the function value and subdifferential set of
the population objective and giving a complete characterization of its
stationary points, i.e. local minimizers and saddles.
\begin{proposition}[Population objective value and gradient]
  \label{proposition:pop-val-grad}
  We have
  \begin{align}
    & \expect{f}(\mb q) = \objscale\expect{\abs{\mb q^\top\mb x}} = \EOmega{\norm{\mb q_\Omega}{}}\\
    & \partial\expect{f}(\mb q) = \expect{\partial f}(\mb q) = \objscale\expect{\sign\paren{\mb q^\top\mb x}\mb x} = \bb E_\Omega \left\{
      \begin{aligned}
        &\mb q_\Omega/\norm{\mb q_\Omega}{}, & \mb q_\Omega \neq 0, \\
        &\set{\mb v_\Omega: \norm{\mb v_\Omega}{} \le 1}, & \mb q_\Omega = 0.
      \end{aligned}
      \right. \label{equation:pop-grad}
  \end{align}
\end{proposition}
\begin{proposition}[Stationary points]
  \label{proposition:pop-stat-points}
  The stationary points of $\expect{f}$ on the sphere are
  \begin{equation}
    \mc{S} = \set{\frac{1}{\sqrt{k}}\mb q: \mb q\in\set{-1,0,1}^{n},\norm{\mb q}{0}=k,~k\in[n]}.
  \end{equation}
  The case $k=1$ corresponds to the $2n$ global minimizers $\mb q=\pm\mb e_i$, and all other values of $k$ correspond to saddle points.
\end{proposition}
A consequence of~\cref{proposition:pop-stat-points} is that though the problem itself is non-convex (due to the constraint), the population objective has no ``spurious local minima'': each stationary point is either a \emph{global} minimizer or a saddle point.

{\bf Identifying $2n$ ``good'' subsets}
We now define $2n$ disjoint subsets on the sphere, each containing one of the global minimizers $\set{\pm\mb e_i}$ and possessing benign geometry for both the population and empirical objective, following~\citep{GilboaEtAl2018Efficient}.
For any $\zeta\in[0,\infty)$ and $i\in[n]$ define
\begin{align}
  \mc S_{\zeta}^{(i+)} &\doteq \set{\mb q \in \bb S^{n-1}: q_i>0,\,  \frac{q_i^2}{\norm{\mb q_{-i}}{\infty}^2} \ge 1 + \zeta},\\
  \mc S_{\zeta}^{(i-)} &\doteq \set{\mb q \in \bb S^{n-1}: q_i<0,\,  \frac{q_i^2}{\norm{\mb q_{-i}}{\infty}^2} \ge 1 + \zeta}.
\end{align}
For points in $\mc{S}_\zeta^{(i+)}\cup\mc{S}_{\zeta}^{(i-)}$, the $i$-th index is larger than all other indices (in absolute value) by a multiplicative factor of $\zeta$. In particular, for any point in these subsets, the largest index (in absolute value) is unique, so by~\cref{proposition:pop-stat-points} all population saddle points are excluded from these $2n$ subsets. Each set contains exactly one of the $2n$ global minimizers.

Intuitively, this partition can serve as a ``tiebreaker'': points in $\mc{S}_{\zeta}^{(i+)}$ is closer to $\mb e_i$ than all the other $2n-1$ signed basis vectors. Therefore, we hope that optimization algorithms initialized in this region could favor $\mb e_i$ over the other standard basis vectors, which is indeed the case as we are going to show. For simplicity, we are going to state our geometry results in $\mc{S}_{\zeta}^{(n+)}$; by symmetry the results will automatically carry over to all the other $2n-1$ subsets.


\begin{theorem}[Lower bound on directional subgradients]
  \label{theorem:pop-geometry}
  Fix any $\zeta_0\in(0,1)$. We have
  \begin{enumerate}[(a)]
  \item For all $\mb q \in \mc{S}_{\zeta_0}^{(n+)}$ and all indices $j\neq n$ such that $q_j\neq 0$,
    \begin{align}
      \label{equation:pop-grad-inward}
      \inf \innerprod{\expect{\partial_R f}\paren{\mb q}}{\frac{1}{q_j} \mb e_j - \frac{1}{q_n} \mb e_n} \ge \frac{1}{2n}\theta\paren{1-\theta} \frac{\zeta_0}{1+\zeta_0};
    \end{align}
  \item For all $\mb q \in \mc{S}_{\zeta_0}^{(n+)}$, we have that
    \begin{equation}
      \label{equation:pop-grad-star}
      \inf \innerprod{\expect{\rgrad f}\paren{\mb q}}{ q_n \mb q - \mb e_n} \ge \frac{1}{8} \theta \paren{1-\theta}\zeta_0 n^{-3/2} \norm{\mb q_{-n}}{}.
    \end{equation}
  \end{enumerate}
\end{theorem}
These lower bounds verify our intuition: points inside $\mc{S}_{\zeta_0}^{(n+)}$ have negative subgradients pointing toward $\mb e_n$, both in a coordinate-wise sense (\cref{equation:pop-grad-inward}) and a combined sense (\cref{equation:pop-grad-star}): the direction $\mb e_n - q_n \mb q$ is exactly the tangent direction of the sphere at $\mb q$ that points toward $\mb e_n$.

\subsection{Benign geometry of the empirical objective}
\label{section:empirical-geometry}
We now show that the benign geometry in~\cref{theorem:pop-geometry} is carried onto the empirical objective $f$ given sufficiently many samples, using a concentration argument. The key result behind is the concentration of the empirical subdifferential to the population subdifferential, where concentration is measured in the Hausdorff distance between sets.
\begin{proposition}[Uniform convergence of subdifferential]
  \label{prop:egrad_concentrate_uniform}
  For any $t \in (0, 1)$, when
  \begin{align}
    m \ge Ct^{-2}n\log (n/t),
  \end{align}
  with probability at least $1 - \exp\paren{-cm \theta t^2/\log m}$, we have
  \begin{align}
    \hd\paren{\partial f\paren{\mb q}, \expect{\partial f}\paren{\mb q}} \le t~~~\textrm{for all}~\mb q\in\bb S^{n-1}.
  \end{align}
  Here $C,c\ge 0$ are universal constants.
\end{proposition}
The concentration result guarantees that the subdifferential set is close to its expectation given sufficiently many samples with high probability. Choosing an appropriate concentration level $t$, the lower bounds on the directional subgradients carry over to the empirical objective $f$, which we state in the following theorem.
\begin{theorem}[Directional subgradient lower bound, empirical objective]
  \label{theorem:emp-grad-inward}
  There exist universal constants $C, c \ge 0$ so that the following holds: for all $\zeta_0\in(0, 1)$, when $m \ge C n^4 \theta^{-2} \zeta_0^{-2} \log \paren{n/\zeta_0}$, with probability at least $1 - \exp\paren{-c m \theta^3 \zeta_0^2n^{-3} \log^{-1} m}$, the below estimates hold simultaneously for all the $2n$ subsets $\set{\mc{S}_{\zeta_0}^{(i+)}, \mc{S}_{\zeta_0}^{(i-)}:i\in[n]}$: (stated only for $\mc{S}_{\zeta_0}^{(n+)}$)
  \begin{enumerate}[(a)]
  \item For all $\mb q \in \mc{S}_{\zeta_0}^{(n+)}$ and all $j \in [n]$ with $q_j \ne 0$ and $q_n^2 /q_j^2 \le 3$,
    \begin{align}
      \label{equation:emp-grad-inward}
      \inf \innerprod{\partial_R f\paren{\mb q}}{\frac{1}{q_j} \mb e_j - \frac{1}{q_n} \mb e_n}
      \ge \frac{1}{4n}\theta\paren{1-\theta} \frac{\zeta_0}{1+\zeta_0};
    \end{align}
  \item For all $\mb q\in\mc{S}_{\zeta_0}^{(n+)}$,
    \begin{align}
      \label{equation:emp-grad-star}
      \inf \innerprod{\partial_R f\paren{\mb q}}{q_n \mb q - \mb e_n}
      \ge \frac{\sqrt{2}}{16} \theta \paren{1-\theta} n^{-\frac{3}{2}} \zeta_0 \norm{\mb q_{-n}}{} \ge  \frac{1}{16} \theta \paren{1-\theta} n^{-\frac{3}{2}} \zeta_0 \norm{\mb q - \mb e_n}{}.
    \end{align}
  \end{enumerate}
\end{theorem}
The consequence of~\cref{theorem:emp-grad-inward} is two-fold. First, it guarantees that the only possible stationary point of $f$ in $\mc{S}_{\zeta_0}^{(n+)}$ is $\mb e_n$: for every other point $\mb q\neq \mb e_n$, property (b) guarantees that $0\notin \rgrad f(\mb q)$, therefore $\mb q$ is non-stationary. Second, the directional subgradient lower bounds allow us to establish convergence of the Riemannian subgradient descent algorithm, in a way similar to showing convergence of unconstrained gradient descent on star strongly convex functions.

We now present an upper bound on the norm of the subdifferential sets, which is needed for the convergence analysis.
\begin{proposition}
  \label{prop:grad_norm_concentrate}
  There exist universal constants $C,c\ge 0$ such that
  \begin{align}
    \sup \norm{\partial f\paren{\mb q}}{} \le 2 \quad \forall\;  \mb q \in \bb S^{n-1}
  \end{align}
  with probability at least $1 - \exp\paren{-cm\theta \log^{-1} m}$, provided that $m \ge C n \log n$. This particularly implies that
  \begin{align}
    \sup \norm{\partial_R f\paren{\mb q}}{} \le 2  \quad \forall\;  \mb q \in \bb S^{n-1}.
\end{align}
\end{proposition}

\subsection{Finding one basis via Riemannian subgradient descent}
\label{section:opt-one-basis}
The benign geometry of the empirical objective allows a simple Riemannian subgradient descent algorithm to find one basis vector a time. The Riemannian subgradient descent algorithm with initialization $\mb q^{(0)}$ and step size $\set{\eta^{(k)}}_{k\ge 0}$ proceeds as follows: choose an arbitrary Riemannian subgradient $\mb v^{(k)} \in \partial_R f\paren{\mb q^{(k)}}$, and iterate
\begin{align}
  \mb q^{(k+1)} = \frac{\mb q^{(k)} - \eta^{(k)} \mb v^{(k)} }{\norm{\mb q^{(k)} - \eta^{(k)} \mb v^{(k)}}{}}, \quad \text{for}\; k = 0, 1, 2, \dots  \; .
\end{align}
Each iteration moves in an arbitrary Riemannian subgradient direction followed by a projection back onto the sphere.\footnote{The projection is a retraction in the Riemannian optimization framework~\citep{AbsilEtAl2008Optimization}.  It is possible to use other retractions such as the canonical exponential map, resulting in a messier analysis problem for our case. } We show that the algorithm is guaranteed to find one basis as long as the initialization is in the ``right'' region. To give a concrete result, we set $\zeta_0 = 1/(5\log n)$.\footnote{It is possible to set $\zeta_0$ to other values, inducing different combinations of the final sample complexity, iteration complexity, and repetition complexity in~\cref{theorem:opt-all-bases}. }

\begin{theorem}[One run of subgradient descent recovers one basis]
  \label{theorem:opt-one-basis}
Let $m \ge C\theta^{-2}n^4\log^3 n$ and $\eps \in (0, 2\theta/25]$. With probability at least $1 - \exp\paren{-cm\theta^3 n^{-3} \log^{-3} m}$ the following happens.  If the initialization $\mb q^{(0)}\in\mc{S}_{1/(5 \log n)}^{(n+)}$, and we run the projected Riemannian subgradient descent with step size $\eta^{(k)} = k^{-\alpha}/\paren{100\sqrt{n}}$ with $\alpha\in \paren{0, 1/2}$, and keep track of the best function value so far until after the iteration $K$ is performed (i.e., stopping at $\mb q^{(K+1)}$), producing $\mb q^{\mathrm{best}}$. Then, $\mb q^{\mathrm{best}}$ obeys
  \begin{align}
    f\paren{\mb q^{\mathrm{best}}} - f\paren{\mb e_n}
    & \le \eps, \quad \text{and} \quad
      \norm{\mb q^{\mathrm{best}} - \mb e_n}{} \le \frac{16 }{\theta\paren{1-\theta}} \eps,
  \end{align}
  provided that
  \begin{align}
    K \ge \max\set{\paren{\frac{32000n^{5/2} \log n \paren{1-\alpha}}{\theta \paren{1-\theta} \eps} }^{\frac{1}{1-\alpha}}, \paren{\frac{64n^{3/2} \log n \frac{1-\alpha}{1-2\alpha}}{5\eps \theta \paren{1-\theta} }}^{\frac{1}{\alpha}}}
  \end{align}
  In particular, choosing $\alpha=3/8 < 1/2$, it suffices to let
  \begin{equation}
    \label{equation:K-one-third}
    K \ge K_{3/8} \doteq C'\theta^{-8/3} \eps^{-8/3} n^4 \log^{8/3} n.
  \end{equation}
  Here $C, C', c\ge 0$ are universal constants.
\end{theorem}
\begin{lemma}[Random initialization falls in ``good set'']
  \label{lemma:random-init-works}
  Let $\mb q^{(0)}\sim{\rm Uniform}(\bb S^{n-1})$, then with probability at least $1/2$, $\mb q^{(0)}$ belongs to one of the $2n$ sets $\set{\mc{S}_{1/(5\log n)}^{(i+)}, \mc{S}_{1/(5\log n)}^{(i-)}:i\in[n]}$, and the set it belongs to is uniformly at random.
\end{lemma}

\subsection{Recovering all bases from multiple runs}
\label{section:recovery}

As long as the initialization belongs to $\mc{S}_{1/(5\log n)}^{(i+)}$ or $\mc{S}_{1/(5\log n)}^{(i-)}$, our finding-one-basis result in~\cref{theorem:opt-one-basis} guarantees that Riemannian subgradient descent will converge to $\mb e_i$ or $-\mb e_i$ respectively. Therefore if we run the algorithm with independent, uniformly random initializations on the sphere multiple times, by a coupon collector's argument, we will recover all the basis vectors. This is formalized in the following theorem.
\begin{theorem}[Recovering the identity dictionary from multiple random initializations]
  \label{theorem:opt-all-bases}
 Let $m \ge C\theta^{-2} n^4\log^3 n$ and $\eps \in (0, 1)$, with probability at least $1-\exp\paren{-cm \theta^3 n^{-3} \log^{-3} m}$ the following happens. Suppose we run the Riemannian subgradient descent algorithm independently for $R$ times, each with a uniformly random initialization on $\bb S^{n-1}$, and choose the step size as $\eta^{(k)}=k^{-3/8}/(100\sqrt{n})$. Then, provided that $R \ge C' n \log n$, all standard basis vectors will be recovered up to $\eps$ accuracy with probability at least $1 - \exp\paren{-cR/n}$ in a total of $C' R\paren{\theta^{-16/3} \eps^{-8/3} n^4 \log^{8/3}n}$ iterations. Here $C, C', c \ge 0$ are universal constants.
\end{theorem}
When the dictionary $\mb A$ is not the identity matrix, we can apply
the rotation argument sketched in the beginning of this section to get
the same result, which leads to our main result in~\cref{theorem:main}.

%% file: Sections/proof-highlight.tex
\section{Proof Highlights}
\label{section:proof-highlight}
A key technical challenge is establishing the uniform convergence of subdifferential sets in~\cref{prop:egrad_concentrate_uniform}, which we now elaborate. Recall that the population and empirical subdifferentials are
\begin{equation}
 \partial f(\mb q) = \objscale \frac{1}{m}\sum_{i=1}^m \sign(\mb q^\top\mb x_i)\mb x_i,~~~\expect{\partial f}(\mb q) = \objscale \bb E_{\mb x\sim\bgt}\brac{\sign(\mb q^\top\mb x)\mb x},
\end{equation}
and we wish to show that the difference between $\partial f(\mb q)$ and $\expect{\partial f}(\mb q)$ is small uniformly over $\mb q\in\mc{Q}=\bb S^{n-1}$. Two challenges stand out in showing such a uniform convergence:
\begin{enumerate}
\item The subdifferential is set-valued and random, and it is unclear a-priori how one could formulate and analyze the concentration of random sets.
\item The usual covering argument will not work here, as the Lipschitz gradient property does not hold: $\partial f(\mb q)$ and $\expect{\partial f}(\mb q)$ are not Lipschitz in $\mb q$ wrt the Euclidean metric nor the angle metric. 
\end{enumerate}

\subsection{Concentration of random sets}
\label{section:set-concentration}
We state and analyze concentration of random sets in the Hausdorff
distance (defined in~\cref{section:prelim}). We now illustrate why the
Hausdorff distance is the ``right'' distance to consider for
concentration of subdifferentials---the reason is that the Hausdorff
distance is closely related to the \emph{support function} of sets,
which for any set $S\in\bb R^n$ is defined as
\begin{equation}
  h_S(\mb u)\doteq \sup_{\mb x\in S}\innerprod{\mb x}{\mb u}.
\end{equation}
For convex compact sets, the sup difference between their support
functions is exactly the Hausdorff distance.
\begin{lemma}[Section 1.3.2,~\cite{Molchanov2013foundations}]
  \label{lemma:dist-support}
  For convex compact sets $X,Y\subset\R^n$, we have
  \begin{equation}
    \hd\paren{X, Y} = \sup_{\mb u\in\bb S^{n-1}} \abs{h_X(\mb u) - h_Y(\mb u)}.
  \end{equation}
\end{lemma}
\cref{lemma:dist-support} is convenient for us in the following sense. Suppose we wish to upper bound the difference of $\partial f(\mb q)$ and $\expect{\partial f}(\mb q)$ along some direction $\mb u\in\bb S^{n-1}$ (as we need in proving the key empirical geometry result~\cref{theorem:emp-grad-inward}). As both subdifferential sets are convex and compact, by~\cref{lemma:dist-support} we immediately have
\begin{align}
  \abs{\inf_{\mb g\in\partial f(\mb q)}\innerprod{\mb g}{\mb u} - \inf_{\mb g\in\expect{\partial f}(\mb q)}\innerprod{\mb g}{\mb u}}
  = \abs{-h_{\partial f(\mb q)}(-\mb u) + h_{\expect{\partial f}(\mb q)}(-\mb u)} \le \hd(\partial f(\mb q), \expect{\partial f}(\mb q)).
\end{align}
Therefore, as long as we are able to bound the Hausdorff distance, all
directional differences between the subdifferentials are
simultaneously bounded, which is exactly what we want to show to carry
the benign geometry from the population to the empirical
objective.

\subsection{Covering in the $\dexp$ metric}
We argue that the absence of gradient Lipschitzness is because the Euclidean distance is not the ``right'' metric in this problem. Think of the toy example $f(x)=|x|$, whose subdifferential set $\partial f(x)=\sign(x)$ is not Lipschitz across $x=0$ wrt the Euclidean distance. However, once we partition $\R$ into $\R_{>0}$, $\R_{<0}$ and $\set{0}$ (i.e. according to the sign pattern), the subdifferential set is Lipschitz on each subset.

The situation with the dictionary learning objective is quite similar: we resolve the gradient non-Lipschitzness by proposing a stronger metric $\dexp$ on the sphere which is sign-pattern aware and averages all ``subset angles'' between two points. Formally, we define $\dexp$ as
\begin{equation}
  \label{equation:dexp-closed-form}
  \dexp(\mb p, \mb q) \doteq \bb P_{\mb x\sim\bgt}\brac{\sign(\mb p^\top\mb x)\neq \sign(\mb q^\top\mb x)} = \frac{1}{\pi}\EOmega{\angle\paren{\mb p_\Omega, \mb q_\Omega}},
\end{equation}
(the second equality shown in~\cref{lemma:exp_metric}) on subsets of the sphere with consistent support patterns. Our plan is to perform the covering argument in $\dexp$, which requires showing gradient Lipschitzness in $\dexp$ and bounding the covering number.

\paragraph{Lipschitzness of $\partial f$ and $\expect{\partial f}$ in $\dexp$}
For the population subdifferential $\expect{\partial f}$, note that $\expect{\partial f}(\mb q)=\bb E_{\mb x\sim\bgt}[\sign(\mb q^\top \mb x)\mb x]$ (modulo rescaling). Therefore, to bound $\hd(\expect{\partial f}(\mb p), \expect{\partial f}(\mb q))$ by~\cref{lemma:dist-support}, we have the bound for all $\mb u\in \bb S^{n-1}$
\begin{multline}
  \abs{h_{\expect{\partial f}(\mb p)}(\mb u) - h_{\expect{\partial f}(\mb q)}(\mb u)}
  = \expect{\sup \abs{\sign(\mb p^\top\mb x) - \sign(\mb q^\top\mb x)} \cdot \abs{\mb x^\top\mb u}} \\
  \le 2\expect{\indicator{\sign(\mb p^\top\mb x)\neq \sign(\mb q^\top\mb x)}\abs{\mb x^\top\mb u}}.
\end{multline}
As long as $\dexp(\mb p,\mb q)\le\eps$, the indicator is non-zero with probability at most $\eps$, and thus the above expectation should also be small---we bound it by $O(\eps\sqrt{\log(1/\eps)})$ in~\cref{lemma:expected-gradient-diff}.

To show the same for the empirical subdifferential $\partial f$, one only needs to bound the observed proportion of sign differences for all $\mb p,\mb q$ such that $\dexp(\mb p, \mb q)\le \eps$, which by a VC dimension argument is uniformly bounded by $2\eps$ with high probability (\cref{lemma:uniform_sign_pert}).

\paragraph{Bounding the covering number in $\dexp$}
Our first step is to reduce $\dexp$ to the \emph{\bf maximum length-2 angle} (the $\d2$ metric) over any consistent support pattern. This is achieved through the following \emph{vector angle inequality}~(\cref{lemma:angle_ineqn}): for any $\mb p, \mb q\in\R^d$ ($d\ge 3$), we have
\begin{align}
  \angle\paren{\mb p, \mb q} \le \sum_{\Omega\subset[d],|\Omega|=2} \angle\paren{\mb p_\Omega, \mb q_\Omega}~~~\textrm{provided}~\angle\paren{\mb p, \mb q} \le \pi/2.
\end{align}
Therefore, as long as $\sign(\mb p)=\sign(\mb q)$ (coordinate-wise) and $\max_{|\Omega|=2}\angle(\mb p_\Omega, \mb q_\Omega)\le \eps/n^2$, we would have for all $|\Omega|\ge 3$ that
\begin{equation}
  \angle\paren{\mb p_\Omega,\mb q_\Omega} \le \pi/2~~~{\rm and}~~~\angle\paren{\mb p_\Omega,\mb q_\Omega} \le \sum_{\Omega'\subset\Omega, |\Omega'|=2} \angle\paren{\mb p_{\Omega'}, \mb q_{\Omega'}} \le \binom{|\Omega|}{2}\cdot \frac{\eps}{n^2}\le \eps.
\end{equation}
By~\cref{equation:dexp-closed-form}, the above implies that $\dexp(\mb p, \mb q)\le \eps/\pi$, the desired result.
Hence the task reduces to constructing an $\eta=\eps/n^2$ covering in $\d2$ over any consistent sign pattern.

Our second step is a {\bf tight bound on this covering number}: the $\eta$-covering number in $\d2$ is bounded by
$\exp(Cn\log(n/\eta))$~(\cref{lemma:cover-length-2}). To obtain this, a first thought would be to take the covering in all size-2 angles (there are $\binom{n}{2}$ of them) and take the common refinement of all their partitions, which gives covering number $(C/\eta)^{O(n^2)}=\exp(Cn^2\log(1/\eta))$. We improve upon this strategy by \emph{sorting} the coordinates in $\mb p$ and restricting attentions in the \emph{consecutive} size-2 angles after the sorting (there are $n-1$ of them). We show that a proper covering in these consecutive size-2 angles by $\eta/n$ will yield a covering for all size-2 angles by $\eta$. The corresponding covering number in this case is thus $(Cn/\eta)^{O(n)}=\exp(Cn\log(n/\eta))$, which modulo the $\log n$ factor is the tightest we can get.

%% file: Sections/experiment.tex
\section{Experiments}
\label{section:experiment}
\subsection{Experiments with simulated data}
\paragraph{Setup}
We set the true dictionary $\mb A$ to be the identity and random orthogonal matrices, respectively. For each choice, we sweep the combinations of $\paren{m, n}$ with $n \in \set{30, 50, 70, 100}$ and $m = 10n^{\set{0.5, 1, 1.5, 2, 2.5}}$, and fix the sparsity level at $\theta=0.1, 0.3, 0.5$, respectively. For each $\paren{m, n}$ pair, we generate $10$ problem instances, corresponding to re-sampling the coefficient matrix $\mb X$ for $10$ times. Note that our theoretical guarantee applies for $m=\widetilde{\Omega}(n^4)$, and the sample complexity we experiment with here is lower than what our theory requires. To recover the dictionary, we run the Riemannian subgradient descent algorithm~\cref{algorithm:rgd} with  decaying step size $\eta^{(k)}=1/\sqrt{k}$, corresponding to the boundary case $\alpha=1/2$ in~\cref{theorem:opt-one-basis} with a much better base size.

\paragraph{Metric}
As~\cref{theorem:main} guarantees recovering the entire
dictionary with $R\ge Cn\log n$ independent runs, we perform $R=\mathrm{round}\paren{5n\log n}$ runs on each
instance.  For each run, a true dictionary element $\mb a_i$ is considered to be found if $\norm{\mb a_i - \mb q_{\mathrm{best}}}{} \le 10^{-3}$. For each instance, we regard it a successful recovery if the $R = \mathrm{round}\paren{5n\log n}$ runs have found all the dictionary elements, and we report the empirical success rate over the $10$ instances.

\paragraph{Result}
From our simulations, Riemannian subgradient descent succeeds in recovering the dictionary as long as $m \ge Cn^2$~(\cref{figure:exp}), across different sparsity levels. The dependency on $n$ is consistent with our theory and suggests that the actual sample complexity requirement for guaranteed recovery might be even lower than $\wt{O}(n^4)$ that we established.\footnote{The $\wt{O}(\cdot)$ notation ignores the dependency on logarithmic terms and other factors. }
\begin{figure}[!htbp]
\centering
\includegraphics[width=0.32\textwidth]{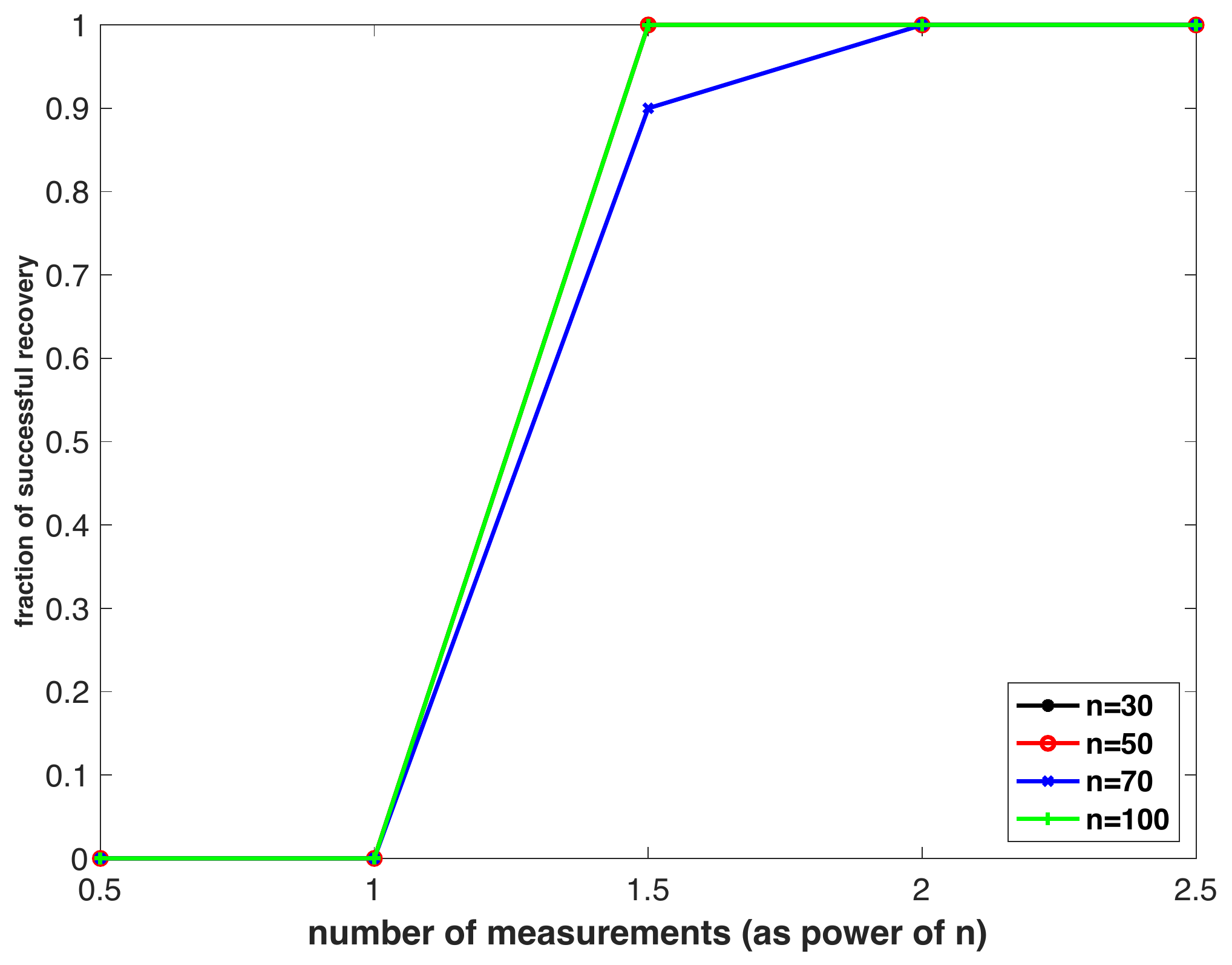}
\includegraphics[width=0.32\textwidth]{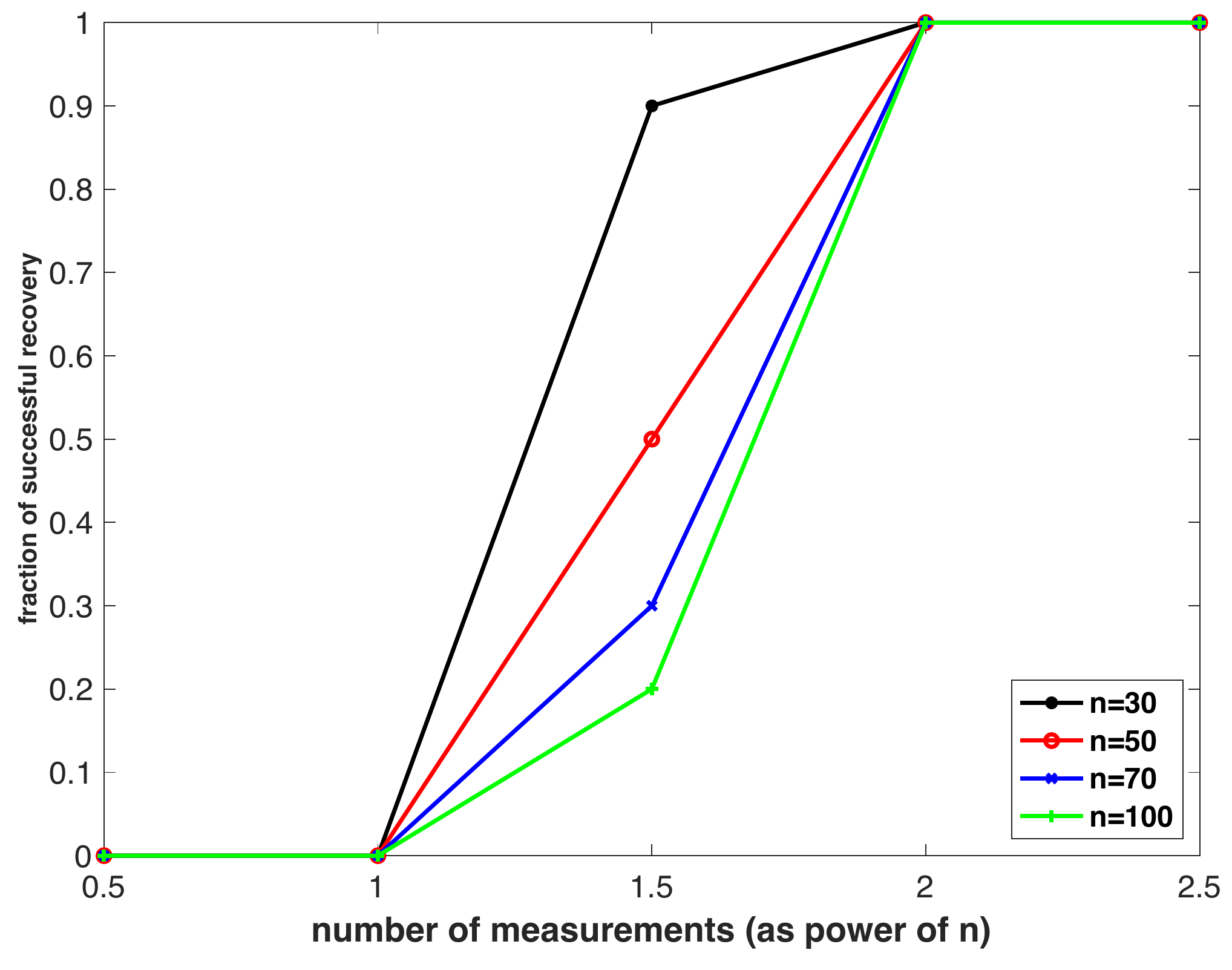}
\includegraphics[width=0.32\textwidth]{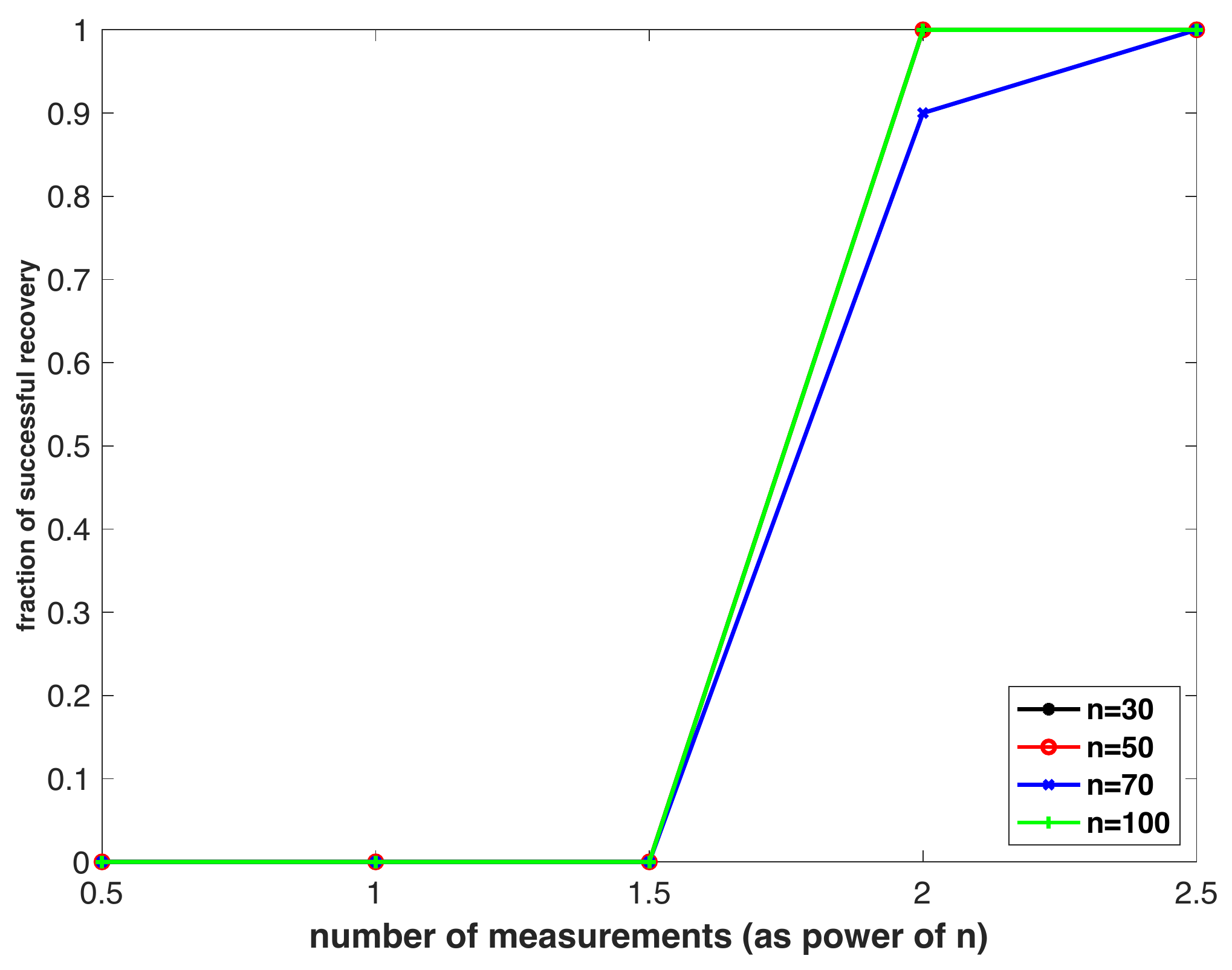} \\
\vspace{0.25cm}
\includegraphics[width=0.32\textwidth]{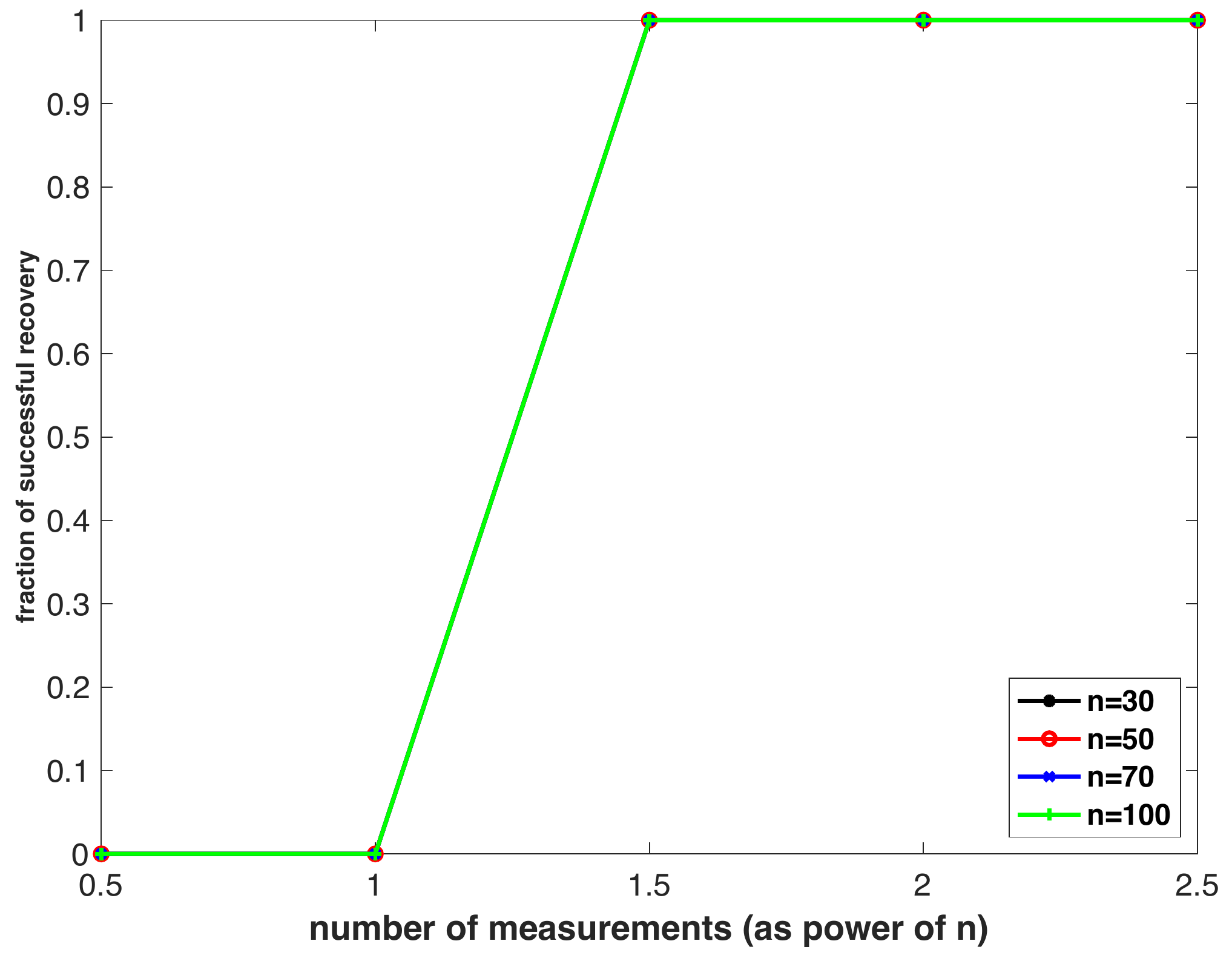}
\includegraphics[width=0.32\textwidth]{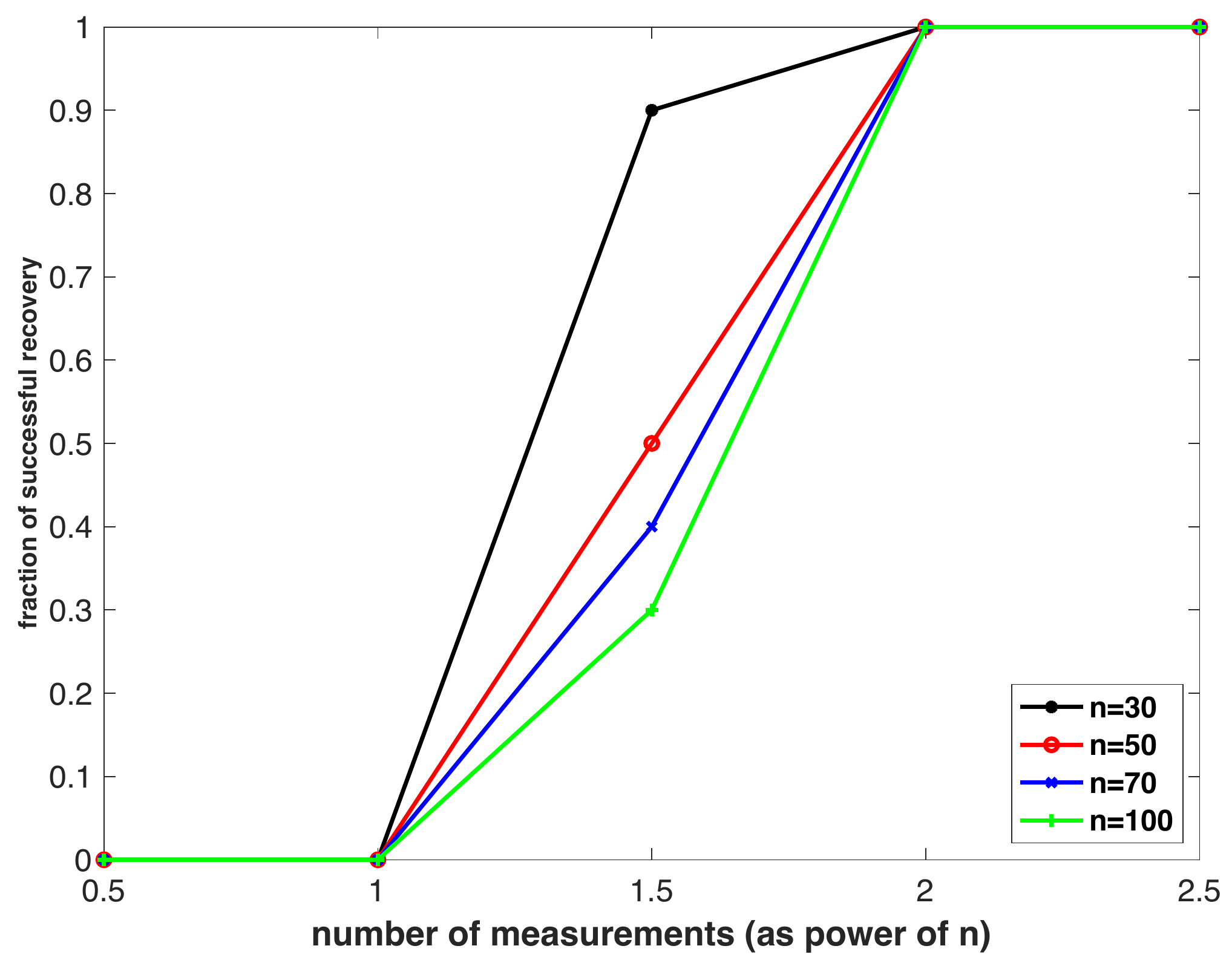}
\includegraphics[width=0.32\textwidth]{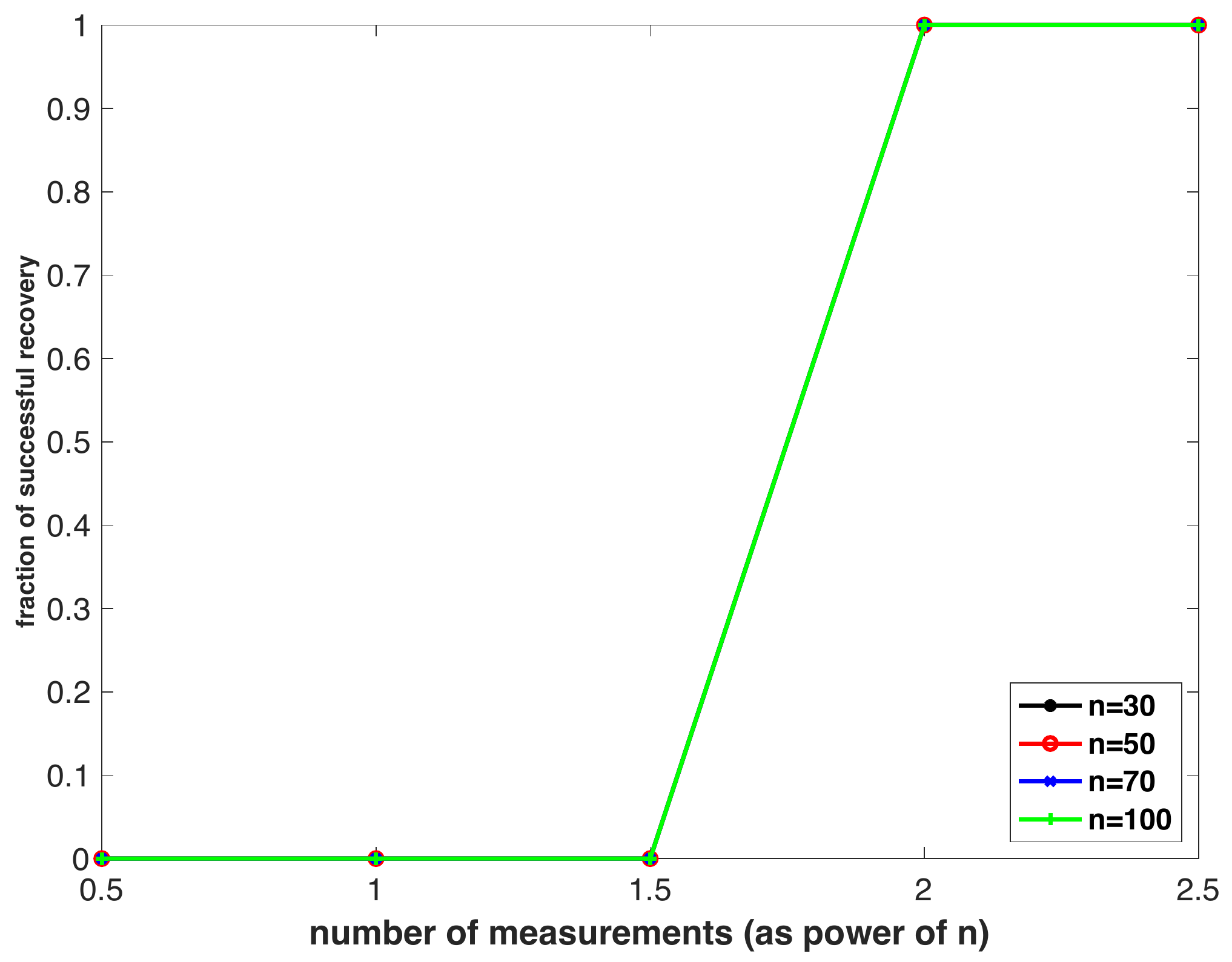}
\caption{\small Empirical success rates of recovery of the Riemannian subgradient descent with $R=5n\log n$ runs, averaged over $10$ instances. The top row is for the identity dictionary, and the bottom row for a random orthogonal dictionary. 
Left to right columns: $\theta=0.1, 0.3, 0.5$, respectively.}
\label{figure:exp}
\end{figure}
The $\wt{O}(n^2)$ rate we observe also matches the conjectured complexity based on the SOS method in~\citep{SchrammSteurer2017Fast}.\footnote{Previous methods based on SOS~\citet{BarakEtAl2015Dictionary,MaEtAl2016Polynomial} provide no concrete sample complexity estimates besides being polynomial in $n$. } Moreover,
the problem seems to become harder when $\theta$ grows, evident from the observation that the success transition threshold being pushed to the right.

\paragraph{Faster alternative for large-scale instances}
The Riemannian subgradient descent is cheap per iteration but slow in overall convergence, similar to many other first-order methods. We also test a faster quasi-Newton type method, GRANSO,\footnote{Available online: \url{http://www.timmitchell.com/software/GRANSO/}. Our experiment is based on version 1.6. } that employs BFGS for solving constrained nonsmooth problems based on sequential quadratic optimization~\citep{CurtisEtAl2017BFGS}. For a large dictionary of dimension $n=400$ and sample complexity $m=10n^2$ (i.e., $1.6 \times 10^{6}$), GRANSO successfully identifies a basis after $1500$ iterations with CPU time $4$ hours on a two-socket Intel Xeon E5-2640v4 processor (10-core Broadwell, 2.40 GHz)---this is approximately 10$\times$ faster than the Riemannian subgradient descent method, showing the potential of quasi-Newton type methods for solving large-scale problems.

\subsection{Experiments with images}
\label{section:image_exp}
\begin{figure}[!htbp]
\centering
\includegraphics[width=0.38\textwidth]{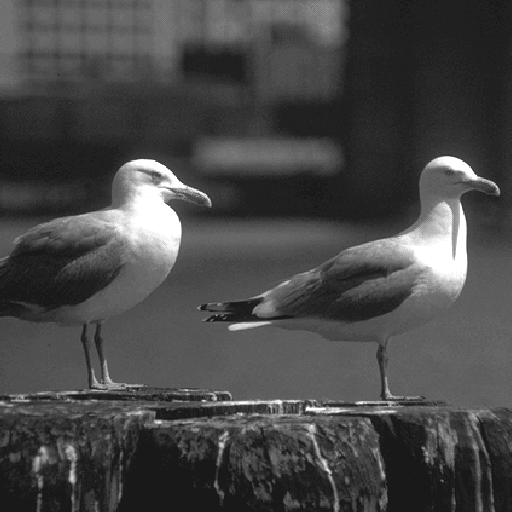}
\hspace{0.2in}
\includegraphics[width=0.38\textwidth]{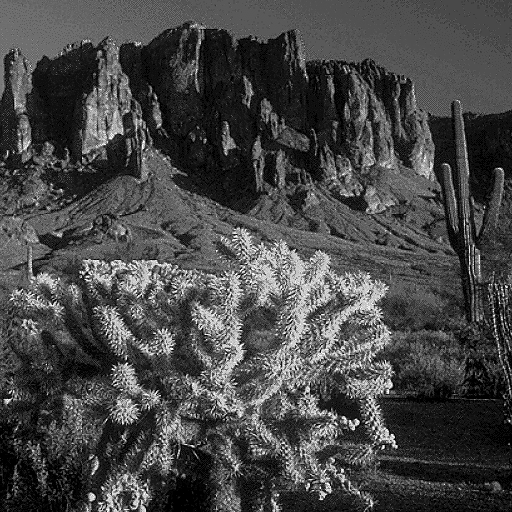} \\
\vspace{0.05in}
\includegraphics[width=0.38\textwidth]{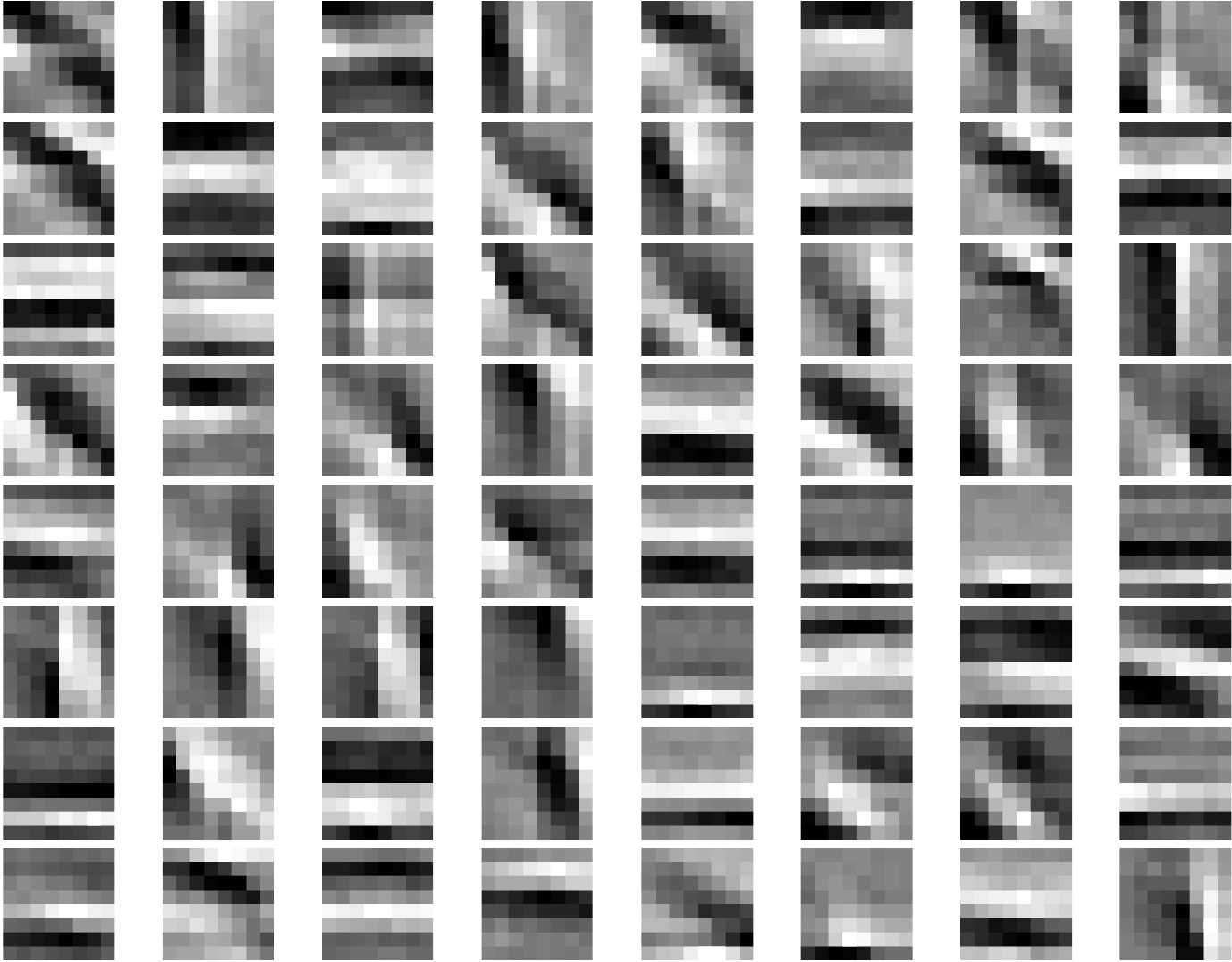}
\hspace{0.2in}
\includegraphics[width=0.38\textwidth]{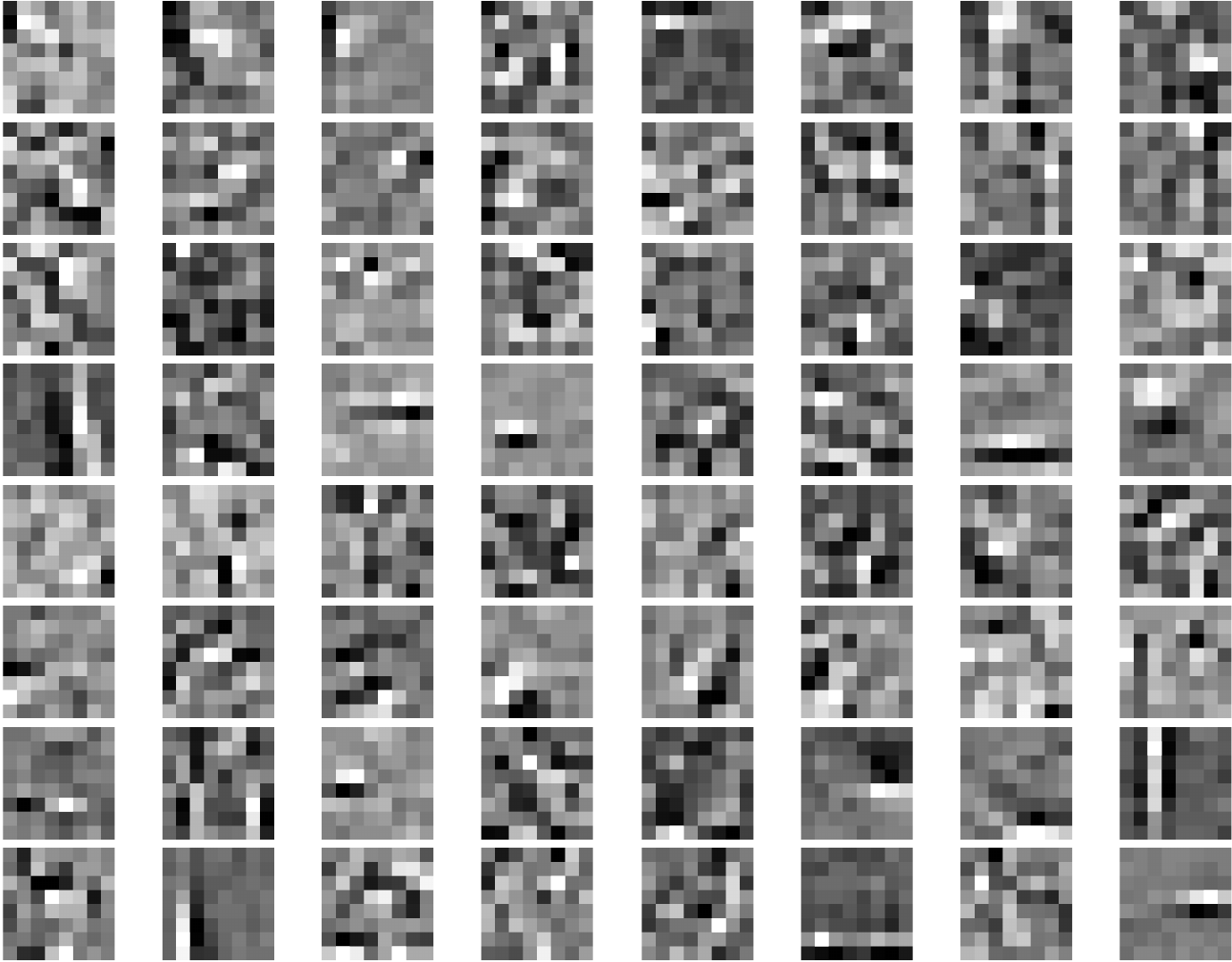}\\
\includegraphics[width=0.38\textwidth]{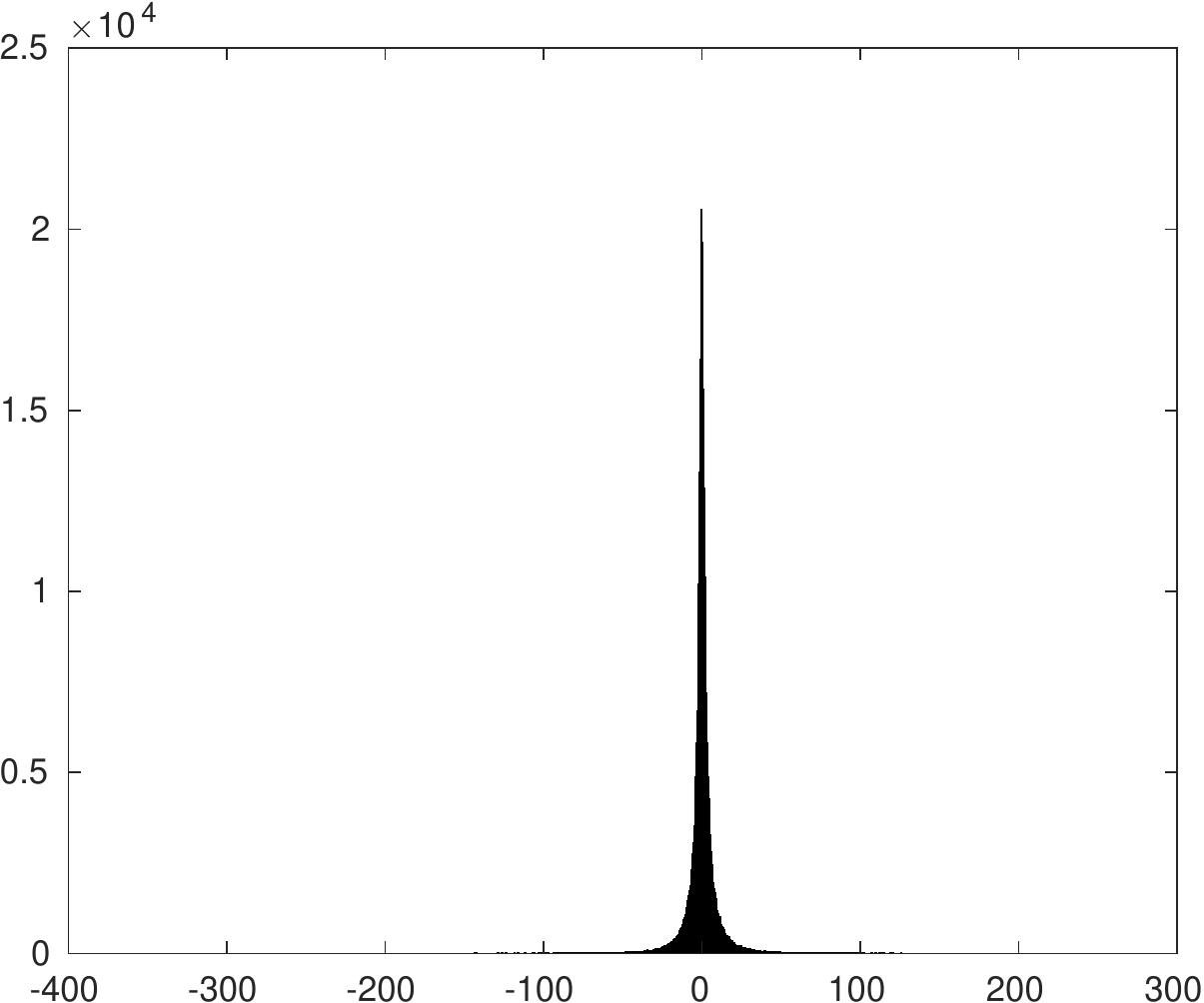}
\hspace{0.2in}
\includegraphics[width=0.38\textwidth]{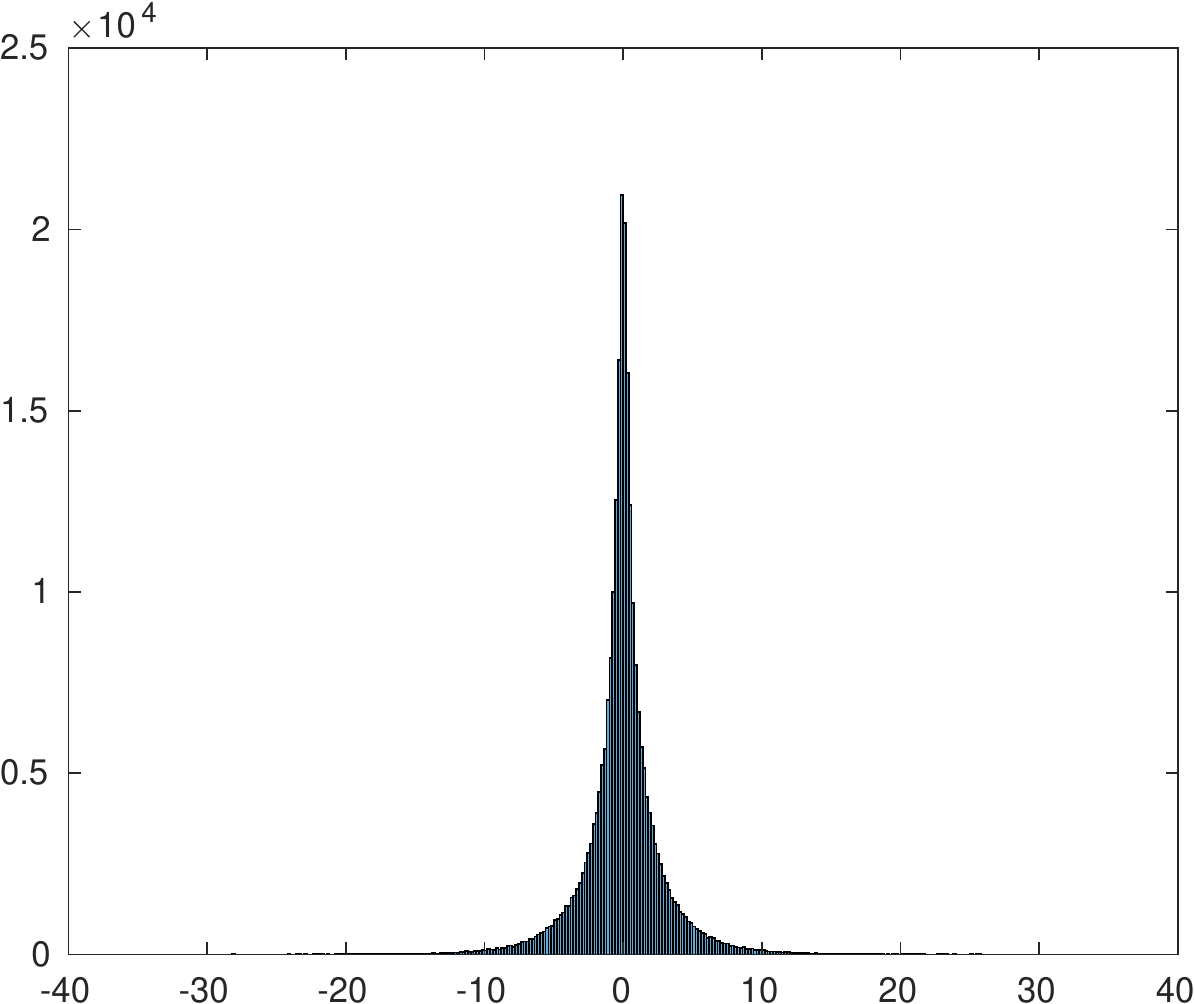}\\
\includegraphics[width=0.38\textwidth]{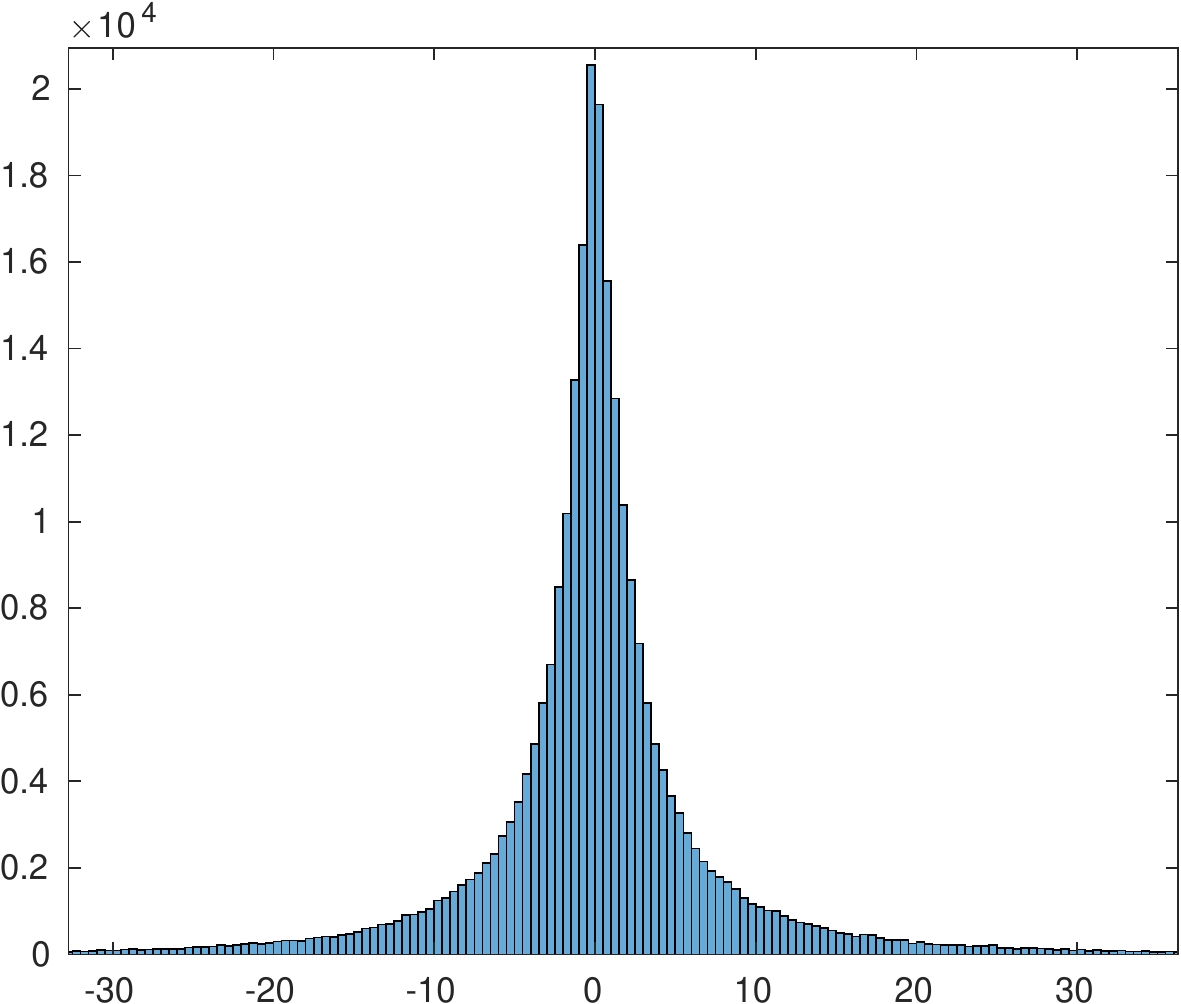}
\hspace{0.2in}
\includegraphics[width=0.38\textwidth]{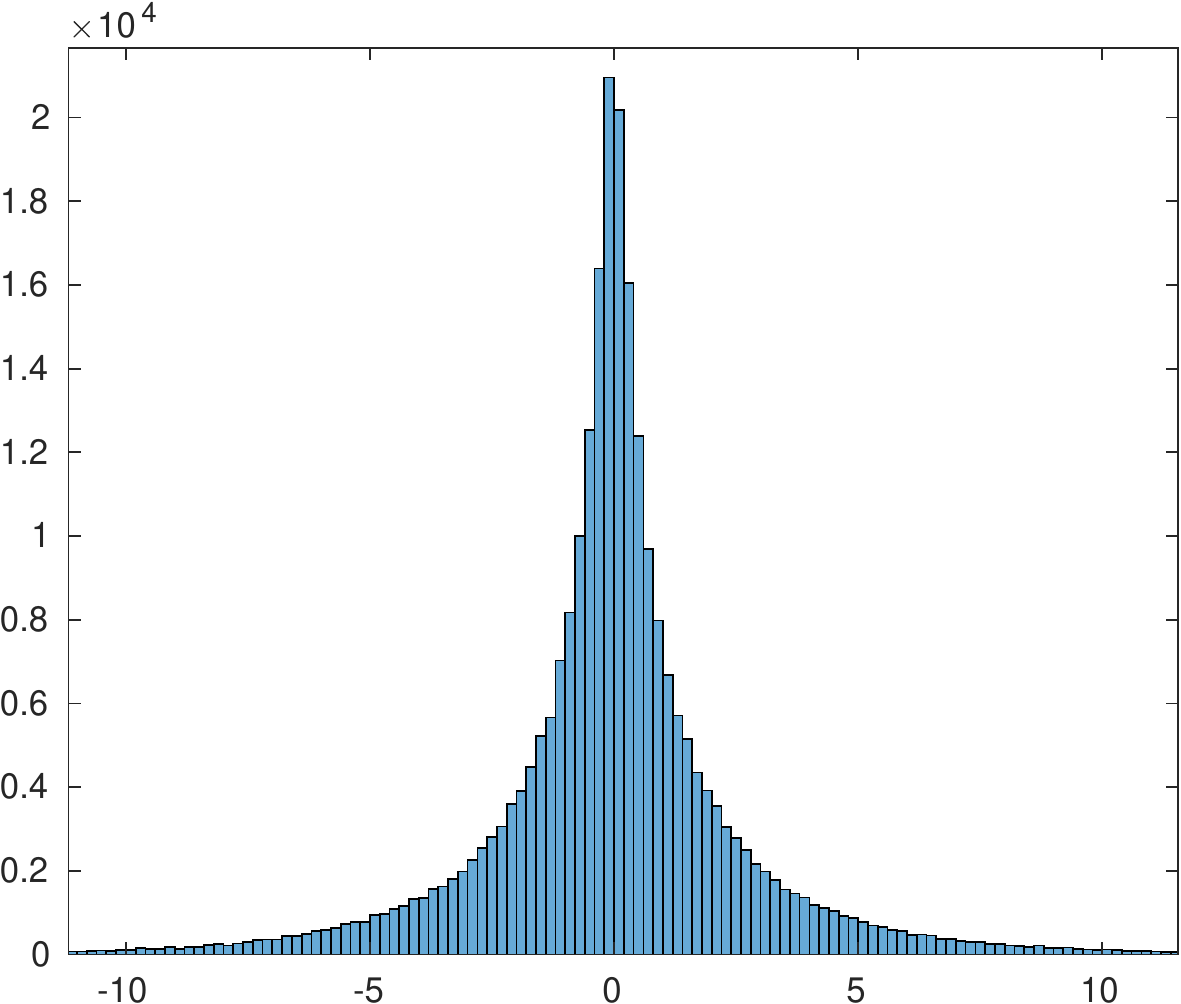}
\caption{Results on two images.  First row: the images; Second row: learned dictionaries; Third row: histograms of the representation coefficients; Fourth row: zoomed-in versions of the histograms around zero.}
\label{fig:image_exp}
\end{figure}

To experiment with images, we follow a typical setup for dictionary learning as used in image processing~\citep{MairalEtAl2014Sparse}. We focus on testing if complete (i.e., square and invertible) dictionaries are reasonable sparsification bases for real images, instead on any particular image processing or vision tasks.

\paragraph{Setup}
Two natural images are picked for this experiment, as shown in the first row of~\cref{fig:image_exp}, each of resolution $512 \times 512$. Each image is divided into $8 \times 8$ non-overlapping blocks, resulting in $64 \times 64 = 4096$ blocks. The blocks are then vectorized, and  stacked columnwise into a data matrix $\mb Y \in \R^{64 \times 4096}$. We precondition the data to obtain
\begin{align}
\ol{\mb Y} = \paren{\mb Y \mb Y^\top}^{-1/2} \mb Y,
\end{align}
so that nonvanishing singular values of $\ol{\mb Y}$ are identically one. We then solve formulation (\ref{problem:nonsmooth-dl}) $\mathrm{round}\paren{5n \log n}$ times with $n = 64$ using the BFGS solver based on GRANSO, obtaining $\mathrm{round}\paren{5n \log n}$ vectors. Negative equivalent copies are pruned and vectors with large correlations with other remaining vectors are sequentially removed until only $64$ vectors are left. This forms the final complete dictionary.

\paragraph{Results}
The learned complete dictionaries for the two test images are displayed in the second row of~\cref{fig:image_exp}. Visually, the dictionaries seem reasonably adaptive to the image contents: for the left image with prevalent sharp edges, the learned dictionary consists of almost exclusively oriented sharp corners and edges, while for the right image with blurred textures and occasional sharp features, the learned dictionary does seem to be composed of the two kinds of elements. Let the learned dictionary be $\mb A$. We estimate the representation coefficients as $\mb A^{-1} \ol{\mb Y}$. The third row of~\cref{fig:image_exp}  contains the histograms of the coefficients. For both images,  the coefficients are sharply concentrated around zero (see also the fourth row for zoomed versions of the portions around zero), and the distribution resembles a typical zero-centered Laplace distribution---which is a good indication of sparsity. Quantitatively, we calculate the mean sparsity level of the coefficient vectors (i.e., columns of $\mb A^{-1} \mb Y$) by the metric $\norm{\cdot}{1}/\norm{\cdot}{2}$:  for a vector $\mb v \in \R^n$, $\norm{\mb v}{1}/\norm{\mb v}{2}$ ranges from $1$ (when $\mb v$ is one-sparse) to $\sqrt{n}$ (when $\mb v$ is fully dense with elements of equal magnitudes), which serves as a good measure of sparsity level for $\mb v$. For our two images, the sparsity levels by the norm-ratio metric are $5.9135$ and $6.4339$, respectively,  while the fully dense extreme would have a value $\sqrt{64} = 8$, suggesting the complete dictionaries we learned are reasonable sparsification bases for the two natural images, respectively. 

\subsection{Reproducible research}
Codes to reproduce all the experimental results are available online:
\begin{quote}
\centering
\url{https://github.com/sunju/ODL_L1}
\end{quote}

%% file: Sections/discuss.tex
\section{Discussion}
\label{section:discuss}
We close the paper by identifying a number of potential future directions.

\paragraph{Improving the sample complexity}
There is an $O(n^2)$ sample complexity gap between what we established in~\cref{theorem:main} and what we observed in the simulations. However, the main geometric result we obtained in~\cref{theorem:emp-grad-inward}, which dictates the final sample complexity, seem tight, as the order of the lower bounds are achievable for points that lie near the $\mc S_{\zeta_0}^{(i+)}$ set boundaries where the gradients are weak. Thus, to improve the sample complexity based on the current algorithm, one possibility is dispensing with the algorithm-independent landscape characterization and performing an algorithm-specific analysis. On the other hand, near the set boundaries, the inward directional curvatures are strong (\cref{prop:asym_directional_curv}) despite the weak gradients. This suggests that the second-order curvature information may come to rescue in saving the sample complexity. This calls for algorithms that can appropriately exploit the second-order curvature information.

\paragraph{Learning complete and overcomplete dictionaries}
We cannot directly generalize the result of learning orthogonal dictionaries to that of learning complete ones via Lipschitz continuity arguments as in~\citet{SunEtAl2015Complete}---most of the quantities needed in the generalization are not Lipschitz for our nonsmooth formulation. It is possible, however, to get around the continuity argument and perform a direct analysis. For learning overcomplete dictionaries in the linear sparsity regime, to date the only known provable results are based on the SOS method~\citep{BarakEtAl2015Dictionary,MaEtAl2016Polynomial,SchrammSteurer2017Fast}. The line of work~\citep{SpielmanEtAl2012Exact,SunEtAl2015Complete} that we follow here breaks the intrinsic bilinearity of dictionary learning problems (i.e., $\mb Y$ is bilinear in $\paren{\mb A, \mb X}$ when $\mb Y = \mb A \mb X$) by recovering one factor first based on structural assumptions on the data. The same strategy may also be applied to the overcomplete case. Nonetheless, directly solving the problem in the product space $\paren{\mb A, \mb X}$ would close the prolonged theory-practice gap for dictionary learning, and bears considerable ramifications on solving other bilinear problems.

\paragraph{Learning nonlinear transformations}
As  we mentioned in the introduction, one can also pose dictionary learning in an analysis setting, in which a linear transformation on given data is learned to reveal hidden structures. In this sense, we effectively learn an orthogonal transformation that sparsifies the given data. Following the success, it is natural to ask how to learn overcomplete transformations, convolutional transformations (as in signal processing and computer vision), and general nonlinear transformations to expose the intrinsic data structures---this is exactly the problem of representation learning and has been predominantly tackled by training deep networks~\citep{LeCunEtAl2015Deep} recently. Even for provable and efficient training of ReLU networks with two unknown weight layers, the current understanding is still partial,\footnote{Most of the existing results actually pertain only to the cases with one unknown weight layer, not two. } as discussed in~\cref{sec:related_work}.


\paragraph{Solving practical problems with nonsmoothness}
Besides applications that are directly formulated as nonsmooth, nonconvex problems as sampled in~\cref{sec:related_work}, there are cases where smooth formulations are taken instead of natural nonsmooth formulations: \citet{SunEtAl2015Complete} is certainly one example that has motivated the current work; others include (multi-channel) sparse blind deconvolution~\citep{ZhangEtAl2018Structured,LiBresler2018Global}. For the latter two, maximization of $\norm{\cdot}{4}^4$ is used as a proxy for minimization of $\norm{\cdot}{1}$ toward promoting sparsity. The $\norm{\cdot}{4}^4$ induces polynomials and hence nice smoothness properties, but also entails heavy-tailed distributions in its low-order derivatives if the data contain randomness. For all the applications covered above, the tools we gather here around nonsmooth analysis, set-valued analysis, and random set theory provide a solid and convenient framework for theoretical understanding and practical optimization.

\paragraph{Optimizing nonsmooth, nonconvex functions}
General iterative methods for nonsmooth problems discussed in~\cref{sec:related_work} are only guaranteed to converge to stationary points. By contrast, the recent progress in smooth nonconvex optimization has demonstrated that it is possible to separate the efforts of function landscape characterization and algorithm development~\citep{SunEtAl2015When,GeEtAl2015Escaping,ChiEtAl2018Nonconvex}, so that on a class of benignly structured problems, global optimizers can be found by virtually any reasonably iterative methods without special initialization. In this paper, we do not obtain a strong result of this form, as we have empirically found suspicious local minimizers near the population saddles, \footnote{We observed a similar phenomenon in our unpublished work on a nonsmooth formulation of the generalized phase retrieval problem. } and hence we only provide an ``almost global" landscape characterization. It is natural to wonder to what extent we can attain the nice separation similar to the smooth cases for nonsmooth functions, and what ``benign structures" are in this context for practical problems. The recent works~\citep{DavisEtAl2018Subgradient,DavisDrusvyatskiy2018Uniform,DuchiRuan2017Solving,DavisEtAl2017nonsmooth,LiEtAl2018Nonconvexa} represent pioneering developments in this line.

%% file: Sections/tools.tex
\section{Technical tools}
\label{sec:technical_tools}

\subsection{Nonsmooth analysis on Riemannian manifolds}
We first sketch the main results in $\R^d$, and then discuss the extension to embedded Riemannian manifolds of $\R^d$. Our treatment loosely follows~\citet[Chapter 2]{Clarke1990Optimization} and~\cite{HosseiniPouryayevali2011Generalized, HosseiniUschmajew2017Riemannian}, and presents a minimal background needed for the current paper.

First consider functions $f: X \mapsto \R$ for $X \subset \R^d$. We focus on locally Lipschitz functions, which, as the name suggests, are functions that are Lipschitz only \emph{locally}. Simple examples are continuous convex (e.g., $\norm{\cdot}{1}$ and $\mb q \mapsto \norm{\mb q^\top \mb Y}{1}$) and concave functions (e.g., $-\norm{\cdot}{1}$), and continuously differentiable functions, as well as sums, products, quotients, and compositions of locally Lipschitz functions.

We now introduce the Clarke's subdifferential for locally Lipschitz functions.
\begin{definition}[Clarke directional derivatives for locally Lipschitz functions]
  Let $f: \R^d \to \R$ be locally Lipschitz at $\mb x$. The Clarke directional derivative of function $f$ at $\mb x$ in the direction $\mb v$, $D^c_{\mb v}$, is defined as
  \begin{align}
    D^c_{\mb v} f\paren{\mb x} \doteq \limsup_{t \downto 0, \; \mb y \to \mb x} \frac{f\paren{\mb y + t \mb v} - f\paren{\mb y}}{t}.
  \end{align}
\end{definition}

This generalizes the notion of right directional derivative, defined as
\begin{align}
  D_{\mb v} f\paren{\mb x} \doteq \lim_{t \downto 0} \frac{f\paren{\mb x + t \mb v} - f\paren{\mb x}}{t}.
\end{align}
For $f$ locally Lipschitz at $\mb x$, the right directional derivative may not be well defined, while the Clarke directional derivative always is. The cases when they are identical prove particularly relevant for applications.
\begin{definition}[(Subdifferential) regularity]
  A function $f: \R^d \to \R$ is said to be regular at point $\mb x$ if: i) $D_{\mb v}f\paren{\mb x}$ exists for all $\mb v$; and ii) $D^c_{\mb v}f\paren{\mb x} = D_{\mb v}f\paren{\mb x}$ for all $\mb v$.
\end{definition}
If $f$ is locally Lipschitz at $\mb x$ and convex, or continuously differentiable at $\mb x$, then it is regular at $\mb x$. Moreover,  if $f$ is a nonnegative linear combination of functions that are regular at $\mb x$, then $f$ is also regular at $\mb x$.
\begin{definition}[Clarke subdifferential, derivative-based definition]  \label{def:clarke_subd_euclidean_dd}
Let $f: \R^d \to \R$ be locally Lipschitz at $\mb x$ so that Clarke directional derivative $D^c_{\mb v} f\paren{\mb x}$ exists for all $\mb v$. The Clarke subdifferential of $f$ at $\mb x$ is defined as
\begin{align}
  \partial f\paren{\mb x} \doteq \set{\mb y \in \R^d: \; \innerprod{\mb y}{\mb v} \le D_{\mb v}^c f\paren{\mb v}\; \forall\; \mb v}.
\end{align}
\end{definition}
One important implication of this definition is that
\begin{align}
  D_{\mb v}^c f\paren{\mb x} = \sup \innerprod{\partial f\paren{\mb x}}{\mb v}  \; \; \forall \; \mb v.
\end{align}
If $f$ is Lipschitz over an open set $U \subset \R^d$ , it is differentiable almost everywhere in $U$, due to the celebrated Rademacher's theorem~\citep{Federer1996Geometric}. We then have the following well-defined, equivalent definition of Clarke subdifferential.
\begin{definition}[Clarke subdifferential, sequential definition]  \label{def:clarke_subd_euclidean}
For $f: X \mapsto \R$ locally Lipschitz and any $\mb x \in X$, the Clarke's subdifferential is defined as
\begin{align}
  \partial f\paren{\mb x} \doteq \mathrm{conv} \set{\lim \nabla f\paren{\mb x_k}: \; \mb x_k \to \mb x, f\; \text{differential at}\; \mb x_k}.
\end{align}
In particular, $\partial f\paren{\mb x}$ is a nonempty, convex, compact set.
\end{definition}
When $f$ is continuously differentiable at $\mb x$, $\partial f\paren{x}$ reduces to the singleton $\set{\nabla f\paren{\mb x}}$. When $f$ is convex, $\partial f$ coincides with the usual subdifferential in convex analysis~\citep{Hiriart-UrrutyLemarechal2001Fundamentals}. Most natural calculus rules hold with subdifferential inclusion, and they can often be strengthened to equalities under the subdifferential regularity condition. Here we provide a sample of these results.
\begin{theorem}[Subdifferential calculus]
Assume $f, f_1, \dots, f_K$ are functions mapping from $\R^d$ to $\R$ and locally Lipschitz at $\mb x$. We have
\begin{enumerate}
  \item $\alpha f$ is locally Lipschitz at $\mb x$ and $\partial \paren{\alpha f}\paren{\mb x} = \alpha \partial f\paren{\mb x}$ for all $\alpha \in \R$;
  \item $g \paren{\mb x} \doteq \sum_{i=1}^K \lambda_i f_i\paren{\mb x}$ is locally Lipschitz at $\mb x$ for any $\lambda_i \in \R$ and $\partial g\paren{\mb x} \subset \sum_{i=1}^K \lambda_i \partial f_i\paren{\mb x}$. If in addition $\lambda_i \ge 0$ for all $i \in [K]$ and all $f_i$'s are regular at $\mb x$, $g$ is regular at $\mb x$ and $\partial g\paren{\mb x} = \sum_{i=1}^K \lambda_i \partial f_i\paren{\mb x}$;
  \item The function $f_1 f_2$ is locally Lipschitz at $\mb x$ and $\partial \paren{f_1 f_2}\paren{\mb x} \subset f_1\paren{\mb x} \partial f_2\paren{\mb x} + f_2\paren{\mb x} \partial f_1\paren{\mb x}$. If in addition $f_1\paren{\mb x}, f_2\paren{\mb x} \ge 0$ and $f_1, f_2$ are both regular at $\mb x$, $f_1f_2$ are regular at $\mb x$ and $\partial \paren{f_1 f_2}\paren{\mb x} = f_1\paren{\mb x} \partial f_2\paren{\mb x} + f_2\paren{\mb x} \partial f_1\paren{\mb x}$;
  \item The function $f_1/f_2$ (assuming $f_2 \ne 0$) is locally Lipschitz at $\mb x$ and
  \begin{align*}
    \partial \paren{\frac{f_1}{f_2}}\paren{\mb x} \subset \frac{f_2\paren{\mb x} \partial f_1\paren{\mb x} - f_1\paren{\mb x} \partial f_2\paren{\mb x}}{f_2^2\paren{\mb x}}.
  \end{align*}
  If in addition $f_1 \paren{\mb x} \ge 0$ and $f_2\paren{\mb x} > 0$ and both $f_1$ and $-f_2$ are regular at $\mb x$, then $f_1/f_2$ is regular at $\mb x$ and the above inclusion becomes an equality.
  \item The function $g\paren{\mb x} \doteq \max\set{f_i\paren{\mb x}: i \in [K]}$ is locally Lipschitz at $\mb x$ and
  \begin{align*}
    \partial g\paren{\mb x} \subset \conv \set{\partial f_i\paren{\mb x}: \; f_i \; \text{active at}\; \mb x, i.e., f_i\paren{\mb x} = f\paren{\mb x}}.
  \end{align*}
  If in addition all active $f_i$'s are regular at $\mb x$, $f$ is regular at $\mb x$ and the above inclusion becomes an inequality.
\end{enumerate}
\end{theorem}

\begin{theorem}[Chain rule, Theorems 2.3.9 and 2.3.10 of~\citep{Clarke1990Optimization}]\label{thm:subdiff_chain_rule}
Consider $h: \R^d \mapsto \R^n$ with component functions $h_1, \dots, h_n$ and $g: \R^n \mapsto \R$, and assume each $h_i$ for all $i \in [n]$ is locally Lipschitz at $\mb x \in \R^d$ and $g$ is locally Lipschitz at $h\paren{\mb x}$. Then, the composition $f = g \circ h$ is locally Lipschitz at $\mb x$, and
\begin{align}
  \partial f\paren{\mb x} \subset \conv \set{\sum_i^n \alpha_i \mb \xi_i: \; \mb \xi_i \in \partial h_i\paren{\mb x}, \alpha \in \partial g\paren{h\paren{\mb x}}}.
\end{align}
The inclusion becomes an equality under additional assumptions:
\begin{enumerate}
  \item If each $h_i$ is regular at $\mb x$ and $g$ is regular at $h\paren{\mb x}$, and $\alpha \ge 0$ for all $\alpha \in \partial g\paren{h\paren{\mb x}}$,  $f$ is regular at $\mb x$ and
  \begin{align}
      \partial f\paren{\mb x} = \conv \set{\sum_i^n \alpha_i \mb \xi_i: \; \mb \xi_i \in \partial h_i\paren{\mb x}, \alpha \in \partial g\paren{h\paren{\mb x}}}.
  \end{align}
  \item If $g$ is continuously differentiable at $h\paren{\mb x}$ and $n=1$,
  \begin{align}
    \partial f\paren{\mb x} = g'\paren{h\paren{\mb x}} \partial h\paren{\mb x}.
  \end{align}
  \item If $\mb g$ is regular at $h\paren{\mb x}$ and $h$ is continuously differentiable at $\mb x$, $f$ is regular at $\mb x$ and
  \begin{align}
    \partial f\paren{\mb x} = \nabla h\paren{\mb x}^\top \partial g\paren{h\paren{\mb x}}.
  \end{align}
  \item If $\mb h$ is continuously differentiable at $\mb x$, and $\nabla h\paren{\mb x}$ is onto (locally),
  \begin{align}
    \partial f\paren{\mb x} = \nabla h\paren{\mb x}^\top \partial g\paren{h\paren{\mb x}}.
  \end{align}
\end{enumerate}
\end{theorem}

We also have a remarkable extension of mean value theorem to nonsmooth functions (cf. the mean value theorem for convex functions in, e.g., Theorem D 2.3.3 of~\citet{Hiriart-UrrutyLemarechal2001Fundamentals}).
\begin{theorem}[Lebourg's mean value theorem, Theorem 2.3.7 of~\citet{Clarke1990Optimization}] \label{thm:nsms_mvt}
Let $\mb x, \mb y$ be distinct points in $\R^d$, and suppose that $f: \R^d \mapsto \R$ is locally Lipschitz on an open set $U$ containing the line segment $[\mb x, \mb y]$. Then there exists a $\mb u \in \paren{\mb x, \mb y}$ such that
\begin{align}
  f\paren{\mb y} - f\paren{\mb x} \in \innerprod{\partial f\paren{\mb u}}{\mb y -\mb x}.
\end{align}
\end{theorem}
\begin{remark}
In~\cite{Clarke1990Optimization}, $f$ is assumed to be Lipschitz on $U$, instead of locally Lipschitz. Inspecting the proof, we can easily see that requiring $f$ being locally Lipschitz suffices. The version as stated here can also be found in~\citep[Theorem 4]{Hiriart-Urruty1980Mean} and~\citep[Theorem 7.4.4]{Schirotzek2007Nonsmooth}.
\end{remark}

\begin{theorem}[Exchanging subdifferentiation and integration, Theorem 2.7.2 of~\citet{Clarke1990Optimization}]
Let $U$ be an open subset in $\R^d$ and $\paren{T, \mc T, \mu}$ be a positive measure space. Consider a family of functions $f_t: U \mapsto \R$ satisfying:
\begin{enumerate}
  \item For each $\mb x \in U$, the map $t \mapsto f_t\paren{\mb x}$ is measurable;
  \item There exists an integrable function $k\paren{\cdot}: T \mapsto \R$ so that for all $\mb x, \mb y \in U$ and all $t \in T$, $\abs{f_t\paren{\mb x} - f_t\paren{\mb y}} \le k(t) \norm{\mb x - \mb y}{}$.
\end{enumerate}
If the integral function $f\paren{\mb x} \doteq \int_T f_t\paren{\mb x} \mu \paren{dt}$  is defined on some point $\mb x \in U$, it is defined and Lipschitz in $U$ and
\begin{align}
  \partial f \paren{\mb x} = \partial \int_T f_t\paren{\mb x} \mu \paren{dt} \subset \int_T \partial f_t\paren{\mb x} \mu \paren{dt},
\end{align}
where the second integral is understood as the selection integral~\citep{AubinFrankowska2009Set}. If in addition each $f_t\paren{\cdot}$ is regular at $\mb x$, $f$ is regular at $\mb x$ and the last inclusion becomes an equality.
\end{theorem}

We need an extended notion of Clarke's subdifferential on Riemannian manifolds. Excellent introduction to elements of Riemannian geometry tailored to the optimization context can be found in the monograph~\citep{AbsilEtAl2008Optimization}. We focus on finite-dimensional Riemannian manifolds. In this setting, locally Lipschitz functions can be defined similar to the Euclidean setting, assuming the corresponding Riemannian distance. Then, by Rademacher's theorem and the local equivalence of finite-dimensional Riemannian manifold and finite-dimensional Euclidean distance, locally Lipschitz $f: \mc M \mapsto \R$ on a Riemannian manifold $\mc M$ is differentiable almost everywhere. So, we can generalize naturally the definition of Clarke's subdifferential in~\cref{def:clarke_subd_euclidean} as follows.
\begin{definition}[Clarke subdifferential for locally Lipschitz functions on Riemannian manifolds] \label{def:clarke_subd_manifold}
  Consider $f: \mc M \mapsto \R$ locally Lipschitz on a finite-dimensional Riemannian manifold $\mc M$. For any $\mb x \in \mc M$, the Clarke's subdifferential is defined as
  \begin{align}
    \partial_R f\paren{\mb x} \doteq
    \mathrm{conv}\set{\lim \mathrm{grad} f\paren{\mb x_k}: \; \mb x_k \to \mb x, f\; \text{differentiable at}\; \mb x_k \in \mc M} \subset T_{\mb x} \mc M.
  \end{align}
  Here $\grad\paren{\cdot}$ denotes the Riemannian gradient, and we use subscript $R$ (as $\partial_R\paren{\cdot}$ ) to emphasize the subdifferential in taken wrt the Riemannian manifold $\mb M$, not the ambient space. The set $\partial_R f\paren{\mb x}$ is nonempty, convex, and compact.
\end{definition}

A point $\mb x \in \mc M$ is a stationary point of $f$ on $\mc M$ if $\mb 0 \in \partial_R f(\mb x)$. A necessary condition that $f$ achieves a local minimum at $\mb x$ is that $\mb 0 \in \partial_R f\paren{\mb x}$.

\subsection{Hausdorff distance}
\label{appendix:hausdorff}
We use the Hausdorff metric to measure differences between nonempty sets. For any set $X$ and a point $\mb p$ in $\R^n$, the point-to-set distance is defined as
\begin{align}
\mathrm{d}\paren{\mb q, X} \doteq \inf_{\mb x \in X} \norm{\mb x - \mb p}{}.
\end{align}
For any two sets $X_1, X_2 \in \R^n$, the Hausdorff distance is defined as
\begin{align}
 \hd\paren{X_1, X_2} \doteq \max\set{\sup_{\mb x_1 \in X_1} \dh\paren{\mb x_1, X_2}, \sup_{\mb x_2 \in X_2} \dh\paren{\mb x_2, X_1}},
\end{align}
or equivalently,
\begin{align}
  \hd\paren{X_1, X_2} \doteq \inf \set{\eps \ge 0: \; X_1 \subset X_2 + \bb B(\mb 0, \eps), X_2 \subset X_1 + \bb B(\mb 0, \eps)}.
\end{align}
When $X_1$ is a singleton, say $X_1 = \set{\mb p}$. Then
\begin{align}
\hd\paren{\set{\mb p}, X_2} = \sup_{\mb x_2 \in X_2} \norm{\mb x_2 - \mb p}{}.
\end{align}
Moreover, for any sets $X_1, X_2, Y_1, Y_2 \subset \R^n$,
\begin{align}
\hd\paren{X_1+ Y_1, X_2+ Y_2} \le \hd\paren{X_1, X_2 }  + \hd\paren{Y_1, Y_2},
\end{align}
which can be readily verified from the definition of $\hd$. On the sets of nonempty, compact subsets of $\R^n$, the Hausdorff metric is a valid metric; particularly, it obeys the triangular inequality: for nonempty, compact subsets $X, Y, Z \subset \R^n$,
\begin{align}
  \hd\paren{X, Z} \le \hd\paren{X, Y} + \hd\paren{Y, Z}.
\end{align}
See, e.g., Sec. 7.1 of~\cite{Sternberg2013Dynamical} for a proof.
\begin{lemma}[Restatement of~\cref{lemma:dist-support}]
  \label{lemma:dist-support-appendix}
  For convex compact sets $X,Y\subset\R^n$, we have
  \begin{equation}
    \hd\paren{X, Y} = \sup_{\mb u\in\bb S^{n-1}} \abs{h_X(\mb u) - h_Y(\mb u)},
  \end{equation}
  where $h_S(\mb u)\doteq \sup_{\mb x\in S}\<\mb x, \mb u\>$ is the support function associated with the set $S$.
\end{lemma}

\subsection{Random sets and their (selection) expectations}
\label{appendix:set-expectation}

Here we present the basics of random sets and the notion and properties of (selection) expectation, which is frequently used in this paper. Our treatment follows~\citep{Molchanov2013foundations}. A more comprehensive treatment of the subject can be found in the monograph~\citep{Molchanov2017Theory}.

We focus on closed random sets. Let $\mc F_d$ and $\mc K_d$ be the families of closed sets and compact sets in $\R^d$, respectively. Random closed sets are defined as maps from the probability space $\paren{\Omega, \mc A, P}$ to $\mc F_d$, or simply thought of as set-valued maps from $\Omega$ to $\R^d$.
\begin{definition}[Closed random sets]
A map $X: \Omega \mapsto \mc F_d$ is called a \emph{closed random set} (in $\R^d$) if $\set{\omega: X\paren{\omega} \cap \mc K \ne \emptyset}$ is measurable for each $K \in \mc K_d$.
\end{definition}
So the $\sigma$-algebra $\mc A$ for closed random sets is generated by the family of sets
\begin{align}
\set{F \in \mc F_d: \; F \cap K \ne \emptyset}, \; \forall\; K \in \mc K_d.
\end{align}
A random compact set is a closed random set taking compact values almost surely; random convex sets are defined similarly. Because the notion of Clarke subdifferential always induces convex compact sets, we focus on random convex compact sets in $\R^d$ below.

Typical examples of closed random sets include random singleton ($\set{x}$), random half-lines ($(-\infty, x]$), random intervals ($[x, y]$), random balls ($\bb B(x; y)$), where we assume $x, y$ are random variables in appropriate senses. See Section 1.1.1 of~\citet{Molchanov2017Theory} for more examples and typical random variables associated with random sets.  Random closed sets can also be constructed from basic operations: if $X$ is a random closed set, then the closed convex hull of $X$, $\alpha X$ for $\alpha$ a random variable, the closure of $X^c$ (i.e., $\mathrm{cl}\paren{X^c}$), the closure of the interior of $X$ (i.e., $\mathrm{cl}\paren{\mathrm{int}\paren{X}}$), and the boundary of $X$ (i.e., $\partial X$) are all random closed sets. Moreover, if $X, Y$ are random sets, $X \cap Y$, $X \cup Y$, and $\mathrm{cl}\paren{X+Y}$ are also random sets. Similar properties can be worked out for random convex sets and random compact sets. We note particularly that for a sequence of random compact sets $\set{X_i}_{i \in [K]}$, the Minkowski sum $\sum_{i=1}^K X_i$ is also a random compact set.

To define the (selection) expectation of a random set, we need the notion of selection. A random variable $\xi$ is said to be a
\emph{selection} of a random set $X$ if $\xi\in X$ almost surely. In
order to emphasize that $\xi$ (being a random variable) is measurable,
a selection is also called a \emph{measurable selection}. A possibly
empty random set does not have a selection: if $X(\omega)=\emptyset$,
then $\xi(\omega)$ is not allowed to take any value. Otherwise, the
fundamental selection theorem guarantees that there exists at
least one such selection:
\begin{theorem}[Fundamental selection theorem]
  Suppose a random closed set $X \subset \R^d$ is almost surely non-empty, then $X$ has a measurable
  selection.
\end{theorem}
With the selection theorem in hand, we can define the expectation of a
random set.
\begin{definition}[Selection expectation of a random set]
  The (selection) expectation of a random set $X$ in $\R^d$ is the closure of the set of all expectations of integrable selections of $X$, that is
  \begin{align} \label{eq:def_select_exp}
    \expect{X} \doteq {\rm cl}\set{\expect{\xi}:\xi~\textrm{is a measurable
        selection of}~X,~\expect{\norm{\xi}{}} < \infty} \subset \R^d.
  \end{align}
  In particular, $X$ is said to be integrable if $\expect{X}\neq
  \emptyset$, and integrably bounded if $\expect{\norm{X}{}} < \infty$. The $\mathrm{cl}$ in~\cref{eq:def_select_exp} is extraneous if $X$ is integrably bounded.
\end{definition}
If $X \subset \R^d$ is an intergrably bounded random compact set, then its expectation $\expect{X}$ is a random compact set in $\R^d$. The above definition defines the expectation based on all selection integrals, which is not ideal for performing calculation. Below, we present a crucial exchangeability theorem that enables us to delineate the set expectation through expectation of the associated support function.  Support functions are defined in~\cref{section:set-concentration}, and recall from convex analysis that convex sets can be completely characterized by their support functions.
\begin{theorem}[Exchangeability of set expectation and support
  function]  \label{thm:exchange_ex_supp}
  Suppose a random compact set $X$ is integrably bounded and the underlying
  probability space is non-atomic, then $\expect{X}$ is a \emph{convex} set
  and
  \begin{align}
    h_{\expect{X}}(\mb u) = \expect{h_X(\mb u)},~~\forall u\in\R^d.
  \end{align}
\end{theorem}
This theorem is often used together with the following basic properties of support functions: for convex compact sets $X, Y$, $h_{\alpha X} \paren{\mb u}= \alpha h_X\paren{\mb u}$ and $h_{X + Y}\paren{\mb u} = h_X\paren{\mb u} + h_Y\paren{\mb u}$ (support functions linearize Minkowski sums), $\hd\paren{X, Y} = \sup_{\mb u \in \bb S^{d-1}} \abs{h_X\paren{\mb u} - h_Y\paren{\mb u}}$, and also $\mathrm{rad}\paren{X} = \hd\paren{X, \set{0}} = \sup_{\mb  u \in \bb S^{n-1}} \abs{h_X\paren{\mb u}}$.

Furthermore, suppose $X$ and $Y$ are intergrably bounded random compact, convex sets. Then, we have the following natural equalities and inequalities: $\hd\paren{\bb E X, \bb E Y} \le \bb E \hd\paren{X, Y}$, $\norm{\bb E X}{} \le \bb E \norm{X}{}$ (here $\norm{\cdot}{}$ is the set radius), $\expect{X + Y} = \bb E X+ \bb E Y$ (linearity!), $\bb E X \subset \bb E Y$ provided that $X \subset Y$ almost surely, and $\bb E \innerprod{X}{\mb v} = \innerprod{\bb E X}{\mb v}$ for a fixed $\mb v$. Moreover,
\begin{theorem}  \label{thm:division_minkowski_sum}
Let $X_1, \dots, X_K$ be iid copies of $X$ which is an integrably bounded random convex, compact set in $\R^d$. Then,
\begin{align}
  \bb E\, \frac{1}{K}\sum_{i=1}^K X_i = \bb E X.
\end{align}
\end{theorem}
\begin{proof}
By our assumption, $\frac{1}{K}\sum_{i=1}^K X_i$ is convex and integrably bounded. So,
\begin{align}
  h_{\bb E\, \frac{1}{K}\sum_{i=1}^K X_i}\paren{\mb u}
  & = \frac{1}{K} \bb E h_{\sum_{i=1}^K X_i} \paren{\mb u}  \quad (\text{by~\cref{thm:exchange_ex_supp}}) \\
  & = \frac{1}{K} \bb E \sum_{i=1}^K h_{X_i} \paren{\mb u}  \quad (\text{support functions linearize Minkowski sums}) \\
  & = \frac{1}{K} \sum_{i=1}^K \bb E h_{X_i} \paren{\mb u} = \bb E h_{X}\paren{\mb u} = h_{\bb E X}\paren{\mb u}.
\end{align}
As both $\bb E\, \frac{1}{K}\sum_{i=1}^K X_i$ and $\bb E X$ are convex and compact, we conclude that $\bb E\, \frac{1}{K}\sum_{i=1}^K X_i = \bb E X$, completing the proof.
\end{proof}

Note that the above equality does not hold in general absent the convexity assumption on $X$. Finally, we have the following law of large numbers.
\begin{theorem}[Law of large numbers for random sets in $\R^d$]
  Let $X_1, \dots, X_K$ be iid copies of $X$ which is an integrably bounded random compact set in $\R^d$, and let $S_K \doteq \sum_{i=1}^K X_i$ for $K \ge 1$. Then,
  \begin{align}
    \hd\paren{K^{-1} S_K, \bb E X} \to 0 \; \text{almost surely as}\; n \to \infty,
  \end{align}
  and
  \begin{align}
    \bb E \hd\paren{K^{-1} S_K, \bb E X} \to 0 \; \text{as}\; n \to \infty.
  \end{align}
\end{theorem}

\subsection{Sub-gaussian random matrices and processes}
\begin{proposition}[Talagrand's comparison inequality, Corollary 8.6.3 and Exercise 8.6.5 of~\cite{Vershynin2018High}]  \label{prop:talagrand_comparison_set}
Let $\{X_x\}_{x \in T}$ be a zero-mean random process on $T \subset \R^n$. Assume that for  all $\mb x, \mb y \in T$ we have
\begin{align}
  \norm{X_{\mb x} - X_{\mb y}}{\psi_2} \le K \norm{\mb x -\mb  y}{}.
\end{align}
Then, for any $t > 0$,
\begin{align}
  \sup_{\mb x \in T} \abs{X_{\mb x}} \le CK \brac{w(T) + t \cdot \radi(T)}
\end{align}
with probability at least $1 - 2\exp\paren{-t^2}$. Here $w(T) \doteq \bb E_{\mb g \sim \mc N\paren{\mb 0, \mb I}}\sup_{x \in T} \innerprod{\mb x}{\mb g}$ is the Gaussian width of $T$ and $\mathrm{rad}(T) = \sup_{x \in T} \norm{\mb x}{}$ is the radius of $T$.
\end{proposition}

\begin{proposition}[Deviation inequality for sub-gaussian matrices, Theorem 9.1.1 and Exercise 9.1.8 of~\cite{Vershynin2018High}] \label{prop:matrix_dev_subgauss}
  Let $\mb A$ be an $n \times m$ matrix whose rows $\mb A^i$'s are independent, isotropic, and sub-gaussian random vectors in $\R^m$. Then for any subset $T \subset \R^m$, we have
  \begin{align}
    \prob{\sup_{\mb x \in T} \abs{\norm{\mb A \mb x}{} - \sqrt{n} \norm{\mb x}{}} > CK^2 \brac{w\paren{T} + t \cdot \radi\paren{T}}} \le 2\exp\paren{-t^2}.
\end{align}
Here $K = \max_{i} \norm{\mb A^i}{\psi_2}$.
\end{proposition}

\begin{fact}[Centering]
If $X$ is a sub-gaussian random variable, $X - \mathbb{E}[X]$ is also sub-gaussian with $\|X- \mathbb{E}[X]\|_{\psi_2} \leq C\|X\|_{\psi_2}$.
\end{fact}

\begin{fact}[Sum of independent sub-gaussians]
Let $X_1,\cdots, X_n$ be independent, zero-mean sub-gaussian random variables, $\sum_{i=1}^n X_i$ is also sub-gaussian with $\|\sum_{i=1}^n X_i \|_{\psi_2}^2 \leq C\sum_{i=1}^n\|X_i\|_{\psi_2}^2$.
\end{fact}

%% file: Sections/proof-pop-geometry.tex
\section{Proofs for~\cref{section:population-geometry}}
\subsection{Proof of~\cref{proposition:pop-val-grad}}
We have
\begin{equation}
  \expect{f}(\mb q) = \objscale \bb E_{\mb x \sim \bgt} \abs{\mb q^\top\mb x}
  = \objscale\bb E_\Omega \bb E_{\mb z \sim \normal\paren{\mb 0, \mb I}} \abs{\mb q_\Omega^\top\mb z_\Omega}
  = \bb E_\Omega \norm{\mb q_\Omega}{},
\end{equation}
where the last equality uses the fact that $\bb E_{Z\sim\normal(0, \sigma^2)}[|Z|]=\sqrt{2/\pi}\sigma$.

Moreover,
\begin{align}
  \partial\norm{\mb q_\Omega}{} =
  \begin{cases}
    \frac{\mb q_\Omega}{\norm{\mb q_\Omega}{}}, & \mb q_\Omega\neq 0 \\
    \set{\mb v_\Omega:\norm{\mb v_\Omega}{}\le 1}, & \mb q_\Omega = 0
  \end{cases},
\end{align}
and $\partial \expect{\norm{\mb q_\Omega}{}}=\expect{\partial \norm{\mb q_\Omega}{}}$ (as the sub-differential and expectation can be exchanged for convex functions~\citep{Hiriart-UrrutyLemarechal2001Fundamentals}), and we have
\begin{align}
  \partial \expect{f}(\mb q)
  = \partial \bb E_\Omega \norm{\mb q_\Omega}{}
  =\bb E_\Omega \partial \norm{\mb q_\Omega}{}
  = \bb E_\Omega
  \begin{cases}
    \frac{\mb q_\Omega}{\norm{\mb q_\Omega}{}}, & \mb q_\Omega\neq 0 \\
    \set{\mb v_\Omega:\norm{\mb v_\Omega}{}\le 1}, & \mb q_\Omega = 0
  \end{cases}.
\end{align}
Since $f$ is convex, by the same exchangeability result, we also have $\expect{\partial f(\mb q)}=\partial \expect{f}(\mb q)$. Now all $\sign\paren{\mb q^\top \mb x_i} \mb x_i$ are iid copies of $\sign\paren{\mb q^\top \mb x} \mb x$ with $\mb x \sim_\iid \bgt$, and they are all integrably bounded random convex, compact sets. We can invoke~\cref{thm:division_minkowski_sum} to conclude that
\begin{align}
  \expect{\partial f}\paren{\mb q} = \objscale \expect{\sign\paren{\mb q^\top \mb x} \mb x},
\end{align}
completing the proof.

\subsection{Proof of~\cref{proposition:pop-stat-points}}
We first show that points in the claimed set are indeed stationary points.  Taking $\mb v_\Omega=\mb 0$ in~\cref{equation:pop-grad} gives the subgradient choice $\expect{\mb q_\Omega/\norm{\mb q_\Omega}{}\indicator{\mb q_\Omega\neq 0}}$. For any $k \in [n]$,  let $\mb q\in\mc{S}$ with $\norm{\mb q}{0}=k$.
For all $j\in\supp(\mb q)$, we have
\begin{align}
  \mb e_j^T \expect{\partial f}(\mb q)
  &= \theta q_j\cdot \EOmega\, \frac{1}{\sqrt{q_j^2 + \norm{\mb q_{\Omega\setminus\set{j}}}{}^2}}   \\
  &= \theta q_j\cdot \brac{\sum_{i=1}^k \theta^{i-1} (1-\theta)^{k-i}\cdot \sqrt{\frac{k}{i}}} \\
  &= q_j\cdot \sum_{i=1}^k \theta^i(1-\theta)^{k-i}\cdot\frac{\sqrt{k}}{\sqrt{i}} \doteq c(\theta, k) q_j.
\end{align}
On the other hand, for all $j\notin\supp(\mb q)$, $\mb e_j^\top \expect{\mb q_\Omega/\norm{\mb q_\Omega}{2}\indicator{\mb q_\Omega\neq 0}} = 0$.
Therefore, we have that $\expect{\mb q_\Omega/\norm{\mb q_\Omega}{2}\indicator{\mb q_\Omega\neq 0}}=c(\theta,k)\mb q$, and so
\begin{equation}
  (\mb I - \mb q\mb q^\top)\expect{\partial f}(\mb q)=c(\theta,k)\mb q - c(\theta,k)\mb q = \mb 0.
\end{equation}
Therefore $\mb q\in\mc{S}$ is stationary.

To see that $\set{\pm\mb e_i:i\in[n]}$ are the global minimizers, note that for all $\mb q\in\bb S^{n-1}$, we have
\begin{equation}
  \expect{f}(\mb q) = \expect{\norm{\mb q_\Omega}{}} \ge \expect{\norm{\mb q_\Omega}{}^2} = \theta.
\end{equation}
Equality holds if and only if $\norm{\mb q_\Omega}{} \in\set{0,1}$ for all $\Omega$, which is only satisfied at $\mb q\in\set{\pm\mb e_i:i\in[n]}$. To see that the other $\mb q \in \mc S$ are saddles, we only need to show that for each such $\mb q$ there exists a tangent direction along which $\mb q$ is local maximizer.  Indeed, for any other $\mb q_0 \in \mc S$, there exists at least two non-zero entries (with equal absolute value): wlog assume that $q_1=q_n>0$ and consider the reparametrization in~\cref{appendix:reparam}. By the proof of~\cref{prop:inward_radial_grad}, it is obvious that $\expect{g}$ is differential in $\mb w(\mb q_0)$ direction with $D^c_{\mb w/\norm{\mb w}{}} \expect{g}\paren{\mb w(\mb q_0)} = 0$. Moreover, \cref{prop:asym_directional_curv} implies that $\expect{g}$ has a strict negative curvature in $\mb w\paren{\mb q_0}$ direction. Thus, $\mb w\paren{\mb q_0}$ is a local maximizer for $\expect{g}$ in $\mb w\paren{\mb q_0}$ direction, implying that $\expect{f}(\mb q)$ is locally maximized at $\mb q_0$ along the tangent direction $\brac{-\mb q_{-n}; \paren{1-q_n^2}/q_n}$. So $\mb q_0$ is a saddle point.

The other direction (all other points are not stationary) is implied by (the proof of)~\cref{theorem:pop-geometry}, which guarantees that $\mb 0\notin\expect{\rgrad f}(\mb q)$ whenever $\mb q\notin \mc{S}$. Indeed, as long as $\mb q\notin\mc{S}$, $\mb q$ has a maximum absolute value coordinate (say $n$) and another non-zero coordinate with strictly smaller absolute value (say $j$). For this pair of indices, the proof of~\cref{theorem:pop-geometry}(a) goes through for index $j$ (even if $\mb q\in\mc{S}_{\zeta_0}^{(n+)}$ does not necessarily hold because the max index might not be unique) and the quantity $\inf \innerprod{\expect{\partial_R f}\paren{\mb q}}{\mb e_j/q_j - \mb e_n/q_n}$ is strictly positive, which implies that $\mb 0\notin \expect{\rgrad f}(\mb q)$.

\subsection{Reparametrization}
\label{appendix:reparam}
For analysis purposes, we introduce the reparametrization $\mb w=\mb q_{1:(n-1)}$ in the region $\mc{S}_0^{(n+)}$, following~\citep{SunEtAl2015Complete} . With this reparametrization, the problem becomes
\begin{align} \label{eq:main_ncvx_re}
\mini_{\mb w \in \R^{n-1}} \; g\paren{\mb w} \doteq \objscale \frac{1}{m}\norm{\brac{\mb w; \sqrt{1-\norm{\mb w}{}^2}}^\top \mb X}{1} \quad \st \; \norm{\mb w}{} \le \sqrt{\frac{n-1}{n}}.
\end{align}
The constraint comes from the fact that $q_n\ge 1/\sqrt{n}$ and thus $\|\mb w\| \le \sqrt{(n-1)/n}$.

\begin{lemma} \label{lem:asymp_obj}
We have
\begin{align}
\bb E_{\mb x \sim_{iid} \bgt} g\paren{\mb w} =
\paren{1-\theta} \bb E_{\Omega} \norm{\mb w_{\Omega}}{} + \theta \bb E_{\Omega} \sqrt{1-\norm{\mb w_{\Omega^c}}{}^2}.
\end{align}
\end{lemma}
\begin{proof}
Direct calculation gives
\begin{align}
& \bb E_{\mb x \sim_{iid} \bgt} g\paren{\mb w} \\
=\; & \objscale\bb E_{\Omega \sim_{iid} \mathrm{Ber}(\theta), \omega \sim \mathrm{Ber}(\theta)} \bb E_{\mb z \sim_{iid} \normal(0, 1), z \sim \normal(0, 1)} \abs{\brac{\mb w; \sqrt{1-\norm{\mb w}{}^2}}^\top ([\Omega; \omega] \odot [\mb z; z])} \\
=\; & \objscale \paren{1-\theta} \bb E_{\Omega \sim_{iid} \mathrm{Ber}(\theta)} \bb E_{\mb z \sim_{iid} \normal(0, 1)} \abs{\mb w_{\Omega}^\top \mb z} \nonumber \\
& \quad + \objscale\theta \bb E_{\Omega \sim_{iid} \mathrm{Ber}(\theta)} \bb E_{\mb z \sim_{iid} \normal(0, 1), z \sim \normal(0, 1)} \abs{\mb w_{\Omega}^\top \mb z + \sqrt{1-\norm{\mb w}{}^2}z} \\
=\; & \paren{1-\theta}  \bb E_{\Omega \sim_{iid} \mathrm{Ber}(\theta)} \norm{\mb w_{\Omega}}{} + \theta \bb E_{\Omega \sim_{iid} \mathrm{Ber}(\theta)} \sqrt{1 - \norm{\mb w_{\Omega^c}}{}^2},
\end{align}
as claimed.
\end{proof}

\begin{lemma}[Negative-curvature region]  \label{prop:asym_directional_curv}
For all unit vector $\mb v \in \bb S^{n-1}$ and all $s \in (0, 1)$, let
\begin{align}
h_{\mb v}\paren{s} \doteq \expect{g}\paren{s\mb v}.
\end{align}
It holds that
\begin{align}
\nabla^2 h_{\mb v}\paren{s} \le - \theta \paren{1-\theta}.
\end{align}
In other words, for all $\mb w \ne \mb 0$, $\pm \mb w/\norm{\mb w}{}$ is a direction of negative curvature.
\end{lemma}
\begin{proof}
By~\cref{lem:asymp_obj},
\begin{align}
h_{\mb v}\paren{s} = \paren{1-\theta}  s\bb E_{\Omega} \norm{\mb v_{\Omega}}{} + \theta \bb E_{\Omega} \sqrt{1-s^2\norm{\mb v_{\Omega^c}}{}^2}.
\end{align}
For $s \in (0, 1)$, $h_{\mb v}(s)$ is twice differentiable, and we have
\begin{align}
\nabla^2 h_{\mb v}(s)
& = -\theta \bb E_{\Omega} \frac{\norm{\mb v_{\Omega^c}}{}^2}{\paren{1-s^2\norm{\mb v_{\Omega^c}}{}^2}^{3/2}} \\
& \le -\theta \bb E_{\Omega} \norm{\mb v_{\Omega^c}}{}^2 = -\theta \paren{1-\theta},
\end{align}
completing the proof.
\end{proof}

\begin{lemma}[Inward gradient]  \label{prop:inward_radial_grad}
  For any $\mb w$ with $\norm{\mb w}{}^2 + \norm{\mb w}{\infty}^2 \le 1$,
  \begin{align}
    D^c_{\mb w/\norm{\mb w}{}}\expect{g}(\mb w)
    \ge \theta \paren{1-\theta} \paren{1/\sqrt{1+\norm{\mb w}{\infty}^2/\norm{\mb w}{}^2} - \norm{\mb w}{}}.
  \end{align}
\end{lemma}
\begin{proof}
For any unit vector $\mb v \in \R^{n-1}$, define $h_{\mb v}\paren{t} \doteq \expect{g}\paren{t \mb v}$ for $t \in (0, 1)$. We have from~\cref{lem:asymp_obj}
\begin{align}
  h_{\mb v}\paren{t}
  = \paren{1-\theta} t \bb E_{\Omega} \norm{\mb v_\Omega}{} + \theta \bb E_{\Omega} \sqrt{1-t^2 \norm{\mb v_\Omega^c}{}^2}.
\end{align}
Moreover,
\begin{align}
  \nabla_t h_{\mb v}\paren{t}
  & = \paren{1-\theta} \bb E_{\Omega} \norm{\mb v_\Omega}{} - \theta \bb E_\Omega \frac{t \norm{\mb v_\Omega^c}{}^2}{\sqrt{1 - t^2 \norm{\mb v_\Omega^c}{}^2}} \\
  & = \paren{1-\theta} \bb E_{\Omega} \frac{\norm{\mb v_\Omega}{}^2}{\norm{\mb v_\Omega}{}} - \theta \bb E_\Omega \frac{t \norm{\mb v_\Omega^c}{}^2}{\sqrt{1 - t^2 \norm{\mb v_\Omega^c}{}^2}}   \quad (\text{assuming $\frac{0}{0} \doteq 0$}) \\
  & = \paren{1-\theta}\sum_{i=1}^{n-1} \bb E_{\Omega} \frac{v_i^2 \indicator{i \in \Omega}}{\sqrt{v_i^2 \indicator{i \in \Omega} + \norm{\mb v_{\Omega \setminus\set{i}}}{}^2}} - \theta \sum_{i = 1}^{n-1}\bb E_\Omega \frac{t v_i^2 \indicator{i \notin \Omega}}{\sqrt{1 - t^2 v_i^2 \indicator{i \notin \Omega} - t^2 \norm{\mb v_{\Omega^c\setminus\set{i}}}{}^2}} \\
  & = \theta \paren{1 -\theta} \sum_{i=1}^{n-1} \bb E_\Omega \brac{\frac{v_i^2}{\sqrt{v_i^2 +\norm{\mb v_{\Omega \setminus\set{i}}}{}^2 }}   -  \frac{t v_i^2 }{\sqrt{1 - t^2 v_i^2 - t^2 \norm{\mb v_{\Omega^c\setminus\set{i}}}{}^2}}} \\
  & = \theta \paren{1 -\theta} t \sum_{i=1}^{n-1} v_i^2 \bb E_\Omega \brac{\frac{1}{\sqrt{t^2 v_i^2 + t^2 \norm{\mb v_{\Omega \setminus\set{i}}}{}^2 }}   -  \frac{ 1 }{\sqrt{1 - t^2 v_i^2 - t^2 \norm{\mb v_{\Omega^c\setminus\set{i}}}{}^2}}} \\
  & = \theta \paren{1 -\theta} t \sum_{i=1}^{n-1} v_i^2 \bb E_\Omega \brac{\frac{1}{\sqrt{t^2 v_i^2 + t^2 \norm{\mb v_{\Omega \setminus\set{i}}}{}^2 }}   -  \frac{ 1 }{\sqrt{1-t^2\|\mb v\|^2 + t^2 \norm{\mb v_{\Omega\setminus\set{i}}}{}^2}}}.
\end{align}
We are interested in the regime of $t$ so that
\begin{align}
  1 - t^2\norm{\mb v}{}^2 \ge t^2 \norm{\mb v}{\infty}^2 \Longrightarrow t \le 1/\sqrt{1+\norm{\mb v}{\infty}^2}.
\end{align}
So $\nabla_t h_{\mb v}\paren{t} \ge 0$ holds always for $t \le 1/\sqrt{1+\norm{\mb v}{\infty}^2}$.  By~\cref{prop:asym_directional_curv},  $\nabla^2 h_{\mb v}\paren{t} \le -\theta \paren{1-\theta}$ over $t \in (0, 1)$, which implies
\begin{align}
\innerprod{\nabla_t h_{\mb v}\paren{t_1} - \nabla_t h_{\mb v}\paren{t_2}}{t_1 - t_2} \le -\theta\paren{1-\theta} \paren{t_1 - t_2}^2.
\end{align}
Taking $t_1 =  1/\sqrt{1+\norm{\mb v}{\infty}^2}$ and considering $t_2 \in [0, t_1]$, we have
\begin{align}
  \nabla_t h_{\mb v}\paren{t_2}
  \ge  \nabla_t h_{\mb v}\paren{t_1} + \theta\paren{1-\theta} \paren{t_1 - t_2}
  \ge \theta\paren{1-\theta} \paren{1/\sqrt{1+\norm{\mb v}{\infty}^2} - t_2} .
\end{align}
For any $\mb w$, the above argument implies that $\expect{g}\paren{\mb w}$ is differentiable in the $\mb w/\norm{\mb w}{}$ direction, with $D_{\mb w/\norm{\mb w}{}} \expect{g}\paren{\mb w} = \nabla h_{\mb w/\norm{\mb w}{}}\paren{\norm{\mb w}{}}$. Moreover, since $\expect{g}\paren{\mb w}$ is differentiable along the $\mb w/\norm{\mb w}{}$ direction, $D^c_{\mb w/\norm{\mb w}{}} \expect{g}\paren{\mb w} = D_{\mb w/\norm{\mb w}{}} \expect{g}\paren{\mb w}$, completing the proof.
\end{proof}

\subsection{Proof of~\cref{theorem:pop-geometry}}
We first show~\cref{equation:pop-grad-inward}. For $\mb e_j$ with $q_j \ne 0$,
\begin{multline}
  \frac{1}{q_j} \mb e_j^\top \expect{\partial f}\paren{\mb q}
  = \frac{1}{q_j} \mb e_j^\top \bb E_{\Omega} \brac{\frac{\mb q_{\Omega}}{\norm{\mb q_{\Omega}}{}} \indicator{\mb q_{\Omega} \ne \mb 0}  + \set{\mb v_{\Omega}: \norm{\mb v_{\Omega}}{} \le 1} \indicator{\mb q_{\Omega} = \mb 0}  }\\
  = \frac{1}{q_j} \bb E_{\Omega} \brac{\frac{\innerprod{\mb q_{\Omega}}{\mb e_j}}{\norm{\mb q_{\Omega}}{}} \indicator{\mb q_{\Omega} \ne \mb 0} }
  = \frac{1}{q_j} \theta q_j \bb E_{\Omega} \brac{\frac{1}{\norm{\mb q_{\Omega}}{}} \indicator{j \in \Omega} }
  = \theta \bb E \brac{    \frac{1}{\sqrt{q_j^2 + \norm{\mb q_{\Omega\setminus \set{j}}}{}^2}}   }.
\end{multline}
So for all $j$ with $q_j \ne 0$, we have
\begin{align}
  & \innerprod{\expect{\partial_R f}\paren{\mb q}}{\frac{1}{q_j} \mb e_j - \frac{1}{q_n} \mb e_n} \nonumber \\
  =\; & \innerprod{\paren{\mb I - \mb q \mb q^\top} \expect{\partial f}\paren{\mb q}}{\frac{1}{q_j} \mb e_j - \frac{1}{q_n} \mb e_n} \\
  =\; & \innerprod{\expect{\partial f}\paren{\mb q}}{\frac{1}{q_j} \mb e_j - \frac{1}{q_n} \mb e_n} \\
  =\; & \theta \bb E_{\Omega} \brac{    \frac{1}{\sqrt{q_j^2 + \norm{\mb q_{\Omega\setminus \set{j}}}{}^2}}   } - \theta \bb E_{\Omega} \brac{    \frac{1}{\sqrt{q_n^2 + \norm{\mb q_{\Omega\setminus \set{n}}}{}^2}}   } \\
  =\; &  \theta^2 \bb E_{\Omega} \brac{    \frac{1}{\sqrt{q_j^2 + q_n^2 + \norm{\mb q_{\Omega\setminus \set{j, n}}}{}^2}}   } + \theta \paren{1-\theta} \bb E_{\Omega} \brac{    \frac{1}{\sqrt{q_j^2 +  \norm{\mb q_{\Omega\setminus \set{j, n}}}{}^2}}   } \nonumber \\
  & \quad - \theta^2 \bb E_{\Omega} \brac{    \frac{1}{\sqrt{q_j^2 + q_n^2 + \norm{\mb q_{\Omega\setminus \set{j, n}}}{}^2}}   } - \theta \paren{1-\theta} \bb E_{\Omega} \brac{    \frac{1}{\sqrt{q_n^2 +  \norm{\mb q_{\Omega\setminus \set{j, n}}}{}^2}}   }\\
  =\;  & \theta \paren{1-\theta} \bb E_{\Omega} \brac{ \frac{1}{\sqrt{q_j^2 +  \norm{\mb q_{\Omega\setminus \set{j, n}}}{}^2}}  -  \frac{1}{\sqrt{q_n^2 +  \norm{\mb q_{\Omega\setminus \set{j, n}}}{}^2}} } \\
  =\;  & \theta \paren{1-\theta} \bb E_{\Omega} \int_{q_j^2}^{q_n^2} \frac{1}{2\paren{t + \norm{\mb q_{\Omega\setminus \set{j, n}}}{}^2}^{3/2}}\; dt \\
  \ge\; & \theta\paren{1-\theta} \frac{1}{2}\paren{q_n^2 - q_j^2}
          \ge  \frac{1}{2}\theta\paren{1-\theta} \paren{q_n^2 - \norm{\mb q_{-n}}{\infty}^2}
          \ge \frac{1}{2}\theta\paren{1-\theta} \frac{\zeta_0}{1+\zeta_0} q_n^2 \\
          &\ge \frac{1}{2n}\theta\paren{1-\theta} \frac{\zeta_0}{1+\zeta_0},
\end{align}
as claimed.

We now turn to showing~\cref{equation:pop-grad-star} using the reparametrization in~\cref{appendix:reparam}. We have
\begin{align}
  \innerprod{\partial_R f\paren{\mb q}}{q_n \mb q - \mb e_n}
  = \innerprod{\partial f\paren{\mb q}}{q_n \mb q - \mb e_n}
  = q_n \innerprod{\partial g\paren{\mb w}}{\mb w},
\end{align}
where the second equality follows by differentiating $g$ via the chain rule (\cref{thm:subdiff_chain_rule}). Thus,
\begin{align}
  \innerprod{\expect{\partial_R f}\paren{\mb q} }{q_n \mb q - \mb e_n}
  = q_n \innerprod{\expect{\partial g}\paren{\mb w} }{\mb w}.
\end{align}
By~\cref{prop:inward_radial_grad},
\begin{align}
  D^c_{-\mb w/\norm{\mb w}{}} \expect{g}\paren{\mb w} \le -\theta \paren{1-\theta} \paren{1/\sqrt{1 + \norm{\mb w}{\infty}^2/\norm{\mb w}{}^2} - \norm{\mb w}{}},
\end{align}
as $\expect{g}\paren{\mb w}$ is differentiable in the $\mb w/\norm{\mb w}{}$ direction. Now by the definition of Clarke subdifferential,
\begin{align}
D^c_{-\mb w/\norm{\mb w}{}} \expect{g}\paren{\mb w}
=   \sup \innerprod{\expect{\partial g}\paren{\mb w}}{-\frac{\mb w}{\norm{\mb w}{}}},
\end{align}
implying that
\begin{align}
   q_n \innerprod{\expect{\partial g}\paren{\mb w}}{\mb w}
   \ge \norm{\mb w}{} \theta \paren{1-\theta} \cdot q_n \paren{\frac{\norm{\mb w}{}}{\sqrt{\norm{\mb w}{}^2 + \norm{\mb w}{\infty}^2}} - \norm{\mb w}{}}.
\end{align}
For each radial direction $\mb v \doteq \mb w/\norm{\mb w}{}$, consider points of the form $t \mb v$ with $t \le 1/\sqrt{1+\norm{\mb v}{\infty}^2}$. Obviously, the function
\begin{align}
  \hbar\paren{t} \doteq q_n\paren{t \mb v} \paren{\frac{\norm{t\mb v}{}}{\sqrt{\norm{t \mb v}{}^2 + \norm{t \mb v}{\infty}^2}} - \norm{t\mb v}{}} = q_n\paren{t \mb v} \paren{\frac{1}{\sqrt{1+\norm{\mb v}{\infty}^2}} - t}
\end{align}
is monotonically decreasing wrt $t$. Thus, to derive a lower bound, it is enough to consider the largest $t$ allowed. In $\mc S_{\zeta_0}^{(n+)}$, the limit amounts to requiring $q_n^2 /\norm{\mb w}{\infty}^2 = 1 + \zeta_0$,
\begin{align}
  1 - t_0^2 = t_0^2 \norm{\mb v}{\infty}^2 \paren{1+\zeta_0}
  \Longrightarrow t_0 = \frac{1}{\sqrt{1+\paren{1+\zeta_0} \norm{\mb v}{\infty}^2}}.
\end{align}
So for any fixed $\mb v$ and all allowed $t$ for points in $\mc S_{\zeta_0}^{(n+)}$, a uniform lower bound is
\begin{align}
  & \quad q_n\paren{t_0 \mb v} \paren{\frac{1}{\sqrt{1+\norm{\mb v}{\infty}^2}} - t_0} \\
  & \ge \frac{1}{\sqrt{n}} \paren{\frac{1}{\sqrt{1+\norm{\mb v}{\infty}^2}} - \frac{1}{\sqrt{1+\paren{1+\zeta_0} \norm{\mb v}{\infty}^2}}}
  \ge \frac{1}{8\sqrt{n}} \zeta_0 \norm{\mb v}{\infty}^2 \ge \frac{1}{8}\zeta_0 n^{-3/2}.
\end{align}
So we conclude that for all $\mb q \in \mc S_{\zeta_0}^{(n+)}$,
\begin{align}
  \innerprod{\expect{\partial f}\paren{\mb q}}{ q_n \mb q - \mb e_n} \ge \frac{1}{8} \theta \paren{1-\theta}\zeta_0 n^{-3/2} \norm{\mb w}{} = \frac{1}{8} \theta \paren{1-\theta}\zeta_0 n^{-3/2} \norm{\mb q_{-n}}{},
\end{align}
completing the proof.

%% file: Sections/proof-emp-geometry.tex
\section{Proofs for~\cref{section:empirical-geometry}}

\subsection{Covering in the $\dexp$ metric}
For any $\theta \in (0, 1)$, define
\begin{align} \label{eq:dexp_metric}
  \mathrm{d}_{\bb E, \theta} \paren{\mb p, \mb q} \doteq \expect{\indicator{\sign\paren{\mb p^\top \mb x} \ne \sign\paren{\mb q^\top \mb x}} }    \quad \text{with} \; \mb x \sim_{iid} \bgt.
\end{align}
We stress that this notion always depends on $\theta$, and we will omit the subscript $\theta$ when no confusion arises. This indeed defines a metric on subsets of $\bb S^{n-1}$.

\begin{lemma}  \label{lemma:exp_metric}
Over any subset of $\bb S^{n-1}$ with a consistent support pattern,
$\dexp$ is a valid metric.
\end{lemma}
\begin{proof}
Recall that $\angle\paren{\mb x, \mb y} \doteq \arccos\innerprod{\mb x}{\mb y}$ defines a valid metric on $\bb S^{n-1}$.\footnote{This fact can be proved either directly, see, e.g., page 12 of this online notes: \url{http://www.math.mcgill.ca/drury/notes354.pdf}, or by realizing that the angle equal to the geodesic length, which is the Riemmannian distance over the sphere; see, e.g., Riemannian Distance of Chapter 5 of the book~\cite{Oneill1983Semi}.} In particular, the triangular inequality holds. For $\dexp$ and $\mb p, \mb q \in \bb S^{n-1}$ with the same support pattern, we have
\begin{align}
\dexp\paren{\mb p, \mb q}
  & = \bb E \; \indicator{\sign\paren{\mb p^\top \mb x} \ne \sign\paren{\mb q^\top \mb x}}  \\
  & = \EOmega \bb E_{\mb z \sim \normal\paren{\mb 0, \mb I}} \indicator{\sign\paren{\mb p_\Omega^\top \mb z} \ne \sign\paren{\mb q_\Omega^\top \mb z}} \\
  & = \EOmega \paren{\bb E_{\mb z} \indicator{\mb p_\Omega^\top \mb z \mb q_\Omega \mb z < 0 } +  \bb E_{\mb z} \indicator{\mb p_\Omega^\top \mb z = 0\; \text{or}\; \mb q_\Omega^\top \mb z = 0, \text{not both} }}  \\
  & = \EOmega \paren{\bb E_{\mb z} \indicator{\mb p_\Omega^\top \mb z \mb q_\Omega \mb z < 0 }}  \\
  & = \frac{1}{\pi}\EOmega\angle\paren{\mb p_{\Omega}, \mb q_{\Omega}},
    \label{eq:dexp_angle_equiv}
\end{align}
where we have adopted the convention that $\angle\paren{\mb 0, \mb v} \doteq 0$ for any $\mb v$. It is easy to verify that $\dexp \paren{\mb p, \mb q} = 0 \Longleftrightarrow \mb p = \mb q$, and $\dexp\paren{\mb p, \mb q} = \dexp\paren{\mb q, \mb p}$. To show the triangular inequality, note that for any $\mb p, \mb q$ and $\mb r$ with the same support pattern, $\mb p_{\Omega}$, $\mb q_{\Omega}$, and $\mb r_{\Omega}$ are either identically zero, or all nonzero. For the former case,
\begin{align}
\angle\paren{\mb p_{\Omega}, \mb q_{\Omega}} \le \angle\paren{\mb p_{\Omega}, \mb r_{\Omega}} + \angle\paren{\mb q_{\Omega}, \mb r_{\Omega}}
\end{align}
holds trivially. For the latter, since $\angle\paren{\cdot, \cdot}$ obeys the triangular inequality uniformly over the sphere,
\begin{align}
\angle\paren{\frac{\mb p_{\Omega}}{\norm{\mb p_{\Omega}}{}}, \frac{\mb q_{\Omega}}{\norm{\mb q_{\Omega}}{}}} \le \angle\paren{\frac{\mb p_{\Omega}}{\norm{\mb p_{\Omega}}{}}, \frac{\mb r_{\Omega}}{\norm{\mb r_{\Omega}}{}}} + \angle\paren{\frac{\mb q_{\Omega}}{\norm{\mb q_{\Omega}}{}}, \frac{\mb r_{\Omega}}{\norm{\mb r_{\Omega}}{}}},
\end{align}
which implies
\begin{align}
\angle\paren{\mb p_{\Omega}, \mb q_{\Omega}} \le \angle\paren{\mb p_{\Omega}, \mb r_{\Omega}} + \angle\paren{\mb q_{\Omega}, \mb r_{\Omega}}.
\end{align}
So
\begin{align}
\EOmega \angle\paren{\mb p_{\Omega}, \mb q_{\Omega}} \le \EOmega \angle\paren{\mb p_{\Omega}, \mb r_{\Omega}} + \EOmega \angle\paren{\mb q_{\Omega}, \mb r_{\Omega}},
\end{align}
completing the proof.
\end{proof}

\begin{lemma}[Vector angle inequality] \label{lemma:angle_ineqn}
  For $n\ge 2$, consider $\mb u, \mb v \in \R^n$ so that $\angle\paren{\mb u, \mb v} \le \pi/2$. It holds that\footnote{Fedor Petrov has helped with the proof on MathOverflow: \url{https://mathoverflow.net/questions/306156/controlling-angles-between-vectors-using-sum-of-subvector-angles}. }
  \begin{align}
    \angle \paren{\mb u, \mb v} \le \sum_{\Omega \in \binom{[n]}{2}} \angle \paren{\mb u_{\Omega}, \mb v_{\Omega}}.
  \end{align}
\end{lemma}
\begin{proof}
The inequality holds trivially when either of $\mb u, \mb v$ is zero. Suppose they are both nonzero and wlog assume both are normalized, i.e., $\norm{\mb u}{} = \norm{\mb v}{} = 1$. Then,
\begin{align}
  \sin^2 \angle\paren{\mb u, \mb v}
  & = 1 - \cos^2 \angle\paren{\mb u, \mb v} \\
  & = \norm{\mb u}{}^2 \norm{\mb v}{}^2 - \innerprod{\mb u}{\mb v}^2 \\
  & = \sum_{i, j: j > i} \paren{u_i v_j - u_j v_i}^2 \quad (\text{Lagrange's identity})
  \\
  & =\sum_{\Omega \in \binom{[n]}{2}} \norm{\mb u_{\Omega}}{}^2 \norm{\mb v_{\Omega}}{}^2 - \innerprod{\mb u_{\Omega}}{\mb v_{\Omega}}^2  \\
  & = \sum_{\Omega \in \binom{[n]}{2}} \norm{\mb u_{\Omega}}{}^2 \norm{\mb v_{\Omega}}{}^2 \sin^2 \angle\paren{\mb u_{\Omega}, \mb v_{\Omega}} \\
  & \le \sum_{\Omega \in \binom{[n]}{2}}  \sin^2 \angle\paren{\mb u_{\Omega}, \mb v_{\Omega}} .
\end{align}
If $\sum_{\Omega \in \binom{[n]}{2}} \angle \paren{\mb u_{\Omega}, \mb v_{\Omega}} > \pi/2$, the claimed inequality holds trivially, as $\angle\paren{\mb u, \mb v} \le \pi/2$ by our assumption. Suppose  $\sum_{\Omega \in \binom{[n]}{2}} \angle \paren{\mb u_{\Omega}, \mb v_{\Omega}} \le \pi/2$. Then,
\begin{align}
  \sum_{\Omega \in \binom{[n]}{2}}  \sin^2 \angle\paren{\mb u_{\Omega}, \mb v_{\Omega}}  \le \sin^2 \sum_{\Omega \in \binom{[n]}{2}} \angle \paren{\mb u_{\Omega}, \mb v_{\Omega}}
\end{align}
by recursive application of the following inequality: $\forall\; \theta_1, \theta_2 \in [0, \pi/2]$ with $\theta_1 + \theta_2 \le \pi/2$,
\begin{align}
  \sin^2\paren{\theta_1 + \theta_2} = \sin^2 \theta_1 + \sin^2 \theta_2 + 2\sin \theta_1 \sin \theta_2 \cos\paren{\theta_1 + \theta_2} \ge \sin^2 \theta_1 + \sin^2 \theta_2.
\end{align}
So we have that when $\sum_{\Omega \in \binom{[n]}{2}} \angle \paren{\mb u_{\Omega}, \mb v_{\Omega}} \le \pi/2$,
\begin{align}
  \sin^2 \angle\paren{\mb u, \mb v} \le   \sin^2 \sum_{\Omega \in \binom{[n]}{2}}  \angle\paren{\mb u_{\Omega}, \mb v_{\Omega}} \Longrightarrow \angle\paren{\mb u, \mb v} \le  \sum_{\Omega \in \binom{[n]}{2}}  \angle\paren{\mb u_{\Omega}, \mb v_{\Omega}},
\end{align}
as claimed.
\end{proof}

\begin{lemma}[Covering in maximum length-2 angles]
  \label{lemma:cover-length-2}
  For any $\eta\in(0,1/3)$, there exists a subset $\mc{Q}\subset \bb S^{n-1}$ of size at most $(5n\log(1/\eta)/\eta)^{2n-1}$ satisfying the following: for any $\mb p\in\bb S^{n-1}$, there exists some $\mb q\in \mc{Q}$ such that $\supp\paren{\mb q} = \supp\paren{\mb p}$, and $\angle\paren{\mb p_\Omega, \mb q_\Omega} \le \eta$ for all $\Omega\subset[n]$ with $|\Omega|\le 2$.
\end{lemma}
\begin{proof}
  Define
  \begin{equation}
    \d2(\mb p, \mb q) = \max_{|\Omega|\le 2}\angle\paren{\mb p_\Omega, \mb q_\Omega},
  \end{equation}
  our goal is to give an $\eta$-covering of $\bb S^{n-1}$ in the $\d2$ metric. We will first show how to cover the set
  \begin{align}
    A_n = \set{\mb p\in\bb S^{n-1}:\mb p\ge 0,~0<p_1\le p_2\le \dots\le p_n},
  \end{align}
  and then we will show how to extend it to the claimed covering result for $\bb S^{n-1}$ by symmetry argument.

We first bound the covering number of $A_n$  (denoted as $N\paren{A_n}$) by induction. Write $C_\eta \doteq \frac{5\log\paren{1/\eta}}{\eta}$.  Obviously, $N\paren{A_1} = 1 = \paren{C_\eta}^{1-1}$ and $N\paren{A_2}\le \frac{\pi/2}{\eta} \le \paren{2 C_\eta}^{2-1}$. Suppose that
\begin{align}
N\paren{A_{n'}} \le \paren{C_\eta n'}^{n'-1}
\end{align}
holds for all $n' \le n-1$, and let $\mc C_{n'}$ be the corresponding covering sets. We next construct a covering for $A_n$ ($n \ge 3$). To this end, we partition $A_n$ into two sets: write $R \doteq 1/\eta$, and consider
\begin{align}
  A_{n, R} \doteq \set{\mb p \in A_n: \; \frac{p_{i+1}}{p_i} \le R, \forall\; i \in [n-1]}
\end{align}
and its complement $A_n \setminus A_{n, R}$.

\paragraph{Cover the slowly-varying set $A_{n, R}$}
Note that $R = 1/\eta \ge 3$ when $\eta \le 1/3$.  Let $R=r^k$ for some $r\ge 1$ and $k$ to be determined. Consider the set
  \begin{equation}
    \mc{Q}_{r,k} = \set{\mb q\in A_n: \; \frac{q_{i+1}}{q_i}\in\set{1,r,r^2,\dots,r^{\floor{k}}}\; \forall\, i \in [n-1]} .
  \end{equation}
  We claim that $\mc{Q}_{r,k}$ with properly chosen $(r,k)$ gives a covering of $A_{n, R}$.
  Indeed, we can decompose $[1,R]$ into intervals $[1,r),[r,r^2),\dots,[r^{\floor{k}},R]$. For any $\mb p \in A_{n, R}$ and any $i \in [n-1]$, the consecutive ratio $q_{i+1}/q_i$ must fall into one of these intervals. We can choose $\mb q$ so that for each $i \in [n-1]$, $q_{i+1}/q_i$ is the left endpoint of the interval corresponding to $p_{i+1}/p_i$. Such a $\mb q$ lies in $\mc Q_{r, k}$ and satisfies
  \begin{align}
    \frac{p_{i+1}/p_i}{q_{i+1}/q_i} \in [1, r)   \; \forall\; i \in [n-1].
  \end{align}
  By multiplying these bounds, we obtain that for all $1\le i<j\le n$,
  \begin{equation}
    \frac{p_j/p_i}{q_j/q_i} \in [1, r^{n-1}).
  \end{equation}
  Taking $r = 1+\eta/\paren{2n}$, we have $r^{n-1}=\paren{1+\eta/\paren{2n}}^{n-1}\le \exp(\eta/2) \le 1+\eta$ (the last inequality holds whenever $\eta \le 1$). Therefore, for all $i < j$, we have
  \begin{align}
  \frac{p_j/p_i}{q_j/q_i} \in [1, 1+\eta),
  \end{align}
  which further implies that $\angle\paren{(p_i,p_j), (q_i,q_j)}\le \eta$ by~\cref{lemma:two-dim-angle}. Thus, we have for all $|\Omega|\le 2$ that
  $\angle\paren{\mb p_\Omega, \mb q_\Omega}\le \eta$ (The size-1 angles are all zero as we have sign match), and $\mc Q_{r, k}$ with $r = 1 + \eta/\paren{2n}$ and $k = \log R/\log r$ constitutes an $\eta$-covering for $A_{n, R}$.

  For this choice of $r$, we have $k=\log R/\log r$ and thus
  \begin{equation}
    \abs{\mc{Q}_{r,k}}
    = \ceil{k}^{n-1}
    = \ceil{\frac{\log R}{\log r}}^{n-1}
    =\ceil{\frac{\log(1/\eta)}{\log(1+\eta/(2n))}}^{n-1} \le \paren{\frac{4n\log(1/\eta)}{\eta}}^{n-1} \doteq \wtilde{N}_n,
  \end{equation}
  where to obtain the last inequality we have used the fact $\eta/\paren{2n} \le 1/18$. We have $N(A_{n,R})\le \wtilde{N}_n$.

\paragraph{Cover the spiky set $A_n \setminus A_{n, R}$}
We now construct a covering for $A_n\setminus A_{n,R}$. Consider the set
\begin{align}
  \mc W \doteq \bigcup_{i=1}^{n-1}  \set{\brac{\alpha\mb u; \beta\mb
  v}: \; \paren{\mb u, \mb v} \in \mc C_i \times \mc C_{n-i}, \,
  \alpha^2 + \beta^2 = 1, \, \frac{\beta v_1}{\alpha u_i} = 2R, \,
  \alpha, \beta > 0} \subset A_n \setminus A_{n, R}.
\end{align}
For any $\mb p\in A_n\setminus A_{n,R}$, there exists some $i$ such that $p_{i+1}/p_i\in (R,\infty)$. As $\mb p$ is sorted, we have that
\begin{equation}
  \frac{p_{i+j}}{p_{i-\ell}} \ge R~~\text{for all}~j\ge 1,~\ell\ge 0.
\end{equation}
This implies that for all $j \ge 1$ and all $\ell \ge 0$,
\begin{align}
  \angle \paren{[p_{i - \ell}; p_{i+j}], [1; 0]} \ge \arctan R \ge \pi/2 - \eta.
\end{align}
Any $\mb q \in A_n \setminus A_{n, R}$ satisfying $q_{i+1}/q_i = 2R$ has similar property, and obeys
\begin{align}
    \angle\paren{[p_{i-\ell},p_{i+j}], [q_{i-\ell}, q_{i+j}]} \le \pi/2 - (\pi/2 - \eta) = \eta~~\text{for all}~j\ge 1,~\ell\ge 0.
\end{align}
For length-$2$ subvectors in $\mb  p_{1:i}$ and $\mb p_{i+1:n}$, by
the inductive hypothesis, we can find vector $\mb u \in \mc C_i$ and $\mb v \in \mc C_{n-i}$, such that
\begin{align}
  \d2 \paren{\mb  p_{1:i}, \mb u} \le \eta, \quad   \d2 \paren{\mb  p_{i+1:n}, \mb v} \le \eta.
\end{align}
So for any $\mb p \in A_n \setminus A_{n, R}$, any vector $\mb q = [\alpha \mb u, \beta \mb v]$ with $\alpha, \beta > 0$ and $\beta v_1/\paren{\alpha u_i} = 2R$ obeys $\d2 \paren{\mb p, \mb q} \le \eta$. Further requiring $\alpha^2 + \beta^2 = 1$ ensures $\mb q \in A_n \setminus A_{n, R}$, and in fact $\mb q \in \mc W$. So $\mc W$ is an $\eta$-covering of $A_n \setminus A_{n, R}$ in $\d2$. Moreover, we have
\begin{align}
  N(A_n \setminus A_{n,R})
  \le \abs{\mc W}
  \le \sum_{i=1}^{n-1} N(A_i) N(A_{n-i}).
\end{align}

\paragraph{Putting things together}
The set $\mc Q_{r, k} \cup \mc W$ forms an $\eta$-covering of $A_n$ in $\d2$. By our inductive hypothesis,
  \begin{align}
    \quad N(A_n)
    & \le N(A_{n,R}) + N(A_n\setminus A_{n,R}) \\
    & \le \wtilde{N}_n + \sum_{i=1}^{n-1} N(A_i) N(A_{n-i}) \\
    & \le \paren{\frac{4n\log(1/\eta)}{\eta}}^{n-1} + \sum_{i=1}^{n-1} (C_\eta i)^{i-1} (C_\eta (n-i))^{n-i-1} \\
    &\le \paren{\frac{4}{5}}^{n-1} (C_\eta n)^{n-1} + (n-1) \cdot C_\eta^{n-2} n^{n-2} \\
    &\le \paren{\paren{\frac{4}{5}}^{n-1} + \frac{1}{C_\eta}} (C_\eta n)^{n-1} \le (C_\eta n)^{n-1}.
  \end{align}

\paragraph{Cover $\bb S^{n-1}$ by symmetry argument}
Let $\Pi_n$ be a signed permutation operator. Then, if $\mc C_n$ is a covering set for $A_n$, $\Pi \mc C_n$ is a covering set for $\Pi A_n$. So the set of fully dense vectors is covered by the set $\bigcup_{\Pi} \Pi \mc C_n$ whose size is bounded by $2^n \cdot n! \cdot\paren{C_\eta n}^{n-1}$. Vectors that are not fully dense are similarly covered separately according to their support patterns, entailing lower-dimensional problems. Considering all configurations, we have
  \begin{multline}
    N(\bb S^{n-1})
    \le \sum_{i=1}^n \binom{n}{i} 2^i \cdot i! \cdot \paren{C_\eta i}^{i-1}
    \le 3^n \cdot n! \cdot \paren{C_\eta n}^{n-1} \\
    \le \paren{3n}^n \cdot \paren{\frac{5\log \paren{1/\eta}}{\eta}n}^{n-1}
    \le \paren{\paren{\frac{5\log \paren{1/\eta}}{\eta}n}^{n-1}}^{2n-1},
  \end{multline}
completing the proof.
\end{proof}

\begin{lemma}[Covering number in the $\dexp$ metric]
  \label{lemma:cover-cheaper}
  Assume $n\ge 3$. There exists a universal constant $C>0$ such that for any $\eps\in(0,1)$, $\bb S^{n-1}$ admits an $\eps$-net of size $\exp(Cn\log\frac{n}{\eps})$ wrt $\dexp$ defined in~\cref{eq:dexp_metric}: for any $\mb p \in \bb S^{n-1}$, there exists a $\mb q$ in  the net with $\supp\paren{\mb q} = \supp\paren{\mb p}$ and $\dexp\paren{\mb p, \mb q} \le \eps$. We say such $\eps$ nets are admissible for $\bb S^{n-1}$ wrt $\dexp$.
\end{lemma}
\begin{proof}
  Let $\eta=\eps \pi/n^2$. By~\cref{lemma:cover-length-2}, there exists a subset $\mc{Q}\subset\bb S^{n-1}$ of size at most
  \begin{equation}
    \paren{\frac{5n\log(1/\eta)}{\eta}}^{2n-1}
    = \paren{\frac{5n^3\log(n^2/\paren{\eps \pi})}{\eps \pi}}^{2n-1}
    \le \exp\paren{Cn\log\frac{n}{\eps}},
  \end{equation}
  such that for any $\mb p\in\bb S^{n-1}$, there exists a $\mb q\in \mc Q$ with $\supp(\mb p)=\supp(\mb q)$ and $\angle\paren{\mb p_\Omega, \mb q_\Omega}\le \eta$ for all $|\Omega|\le 2$. In particular, the $|\Omega|=1$ case says that $\sign(\mb p)=\sign(\mb q)$, which implies that
  \begin{align}
    \angle \paren{\mb p_{\Omega}, \mb q_{\Omega}} \le \pi/2\quad \forall\; \Omega \in \set{0, 1}^n.
  \end{align}
  Thus, applying the vector angle inequality~(\cref{lemma:angle_ineqn}), for any $\mb p\in\bb S^{n-1}$ and the corresponding $\mb q\in\mc{Q}$, we have
\begin{align}
  \angle\paren{\mb p_{\Omega}, \mb q_{\Omega}} \le \sum_{|\Omega'|=2,\Omega'\subset\Omega}\angle\paren{\mb p_{\Omega'}, \mb q_{\Omega'}} \le 2\binom{|\Omega|}{2} \eta \le |\Omega|^2 \eta \quad \forall \; \Omega\; \text{with}\; 3 \le |\Omega| \le n.
\end{align}
Summing up, we get
\begin{align}
  \angle \paren{\mb p_\Omega, \mb q_\Omega} \le  \abs{\Omega}^2 \eta \le n^2 \eta = \eps \pi \quad \forall\; \Omega.
\end{align}
Thus, $\dexp(\mb p,\mb q)\le \eps$ in view of~\cref{eq:dexp_angle_equiv}, completing the proof.
\end{proof}

Below we establish the desired ``Lipschitz" property in terms of the $\dexp$ distance.
\begin{lemma}  \label{lemma:uniform_sign_pert}
Fix a $\theta \in (0, 1)$. For any $\eps \in (0, 1)$,  let $N_\eps$ be an admissible $\eps$-net for $\bb S^{n-1}$ wrt $\dexp$. Let $\mb x_1,\dots,\mb x_m$ be iid copies of $\mb x \sim_{iid} \bgt$ in $\R^n$. When $m \ge C\eps^{-2} n$, the inequality
  \begin{align}
    \sup_{\substack{\mb p\,\in\, \bb S^{n-1}, \mb q\,\in\, N_\eps  \\ \supp\paren{\mb p} = \supp\paren{\mb q}, \; \dexp\paren{\mb p, \mb q} \,\le\, \eps  }} R\paren{\mb p, \mb q}
    \doteq \frac{1}{m} \sum_{i=1}^m \indicator{\sign\paren{\mb p^\top \mb x_i} \ne \sign\paren{\mb q^\top \mb x_i}}
    \le 2\eps
  \end{align}
  holds with probability at least $1 - \exp\paren{-c\eps^2 m}$. Here $C, c > 0$ are universal constants independent of $\eps$ and $\theta$.
\end{lemma}
\begin{proof}
We call any pair of $\mb p, \mb q \in \bb S^{n-1}$ with $\mb q \in N_\eps$, $\supp\paren{\mb p} = \supp\paren{\mb q}$, and $\dexp\paren{\mb p, \mb q} \le \eps$ an admissible pair. Over any admissible pair $\paren{\mb p, \mb q}$, $\expect{R} = \dexp\paren{\mb p, \mb q}$. We next bound the deviation $R-\expect{R}$ uniformly over all admissible $\paren{\mb p, \mb q}$ pairs. Observe that the process $R$ is the sample average of $m$ indicator functions. Define the hypothesis class
  \begin{align}
    \mc H = \set{\mb x \mapsto\indic{\sign\paren{\mb p^\top \mb x} \ne \sign\paren{\mb q^\top \mb x}}: \paren{\mb p, \mb q} \text{is an admissible pair}}.
  \end{align}
  and let $d_{\mathrm{vc}}(\mc{H})$ be the VC-dimension of $\mc{H}$. From concentration results for VC-classes (see, e.g., Eq (3) and Theorem 3.4 of~\cite{BoucheronEtAl2005Theory}), we have
  \begin{align}
    \label{eqn:vc-concentration}
    \prob{\sup_{\paren{\mb p, \mb q} \; \text{admissible}}
      \left\{ R(\mb p, \mb q) - \expect{R}(\mb p, \mb q)
      \right\} \ge
      C_0\sqrt{\frac{d_{\mathrm{vc}}(\mc{H})}{m}} + t} \le
    \exp(-mt^2)
  \end{align}
  for any $t > 0$. It remains to bound the VC-dimension $d_{\mathrm{vc}}(\mc{H})$. First, we have
  \begin{align}
    d_{\mathrm{vc}}\paren{\mc H}
    \le d_{\mathrm{vc}} \set{\mb x\mapsto\indic{\sign\paren{\mb p^\top \mb x} \ne \sign\paren{\mb q^\top \mb x}} : \mb p, \mb q \in \bb S^{n-1} } .
  \end{align}
  Observe that each set in the latter hypothesis class can be written as
  \begin{align*}
    & \set{\mb x\mapsto\indic{\sign\paren{\mb p^\top \mb x} \ne \sign\paren{\mb q^\top \mb x}} : \mb p, \mb q \in \bb S^{n-1} } \nonumber \\
   & =\;   \set{\mb x \mapsto \indicator{\mb p^\top \mb x >0, \mb q^\top \mb x \le 0}: \mb p, \mb q \in \bb S^{n-1}}
    \cup \set{\mb x \mapsto \indicator{\mb p^\top \mb x \ge 0, \mb q^\top \mb x < 0}: \mb p, \mb q \in \bb S^{n-1}} \\
    & \; \cup \set{\mb x \mapsto \indicator{\mb p^\top \mb x < 0, \mb q^\top \mb x \ge 0}: \mb p, \mb q \in \bb S^{n-1}} \cup \set{\mb x \mapsto \indicator{\mb p^\top \mb x \le 0, \mb q^\top \mb x > 0}: \mb p, \mb q \in \bb S^{n-1}}.
  \end{align*}
  the union of intersections of two halfspaces. Thus, letting
  \begin{equation}
    \mc{H}_0=\set{\mb x\mapsto \indicator{\mb x^\top\mb z \ge 0}:\mb z\in\R^n}
  \end{equation}
  be the class of halfspaces, we have
  \begin{equation}
    \mc{H} \subset (\mc{H}_0\sqcap\mc{H}_0) \sqcup
    (\mc{H}_0\sqcap \mc{H}_0)\sqcup
    (\mc{H}_0\sqcap \mc{H}_0)\sqcup
    (\mc{H}_0\sqcap \mc{H}_0).
  \end{equation}
  Note that $\mc{H}_0$ has VC-dimension $n+1$. Applying bounds on the VC-dimension of unions and intersections (Theorem
  1.1,~\cite{VanDerVaartWellner2009note}), we get that
  \begin{align}
    d_{\mathrm{vc}}(\mc{H})
    \le Cd_{\mathrm{vc}}(\mc{H}_0\sqcap\mc{H}_0)
    \le Cd_{\mathrm{vc}}(\mc{H}_0)
    \le C'n.
  \end{align}
  Plugging this bound into~\cref{eqn:vc-concentration}, we can set $t = \eps/2$ and make $m$ large enough so that $C_0 \sqrt{C'} \sqrt{n/m} \le \eps/2$, completing the proof.
\end{proof}

\subsection{Pointwise convergence of subdifferential}
\begin{proposition} [Pointwise convergence] \label{prop:egrad_concentrate}
 For any fixed $\mb q \in \bb S^{n-1}$,
\begin{align}
\prob{\hd\paren{\partial f\paren{\mb q}, \expect{\partial f}\paren{\mb q} }> C_a \sqrt{n/m} + C_b t/\sqrt{m}}  \le 2\exp\paren{-t^2} \quad \forall \; t > 0.
\end{align}
Here $C_a, C_b \ge 0$ are universal constants.
\end{proposition}
\begin{proof}
 Recall that
\begin{align}
  \hd\paren{\partial f\paren{\mb q}, \expect{\partial f\paren{\mb q}}}
  = \sup_{\mb u \in \bb S^{n-1}} \abs{h_{\partial f(\mb q)}\paren{\mb u} - h_{\bb E \partial f(\mb q)}\paren{\mb u}}
  = \sup_{\mb u \in \bb S^{n-1}} \abs{h_{\partial f(\mb q)}\paren{\mb u} - \bb E h_{\partial f(\mb q)}\paren{\mb u}}.
\end{align}
Write $X_{\mb u} \doteq h_{\partial f(\mb q)}\paren{\mb u} - \bb E h_{\partial f(\mb q)}\paren{\mb u}$ and consider the zero-mean random process $\{X_{\mb u}\}$ defined on $\bb S^{n-1}$. For any $\mb u, \mb v \in \bb S^{n-1}$, we have
\begin{align}
  \norm{X_{\mb u} - X_{\mb v}}{\psi_2}
  & = \norm{h_{\partial f(\mb q)}\paren{\mb u} - \bb E h_{\partial f(\mb q)}\paren{\mb u} - h_{\partial f(\mb q)}\paren{\mb v} + \bb E h_{\partial f(\mb q)}\paren{\mb v}}{\psi_2} \\
  & = C\norm{\frac{1}{m}\sum_{i \in [m]}  \paren{h_{Q_i} \paren{\mb u} - \bb E h_{Q_i} \paren{\mb u} - h_{Q_i} \paren{\mb v} + \bb E h_{Q_i} \paren{\mb v}}  }{\psi_2}\\
  & \le C\frac{1}{m}  \paren{\sum_{i \in [m]}  \norm{h_{Q_i} \paren{\mb u} - \bb E h_{Q_i} \paren{\mb u} - h_{Q_i} \paren{\mb v} + \bb E h_{Q_i} \paren{\mb v}}{\psi_2}^2   }^{1/2} \\
  & \le C\frac{1}{m} \paren{\sum_{i \in [m]}  \norm{h_{Q_i} \paren{\mb u} - h_{Q_i} \paren{\mb v} } {\psi_2}^2   }^{1/2}  \quad (\text{centering}),
\end{align}
where we write $Q_i \doteq \sign\paren{\mb q^\top \mb x_i} \mb x_i$ for all $i \in [m]$. Next we estimate $\norm{h_{Q_i} \paren{\mb u} - h_{Q_i} \paren{\mb v} } {\psi_2}$. By definition,
\begin{align}
  h_{Q_i}\paren{\mb u} - h_{Q_i}\paren{\mb v}
  = \sup_{\mb  z \in Q_i}   \innerprod{\mb z}{\mb u} - \sup_{\mb  z' \in Q_i}   \innerprod{\mb z'}{\mb v}.
\end{align}
If $ h_{Q_i}\paren{\mb u} - h_{Q_i}\paren{\mb v} \ge 0$ and let $\mb z_* \doteq \argmax_{\mb  z \in Q_i}   \innerprod{\mb z}{\mb u}$, we have
\begin{align}
    h_{Q_i}\paren{\mb u} - h_{Q_i}\paren{\mb v}
    \le \innerprod{\mb z_*}{\mb u} - \innerprod{\mb z_*}{\mb v} = \innerprod{\mb z_*}{\mb u - \mb v},
\end{align}
and
\begin{align}
  \norm{  h_{Q_i}\paren{\mb u} - h_{Q_i}\paren{\mb v}}{\psi_2}
 \le \norm{\innerprod{\mb z_*}{\mb u - \mb v}}{\psi_2}
 \le  \norm{\mb x_i^\top \paren{\mb u - \mb v}}{\psi_2}
  \le C \norm{\mb u - \mb v}{},
\end{align}
where we have used~\cref{lemma:subgauss_para} to obtain the last upper bound. If $h_{Q_i}\paren{\mb u} - h_{Q_i}\paren{\mb v} \le 0$, $h_{Q_i}\paren{\mb v} - h_{Q_i}\paren{\mb u} \ge 0$ and we can use similar argument to conclude that
\begin{align}
  \norm{ h_{Q_i}\paren{\mb u} - h_{Q_i}\paren{\mb v}}{\psi_2} \le C \norm{\mb u -\mb v}{}.
\end{align}
So
\begin{align}
    \norm{X_{\mb u} - X_{\mb v}}{\psi_2} \le \frac{C}{\sqrt{m}} \norm{\mb u - \mb v}{}.
\end{align}
Thus,  $\{X_{\mb u}\}$ is a centered random process with sub-gaussian increments with a parameter $C/\sqrt{m}$. We can apply~\cref{prop:talagrand_comparison_set} to conclude that
\begin{align}
  \prob{\sup_{\mb u \in \bb S^{n-1}} \abs{h_{\partial f(\mb q)}\paren{\mb u} - \bb E h_{\partial f(\mb q)}\paren{\mb u}} > C' \sqrt{n/m} + C'' t/\sqrt{m} } \le 2 \exp\paren{-t^2} \quad \forall\; t> 0,
\end{align}
which implies the claimed result.
\end{proof}

\subsection{Proof of~\cref{prop:egrad_concentrate_uniform} (Uniform convergence)}
Fix an $\eps \in (0, 1/2)$ to be decided later. Let $N_\eps$ be an admissible $\eps$ net for $\bb S^{n-1}$ wrt $\dexp$, with $\abs{N_\eps} \le \exp(Cn\log(n/\eps))$ (\cref{lemma:cover-cheaper}). By~\cref{prop:egrad_concentrate} and the union bound,
\begin{align}
    \label{equation:uc-term-i}
  \prob{\exists\, \mb q \in N_\eps, \hd\paren{\partial f\paren{\mb q}, \expect{\partial f}\paren{\mb q}} > t/3} \le \exp\paren{-c mt^2 + Cn\log\frac{n}{\eps}}
\end{align}
provided that $m \ge C t^{-2} n$.

For any $\mb p \in \bb S^{n-1}$, let $\mb q \in N_\eps$ satisfy $\supp\paren{\mb q} = \supp\paren{\mb p}$ and $\dexp\paren{\mb p, \mb q} \le \eps$. Then we have
\begin{align*}
  \hd\paren{\partial f\paren{\mb p}, \expect{\partial f}\paren{\mb p}} \le
  \underbrace{\hd\paren{\partial f\paren{\mb q}, \expect{\partial f}\paren{\mb q}}}_{\rm I} +
  \underbrace{\hd\paren{\expect{\partial f}\paren{\mb p}, \expect{\partial f}\paren{\mb q}}}_{\rm II} +
  \underbrace{\hd\paren{\partial f\paren{\mb p}, \partial f\paren{\mb q}}}_{\rm III}
\end{align*}
by the triangular inequality for the Hausdorff metric.

By the preceding union bound, term I is bounded by $t/3$ as long as the bad event does not happen.

For term II, we have
\begin{align}
  & \hd\paren{\expect{\partial f}\paren{\mb p}, \expect{\partial f}\paren{\mb q}} \nonumber \\
  =\; & \sup_{\mb u \in \bb S^{n-1}} \abs{h_{\expect{\partial f}\paren{\mb p}}\paren{\mb u} - h_{\expect{\partial f}\paren{\mb q}}\paren{\mb u}} \\
  =\; & \sup_{\mb u \in \bb S^{n-1}} \abs{\bb E\brac{ h_{\partial f\paren{\mb p}}\paren{\mb u} - h_{\partial f\paren{\mb q}}\paren{\mb u} }} \\
  =\; & \objscale \sup_{\mb u\in\bb S^{n-1}} \abs{\expect{\sup\;\sign\paren{\mb p^\top\mb x}\mb x^\top\mb u - \sup\;\sign\paren{\mb q^\top\mb x}\mb x^\top\mb u}} \\
  & \qquad (\text{support function linearizing Minkowski sum, and linearity of selection expectation}) \nonumber \\
  \le\; & \objscale \sup_{\mb u\in\bb S^{n-1}} \abs{\expect{|\mb x^\top\mb u|\indicator{\sign\paren{\mb p^\top\mb x} \neq \sign\paren{\mb q^\top\mb x}}}}\\
  \le\; & \objscale 3\eps\sqrt{\log\frac{1}{\eps}}. \label{equation:uc-term-ii}
\end{align}
where the last line follows from~\cref{lemma:expected-gradient-diff}. As long as $\eps\le ct/\sqrt{\log(1/t)}$ for a sufficiently small $c$, the above term is upper bounded by $t/3$.
For term III, we have
\begin{align}
  & \hd\paren{\partial f \paren{\mb p}, \partial f\paren{\mb q}} \nonumber \\
  =\; & \sup_{\mb u \in \bb S^{n-1}} \abs{h_{\partial f \paren{\mb p}}\paren{\mb u} - h_{\partial f\paren{\mb q}}\paren{\mb u}} \\
  =\; & \objscale \frac{1}{m}\sup_{\mb u \in \bb S^{n-1}} \abs{\sum_{i \in [m]: \sign\paren{\mb p^\top \mb x_i} \ne \sign\paren{\mb q^\top \mb x_i}} \sup\; \sign\paren{\mb p^\top \mb x_i} \mb x_i^\top \mb u  - \sup\; \sign\paren{\mb q^\top \mb x_i} \mb x_i^\top \mb u} \\
  =\; & \objscale \frac{2}{m}\sup_{\mb u \in \bb S^{n-1}} \abs{\sum_{i \in [m]: \sign\paren{\mb p^\top \mb x_i} \ne \sign\paren{\mb q^\top \mb x_i}} s_i \mb x_i^\top \mb u} \quad (\text{$s_i \in \set{+1, -1, 0}$}) \\
  =\; & \objscale \frac{2}{m} \norm{\sum_{i \in [m]: \sign\paren{\mb p^\top \mb x_i} \ne \sign\paren{\mb q^\top \mb x_i}} s_i \mb x_i}{}.
\end{align}
By~\cref{lemma:uniform_sign_pert}, with probability at least $1-\exp(-c\eps^2 m)$, the number of different signs is upper bounded by $2m\eps$ for all $\mb p, \mb q$ such that $\dexp(\mb p, \mb q)\le \eps$. On this good event, the above quantity can be upper bounded as follows. Define a set $T \doteq \set{\mb s \in \R^m: s_i \in \set{+1, -1, 0}, \norm{\mb s}{0} \le 2m\eps}$ and consider the quantity $\sup_{\mb s \in T} \norm{\mb X \mb s}{}$, where $\mb X = [\mb x_1, \dots, \mb x_m]$. Then,
\begin{align}
  \norm{\sum_{i \in [m]: \sign\paren{\mb p^\top \mb x_i} \ne \sign\paren{\mb q^\top \mb x_i}} s_i \mb x_i}{}
  \le \sup_{\mb s \in T} \norm{\mb X \mb s}{}
\end{align}
uniformly (i.e., indepdent of $\mb p, \mb q$ and $\mb u$). We have
\begin{align}
  w\paren{T}
  & = \bb E \sup_{\mb s \in T} \mb s^\top \mb g
  = \bb E \sup_{K \subset [m], \abs{K} \le 2m\eps} \sum_{i \in K} \abs{g_i} \\
  & \le 2m\eps \bb E \norm{\mb g}{\infty} \le 4m\eps\sqrt{\log m}, \quad (\text{here $\mb g \sim \mc N\paren{\mb 0, \mb I_m}$}) \\
  \radi\paren{T}
  & = \sqrt{2m\eps}.
\end{align}
Noting that $1/\sqrt{\theta} \cdot \mb X$ has independent, isotropic, and sub-gaussian rows with a parameter $C/\sqrt{\theta}$, we apply~\cref{prop:matrix_dev_subgauss} and obtain that
\begin{align}
  \sup_{\mb s \in T} \norm{\mb X \mb s}{} \le \sqrt{\theta n} \sqrt{2m\eps} + \frac{C}{\sqrt{\theta}} \paren{4m\eps\sqrt{\log m} + t_0 \sqrt{2m\eps}}
\end{align}
with probability at least $1 - 2\exp\paren{-t_0^2}$. So we have over all admissible $\paren{\mb p, \mb q}$ pairs,
\begin{align}
  \hd\paren{\partial f\paren{\mb p}, \partial f\paren{\mb q}}
  & \le \objscale \frac{2}{m} \brac{\sqrt{\theta n} \sqrt{2m \eps} + \frac{C}{\sqrt{\theta}} \paren{4m\eps\sqrt{\log m} + t_0 \sqrt{2m\eps}} } \\
  & = \objscale \paren{\sqrt{\frac{8\theta n\eps}{m}} + C\eps\sqrt{\frac{\log m}{\theta}} + C't_0\sqrt{\frac{8\eps}{m}}}.
\end{align}
Setting $t_0 = c t\sqrt{m}$ and $\eps=ct\sqrt{\theta/\log m}$, we have that
\begin{align}
  \label{equation:uc-term-iii}
  \hd\paren{\partial f\paren{\mb p}, \partial f\paren{\mb q}} \le \frac{t}{3},
\end{align}
provided that $m \ge C\eps t^{-2} n=Ct^{-1}n\sqrt{\theta/\log m}$, which is subsumed by the earlier requirement $m\ge Ct^{-2}n$.

Putting together the three bounds~\cref{equation:uc-term-i},~\cref{equation:uc-term-ii},~\cref{equation:uc-term-iii}, we can choose
\begin{equation}
  \eps = ct\sqrt{\frac{\theta}{\log(m/t)}} \le ct\cdot \min\set{\sqrt{\frac{\theta}{\log m}}, \frac{1}{\sqrt{\log(1/t)}}}
\end{equation}
and get that $\hd\paren{\partial f(\mb p), \expect{\partial f}(\mb p)} \le t$ with probability at least
\begin{align}
  & \quad 1 - 2\exp\paren{-cmt^2} - \exp(-cm\eps^2) - \exp\paren{-cmt^2+Cn\log\frac{n}{\eps}} \\
  & \ge 1 - 2\exp(-cmt^2) - \exp\paren{-\frac{cm\theta t^2}{\log(m/t)}} - \exp\paren{-cmt^2 + Cn\log\frac{n\log(m/t)}{\theta t}} \\
  & \ge 1 - \exp\paren{-\frac{cm\theta t^2}{\log(m/t)}}
\end{align}
provided that $m\ge Cnt^{-2}\log\frac{n\log(m/t)}{\theta t}$. A sufficient condition is that $m\ge Cnt^{-2}\log (n/t)$ for sufficiently large $C$. When this is satisfied, the probability is further lower bounded by $1-\exp(-cm\theta t^2/\log m)$.

\subsection{Proof of~\cref{theorem:emp-grad-inward}}
Define
\begin{equation}
  t = \frac{1}{32n^{3/2}}\theta(1-\theta)\zeta_0 \le \min\set{\frac{1}{8n^{3/2}}\theta(1-\theta)\frac{\zeta_0}{1+\zeta_0}, \frac{2-\sqrt{2}}{16n^{3/2}}\theta(1-\theta)\zeta_0}.
\end{equation}
By~\cref{prop:egrad_concentrate_uniform}, with probability at least $1 - \exp\paren{-c m \theta^3 \zeta_0^2 n^{-3} \log^{-1}m}$ we have
\begin{equation}
  \hd\paren{\expect{\partial f}\paren{\mb q}, \partial f\paren{\mb q}} \le t,
\end{equation}
provided that $m \ge Cn^4 \theta^{-2} \zeta_0^{-2} \log \paren{n/\zeta_0}$. We now show the properties~\cref{equation:emp-grad-inward} and~\cref{equation:emp-grad-star} on this good event, focusing on $\mc{S}_{\zeta_0}^{(n+)}$; the same results obtain on all other $2n-1$ subsets by analogous arguments.

For~\cref{equation:emp-grad-inward}, we have
\begin{align*}
  \innerprod{\partial_R f\paren{\mb q}}{\mb e_j/q_j - \mb e_n/q_n}
  = \innerprod{\partial f\paren{\mb q}}{\paren{\mb I - \mb q \mb q^\top} \paren{\mb e_j/q_j - \mb e_n/q_n}}
  = \innerprod{\partial f\paren{\mb q}}{\mb e_j/q_j - \mb e_n/q_n}.
\end{align*}
Now
\begin{align}
  & \sup \innerprod{\partial f\paren{\mb q}}{\mb e_n/q_n - \mb e_j/q_j} \nonumber \\
  =\; & h_{\partial f\paren{\mb q}}\paren{\mb e_n/q_n - \mb e_j/q_j}  \\
  =\; & \bb E h_{\partial f\paren{\mb q}}\paren{\mb e_n/q_n - \mb e_j/q_j} - \bb E h_{\partial f\paren{\mb q}}\paren{\mb e_n/q_n - \mb e_j/q_j}  + h_{\partial f\paren{\mb q}}\paren{\mb e_n/q_n - \mb e_j/q_j} \\
  \le\; & \bb E h_{\partial f\paren{\mb q}}\paren{\mb e_n/q_n - \mb e_j/q_j} + \norm{\mb e_n/q_n - \mb e_j/q_j}{}\sup_{\mb u \in \bb S^{n-1}} \abs{\bb E h_{\partial f\paren{\mb q}}\paren{\mb u}  - h_{\partial f\paren{\mb q}}\paren{\mb u}} \\
  =\; & \sup \innerprod{\expect{\partial f}\paren{\mb q}}{\mb e_n/q_n - \mb e_j/q_j} + \norm{\mb e_n/q_n - \mb e_j/q_j}{} \hd\paren{\expect{\partial f}\paren{\mb q}, \partial f\paren{\mb q}}.
\end{align}
By~\cref{theorem:pop-geometry}(a),
\begin{align}
  \sup \innerprod{\expect{\partial f}\paren{\mb q}}{\mb e_n - q_n \mb q} \le -\frac{1}{2n}\theta\paren{1-\theta} \frac{\zeta_0}{1+\zeta_0}.
\end{align}
Moreover, $\norm{\mb e_n/q_n - \mb e_j/q_j}{} = \sqrt{1/q_n^2 + 1/q_j^2} \le \sqrt{1/q_n^2 + 3/q_n^2} \le 2\sqrt{n}$. Meanwhile, we have
\begin{align}
  \hd\paren{\expect{\partial f}\paren{\mb q}, \partial f\paren{\mb q}}
  \le t \le \frac{1}{8n^{3/2}}\theta\paren{1-\theta} \frac{\zeta_0}{1+\zeta_0}.
\end{align}
We conclude that
\begin{align}
  \inf\innerprod{\partial f\paren{\mb q}}{\mb e_j/q_j - \mb e_n/q_n}
  & = - \sup \innerprod{\partial f\paren{\mb q}}{\mb e_n/q_n - \mb e_j/q_j} \\
  & \ge \frac{1}{2n}\theta\paren{1-\theta} \frac{\zeta_0}{1+\zeta_0} -2\sqrt{n} \cdot \frac{1}{8n^{3/2}}\theta\paren{1-\theta} \frac{\zeta_0}{1+\zeta_0}\\
    & \ge \frac{1}{4n}\theta\paren{1-\theta} \frac{\zeta_0}{1+\zeta_0},
\end{align}
as claimed.

We now show~\cref{equation:emp-grad-star} using similar argument based on ~\cref{theorem:pop-geometry}(b). Now,
\begin{align}
  & \sup \innerprod{\partial f\paren{\mb q}}{\mb e_n - q_n \mb q} \nonumber \\
  =\; & h_{\partial f\paren{\mb q}}\paren{\mb e_n - q_n \mb q}  \\
  =\; & \bb E h_{\partial f\paren{\mb q}}\paren{\mb e_n - q_n \mb q} - \bb E h_{\partial f\paren{\mb q}}\paren{\mb e_n - q_n \mb q}  + h_{\partial f\paren{\mb q}}\paren{\mb e_n - q_n \mb q} \\
  \le\; & \bb E h_{\partial f\paren{\mb q}}\paren{\mb e_n - q_n \mb q} + \norm{\mb e_n - q_n \mb q}{}\sup_{\mb u \in \bb S^{n-1}} \abs{\bb E h_{\partial f\paren{\mb q}}\paren{\mb u}  - h_{\partial f\paren{\mb q}}\paren{\mb u}} \\
  =\; & \sup \innerprod{\expect{\partial f}\paren{\mb q}}{\mb e_n - q_n \mb q} + \norm{\mb q_{-n}}{} \hd\paren{\expect{\partial f}\paren{\mb q}, \partial f\paren{\mb q}}.
\end{align}
As we are on the good event
\begin{align}
  \hd\paren{\expect{\partial f}\paren{\mb q}, \partial f\paren{\mb q}} \le t \le \frac{2-\sqrt{2}}{16n^{3/2}} \cdot \theta \paren{1-\theta} \zeta_0.
\end{align}
Thus,
\begin{align}
   \inf\innerprod{\partial f\paren{\mb q}}{q_n \mb q  - \mb e_n}
   & = - \sup \innerprod{\partial f\paren{\mb q}}{\mb e_n - q_n \mb q} \\
   & \ge \frac{1}{8} \theta \paren{1-\theta}\zeta_0 n^{-3/2} \norm{\mb q_{-n}}{} - \norm{\mb q_{-n}}{} \frac{2-\sqrt{2}}{16} \cdot \theta \paren{1-\theta} \zeta_0 n^{-3/2} \\
   & \ge \frac{\sqrt{2}}{16}  \theta \paren{1-\theta} \zeta_0 n^{-3/2}\norm{\mb q_{-n}}{}.
\end{align}
Noting that $\norm{\mb q_{-n}}{} \ge \frac{1}{\sqrt{2}}\norm{\mb q - \mb e_n}{}$ for all $\mb q$ with $q_n \ge 0$ completes the proof.

\subsection{Proof of~\cref{prop:grad_norm_concentrate}}
For any $\mb q \in \bb S^{n-1}$,
\begin{align}
  \sup \norm{\partial f\paren{\mb q}}{}
  = \hd\paren{\set{0}, \partial f\paren{\mb q}}
  \le \hd\paren{\set{0}, \expect{\partial f}\paren{\mb q}} + \hd\paren{\partial f\paren{\mb q}, \expect{\partial f}\paren{\mb q}}
\end{align}
by the metric property of the Hausdorff metric. On one hand, we have
\begin{align}
  \sup \norm{\expect{\partial f}\paren{\mb q}}{} = \sup \norm{\bb E_{\Omega} \brac{ \frac{\mb q_\Omega}{\norm{\mb q_\Omega}{}} \indicator{\mb q_\Omega \ne \mb 0}  + \set{\mb v_\Omega: \norm{\mb v_\Omega}{} \le 1} \indicator{\mb q_\Omega = 0}   }  }{} \le 1.
\end{align}
On the other hand, by~\cref{prop:egrad_concentrate_uniform},
\begin{align}
  \hd\paren{\partial f\paren{\mb q}, \expect{\partial f}\paren{\mb q}} \le 1  \quad \forall\; \mb q \in  \bb S^{n-1}
\end{align}
with probability at least $1 - \exp\paren{-cm\theta \log^{-1} m}$, provided that $m \ge C n\log n$. Combining the two results completes the proof.

\subsection{Bi-Lipschitzness of the objective}
\begin{proposition} \label{prop:func_lipschitz}
On the good event in~\cref{prop:grad_norm_concentrate}, for all $\mb q \in \mc S_{\zeta_0}^{(n+)}$, we have
\begin{align}
f\paren{\mb q} - f\paren{\mb e_n}
  & \le  2 \sqrt{n} \norm{\mb q - \mb e_n}{}.
\end{align}
\end{proposition}
\begin{proof}
We use Lebourg's mean value theorem for locally Lipschitz functions\footnote{It is possible to directly apply the manifold version of Lebourg's mean value theorem, i.e., Theorem 3.3 of~\cite{HosseiniPouryayevali2011Generalized}. We avoid this technicality by working with the Euclidean version in $\mb w$ space. }, i.e., \cref{thm:nsms_mvt}. It is convenient to work in the $\mb w$ space here. By subdifferential chain rules (\cref{thm:subdiff_chain_rule}), $g\paren{w}$ is locally Lipschitz over $\set{\mb w: \norm{\mb w}{} < \sqrt{\frac{n-1}{n}}}$. Thus, we have
\begin{align}
  f\paren{\mb q} - f\paren{\mb e_n}
   = g\paren{\mb w} - g\paren{\mb 0}
   = \innerprod{\mb v}{\mb w}
\end{align}
for a certain $t_0 \in (0, 1)$ and a certain $\mb v \in \partial g\paren{t_0 \mb w}$. Now for any $\mb q$ and the corresponding $\mb w$,
\begin{align}
  \innerprod{\partial g\paren{\mb w}}{\mb w}
  = \frac{1}{q_n} \innerprod{\partial_R f\paren{\mb q}}{q_n \mb q - \mb e_n}.
\end{align}
It follows
\begin{multline}
  \innerprod{\mb v}{\mb w}
  \le  \frac{1}{t_0} \sup \innerprod{\partial g\paren{t_0 \mb w}}{t_0 \mb w}
  = \frac{1}{t_0} \frac{1}{q_n\paren{t_0 \mb w}}\sup \innerprod{\partial_R f\paren{\mb q\paren{t_0 \mb w}}}{q_n\paren{t_0 \mb w}\mb q\paren{t_0 \mb w} - \mb e_n} \\
  \le \sup \norm{\partial_R f\paren{\mb q\paren{t_0 \mb w}}}{} \cdot \frac{\norm{q_n\paren{t_0 \mb w}\mb q \paren{t_0 \mb w} - \mb e_n}{}}{t_0 q_n\paren{t_0 \mb w}}
  \le 2 \frac{\norm{q_n\paren{t_0 \mb w}\mb q \paren{t_0 \mb w} - \mb e_n}{}}{t_0 q_n\paren{t_0 \mb w}},
\end{multline}
where at the last inequality we have used~\cref{prop:grad_norm_concentrate}.
Continuing the calculation, we further have
\begin{align}
  \frac{\norm{q_n\paren{t_0 \mb w}\mb q \paren{t_0 \mb w} - \mb e_n}{}}{t_0 q_n\paren{t_0 \mb w}}
  = \frac{t_0 \norm{\mb w}{}}{t_0 q_n\paren{t_0 \mb w}}
  \le \sqrt{n} \norm{\mb w}{} \le \sqrt{n}\norm{\mb q - \mb e_n}{},
\end{align}
completing the proof.
\end{proof}

\begin{proposition} \label{prop:func_sharpness}
Assume $\theta \in [1/n, 1/2]$. When $m \ge C\theta^{-2} n\log n$, with probability at least $ 1 - \exp\paren{-c m \theta^3 \log^{-1}m}$, the following holds: for all $\mb q \in \mc S_{\zeta_0}^{(n+)}$ satisfying $f\paren{\mb q} - f\paren{\mb e_n} \le \frac{2}{25}\theta$,
\begin{align}
  f\paren{\mb q} - f\paren{\mb e_n} \ge \frac{\sqrt{2}}{16} \theta \paren{1-\theta} \norm{\mb q_{-n}}{} \ge \frac{1}{16} \theta \paren{1-\theta} \norm{\mb q - \mb e_n}{}.
\end{align}
Here $C, c > 0$ are universal constants. 
\end{proposition}
\begin{proof}
We first establish uniform convergence of $f\paren{\mb p}$ to $\expect{f}\paren{\mb p}$. Consider the zero-centered random process $X_{\mb p} \doteq f\paren{\mb p} - \expect{f}\paren{\mb p}$ on $\bb S^{n-1}$. Similar to proof of~\cref{prop:egrad_concentrate}, we can show that for all $\mb p, \mb q \in \bb S^{n-1}$
\begin{align}
  \norm{X_{\mb p} - X_{\mb q}}{\psi_2} \le  \frac{C}{\sqrt{m}} \norm{\mb p - \mb q}{}.
\end{align}
Applying~\cref{prop:talagrand_comparison_set} gives that
\begin{align}  \label{eq:proof_fn_sharpness_eq_1}
  \norm{f\paren{\mb p} - \expect{f}\paren{\mb q}}{} \le \frac{1}{100}\theta \quad \forall\; \mb q \in \bb S^{n-1}
\end{align}
with probability at least $1 - \exp\paren{-cm\theta^2}$, provided that $m \ge C\theta^{-2} n$.

Now we consider $\expect{f}\paren{\mb q} - \expect{f}\paren{\mb e_n}$. For convenience, we first work in the $\mb w$ space and note that $\expect{f}\paren{\mb q} - \expect{f}\paren{\mb e_n} = \expect{g}\paren{\mb w\paren{\mb q}} - \expect{g}\paren{\mb 0}$. By~\cref{prop:inward_radial_grad}, $\expect{g}$ is monotonically increasing in every radial direction $\mb v$ until $\norm{\mb w}{}^2 + \norm{\mb w}{\infty}^2 \le 1$,  which implies that
\begin{align}
  \inf_{\norm{\mb w}{} \ge 1/2} \expect{g}\paren{\mb w\paren{\mb q}} - \expect{g}\paren{\mb 0}
  = \inf_{\norm{\mb w}{} = 1/2}  \expect{g}\paren{\mb w\paren{\mb q}} - \expect{g}\paren{\mb 0}.
\end{align}
For $\mb w$ with $\norm{\mb w}{} = 1/2$,
\begin{align}
  \expect{g}\paren{\mb w} - \expect{g} \paren{\mb 0}
  & = \paren{1-\theta} \bb E_\Omega \norm{\mb w_\Omega}{} + \theta \bb E_\Omega \sqrt{1-\norm{\mb w_{\Omega^c}}{}^2} - \theta    \quad (\text{\cref{lem:asymp_obj}}) \\
  & \ge \paren{1-\theta} \theta \norm{\mb w}{} + \theta \bb E_\Omega \sqrt{1 - \norm{\mb w}{}^2} - \theta \\
  & \ge \frac{1}{4} \theta + \frac{\sqrt{3}}{2} \theta  - \theta  \quad (\text{using $\theta \le 1/2$ and $\norm{\mb w}{} = 1/2$}) \\
  & \ge \frac{1}{10} \theta. 
\end{align}
So, back to the $\mb q$ space,
\begin{align} \label{eq:proof_fn_sharpness_eq_2}
  \inf_{\mb q \in \mc S^{(n+)}_{\zeta_0}: \; \norm{\mb q_{-n}}{} \ge 1/2}  \expect{f}\paren{\mb q} - \expect{f}\paren{\mb 0} \ge \frac{1}{10} \theta.
\end{align}

Combining the results in~\cref{eq:proof_fn_sharpness_eq_1} and~\cref{eq:proof_fn_sharpness_eq_2}, we conclude that with high probability
\begin{align}
    \inf_{\mb q \in \mc S^{(n+)}_{\zeta_0}: \; \norm{\mb q_{-n}}{} \ge 1/2}  f\paren{\mb q} - f\paren{\mb 0} \ge \frac{2}{25} \theta.
\end{align}
So when $f\paren{\mb q} - f\paren{\mb 0} \le 2/25 \cdot \theta$, $\norm{\mb q_{-n}}{} \le 1/2$, which is equivalent to $\norm{\mb w}{} \le 1/2$ in the $\mb w$ space. Under this constraint, by~\cref{prop:inward_radial_grad},
\begin{align}
  D_{-\mb w/\norm{\mb w}{}}^c \expect{g}\paren{\mb w}
  & \le -\theta \paren{1-\theta} \paren{1/\sqrt{1+\norm{\mb w}{\infty}^2/\norm{\mb w}{}^2} - \norm{\mb w}{}} \\
  & \le -\theta \paren{1-\theta} \paren{\frac{1}{\sqrt{2}} - \frac{1}{2}} \le -\frac{1}{5}\theta \paren{1-\theta}.
\end{align}
So, emulating the proof of~\cref{equation:pop-grad-star} in~\cref{theorem:pop-geometry}, we have that for $\mb q \in \mc S_{\zeta_0}^{(n+)}$ with $\norm{\mb q_{-n}}{} \le 1/2$,
\begin{align}
  \innerprod{\expect{\partial_R f}\paren{\mb q}}{q_n \mb q - \mb e_n}
  = q_n \innerprod{\expect{\partial g}\paren{\mb w}}{\mb w}
  \ge q_n \norm{\mb w}{} \cdot \frac{1}{5} \theta \paren{1-\theta} \ge \frac{\sqrt{3}}{10} \theta \paren{1-\theta} \norm{\mb w}{},
\end{align}
where at the last inequality we use $q_n = \sqrt{1-\norm{\mb w}{}^2} \ge \sqrt{3}/2$ when $\norm{\mb w}{} \le 1/2$. Moreover, we emulate the proof of~\cref{equation:emp-grad-star} in~\cref{theorem:emp-grad-inward} to obtain that
\begin{align}
  \inf \innerprod{\partial_R f\paren{\mb q}}{q_n \mb q - \mb e_n} \ge \frac{\sqrt{2}}{16} \theta \paren{1-\theta} \norm{\mb q_{-n}}{} \ge \frac{1}{16} \theta \paren{1-\theta} \norm{\mb q - \mb e_n}{}
\end{align}
with probability at least $1 - \exp\paren{-c m\theta^3 \log^{-1}m}$, provided that $m \ge C \theta^{-2} n \log n$.

The last step of our proof is invoking the mean value theorem, similar to the proof of~\cref{prop:func_lipschitz}. For any $\mb q$, we have
\begin{align}
  f\paren{\mb q} - f\paren{\mb e_n} = g\paren{\mb w} - g\paren{\mb 0} = \innerprod{\mb v}{\mb w}
\end{align}
for a certain $t \in (0, 1)$ and a certain $\mb v \in \partial g\paren{t\mb w}$. We have
\begin{align}
  \innerprod{\mb v}{\mb w}
  \ge \frac{1}{t_0} \inf \innerprod{\partial g\paren{t_0 \mb w}}{t_0 \mb w}
  & = \frac{1}{t_0} \frac{1}{q_n\paren{t_0 \mb w}}\inf \innerprod{\partial_R f\paren{\mb q\paren{t_0 \mb w}}}{q_n\paren{t_0 \mb w}\mb q\paren{t_0 \mb w} - \mb e_n} \\
  & \ge \frac{1}{t_0} \frac{1}{q_n\paren{t_0 \mb w}} \frac{\sqrt{2}}{16} \theta \paren{1-\theta} \norm{t_0 \mb w}{} \\
  & \ge \frac{\sqrt{2}}{16} \theta \paren{1-\theta} \norm{\mb w}{} \\
  & \ge \frac{1}{16} \theta \paren{1-\theta}\norm{\mb q - \mb e_n}{},
\end{align}
completing the proof.
\end{proof}

%% file: Sections/proof-optimization.tex
\section{Proofs for~\cref{section:opt-one-basis}}
\subsection{Staying in the region $\mc{S}_{\zeta_0}^{(n+)}$}
\begin{lemma}[Progress in $\mc{S}_{\zeta_0}^{(n+)}\setminus \mc S_1^{(n+)}$]
  \label{prop:progress_r1}
Set $\eta = t_0/\paren{100\sqrt{n}} $ for any $t_0 \in (0, 1)$. For any $\zeta_0\in(0,1)$, on the good events stated in~\cref{prop:grad_norm_concentrate} and~\cref{theorem:emp-grad-inward}, we have for all $\mb q \in \mc S_{\zeta_0}^{(n+)}\setminus \mc{S}_1^{(n+)}$ that $\mb q_+$ (i.e., the next iterate) obeys
\begin{align}
  \frac{q_{+, n}^2}{\norm{\mb q_{+, -n}}{\infty}^2}
  \ge \frac{q_n^2}{\norm{\mb q_{-n}}{\infty}^2} \paren{1 + t\frac{\theta \paren{1-\theta} \zeta_0}{400n^{3/2} \paren{1+\zeta_0}}  }^2.
\end{align}
In particular, we have $\mb q_+\in\mc{S}_{\zeta_0}^{(n+)}$.
\end{lemma}
\begin{proof}
We divide the index set $[n-1]$ into three sets
\begin{align}
  \mc I_0 & \doteq \set{j \in [n-1]: q_j = 0}, \\
  \mc I_1 & \doteq \set{j \in [n-1]: q_n^2 /q_j^2 > 1 + 2 = 3, q_j \ne 0} \\
  \mc I_2 & \doteq \set{j \in [n-1]: q_n^2 /q_j^2 \le 1 + 2 = 3}.
\end{align}
 We perform different arguments on different sets. We let $\mb g\paren{\mb q} \in \partial_R f\paren{\mb q}$ be the subgradient taken at $\mb q$ and note by~\cref{prop:grad_norm_concentrate} that $\norm{\mb g}{} \le 2$, and so $\abs{g_i} \le 2$ for all $i \in [n]$. We have
 \begin{align}
     \frac{q_{+, n}^2}{q_{+, j}^2}
     = \frac{\paren{q_n - \eta g_n}^2/\norm{\mb q - \eta \mb g}{}^2 }{\paren{q_j - \eta g_j}^2/\norm{\mb q - \eta \mb g}{}^2}
     =  \frac{\paren{q_n - \eta g_n}^2}{\paren{q_j - \eta g_j}^2} .
 \end{align}

For any $j \in \mc I_0$,
\begin{align}
  \frac{q_{+, n}^2}{q_{+, j}^2}
  = \frac{\paren{q_n - \eta g_n}^2}{\eta^2 g_j^2}
  = q_n^2 \frac{\paren{1 - \eta g_n/q_n}^2 }{\eta^2 g_j^2}
  \ge \frac{\paren{1-2 \eta \sqrt{n} }^2}{4n\eta^2}.
\end{align}
Provided that $\eta \le 1/\paren{4\sqrt{n} }$, $1 - 2\eta \sqrt{n}\ge 1/2$, and so
\begin{align}
   \frac{\paren{1-2 \eta \sqrt{n} }^2}{4n\eta^2}
    \ge \frac{1}{16n\eta^2} \ge \frac{5}{2},
\end{align}
where the last inequality holds when $\eta \le 1/\sqrt{40n}$.\\

For any $j \in \mc I_1$,
\begin{align}
    \frac{q_{+, n}^2}{q_{+, j}^2}
    \ge \frac{q_n^2\paren{1 - \eta g_n/q_n}^2}{q_j^2 + \eta^2 g_j^2}
     \ge \frac{q_n^2\paren{1 - \eta g_n/q_n}^2}{q_n^2/3 + 4\eta^2}
     =  \frac{3\paren{1 - \eta g_n/q_n}^2}{1+12\eta^2/q_n^2}
    \ge  \frac{3\paren{1-2\eta \sqrt{n}}^2}{1+12n\eta^2} \ge \frac{5}{2},
\end{align}
where the very last inequality holds when $\eta \le 1/\paren{26\sqrt{n}}$.

Since $\mb q \in\mc{S}_{\zeta_0}^{(n+)} \setminus \mc{S}_1^{(n+)}$, $\mc I_2$ is nonempty. For any $j \in \mc I_2$,
\begin{align}
  \frac{q_{+, n}^2}{q_{+, j}^2}  = \frac{q_n^2}{q_j^2} \paren{1 + \eta \frac{g_j/q_j - g_n/q_n}{1 - \eta g_j/q_j}}^2.
\end{align}
 Since $g_j/q_j \le 2\sqrt{3n}$, $1 - \eta g_j/q_j \ge 1/2$ when $\eta \le 1/\paren{4\sqrt{3n}}$. Conditioned on this and due to that $g_j/q_j - g_n /q_n \ge 0$, it follows
\begin{align*}
   \paren{1 + \eta \frac{g_j/q_j - g_n/q_n}{1 - \eta g_j/q_j}}^2
   \le \brac{1+ 2\eta \paren{g_j/q_j - g_n/q_n}}^2
   \le \brac{1+ 2\eta \paren{2\sqrt{3n} + 2\sqrt{n}}}^2
   \le \paren{1+11\eta \sqrt{n} }^2
\end{align*}
If $q_n^2 /q_j^2 \le 2$, $q_{+, n}^2/q_{+, j}^2 \le 5/2$ provided that
\begin{align}
  \paren{1+11\eta \sqrt{n}}^2 \le \frac{5/2}{2} = \frac{5}{4} \Longleftarrow \eta \le \frac{1}{100\sqrt{n}}.
\end{align}
As $\mb q\notin S_{1}^{(n+)}$, we have $q_n^2/\norm{\mb
  q_{-n}}{\infty}^2\le 2$, so there must be a certain $j \in \mc I_2$
satisfying $q_n^2 /q_j^2 \le 2$. We conclude that when
\begin{align}
  \eta \le \min\set{\frac{1}{\sqrt{40n}}, \frac{1}{26\sqrt{n}}, \frac{1}{100\sqrt{n}}}
  = \frac{1}{100\sqrt{n}},
\end{align}
 the index of largest entries of $\mb q_{+, -n}$ remains in $\mc I_2$.

On the other hand,  when $\eta \le 1/\paren{100\sqrt{n}}$, for all $j \in \mc I_2$,
\begin{align}
  \paren{1 + \eta \frac{g_j/q_j - g_n/q_n}{1 - \eta g_j/q_j}}^2 \ge \brac{1+ \eta \paren{g_j/q_j - g_n/q_n}}^2 \ge \paren{1 + \frac{\eta}{4n} \theta \paren{1-\theta} \frac{\zeta_0}{1+\zeta_0}}^2,
\end{align}
where the last inequality is due to~\cref{equation:emp-grad-inward}.

So when $\eta = t/\paren{100\sqrt{n}}$ for any $t \in (0, 1)$,
\begin{align}
  \frac{q_{+, n}^2}{\norm{\mb q_{+, -n}}{\infty}^2}
  \ge \frac{q_n^2}{\norm{\mb q_{-n}}{\infty}^2} \paren{1 + t\frac{\theta \paren{1-\theta} \zeta_0}{400n^{3/2} \paren{1+\zeta_0}}  }^2,
\end{align}
completing the proof.
\end{proof}

\begin{proposition}  \label{prop:iter_stay_R2}
For any $\zeta_0\in(0,1)$, on the good events stated in~\cref{prop:grad_norm_concentrate} and~\cref{theorem:emp-grad-inward}, if the step sizes satisfy
\begin{align}
  \eta^{(k)} \le \min\set{\frac{1}{100\sqrt{n}}, \frac{1-\zeta_0}{9 \sqrt{n}}}~~\textrm{for all}~k,
\end{align}
the iteration sequence will stay in $\mc S_{\zeta_0}^{(n+)}$ provided that our initialization $\mb q^{(0)} \in \mc S_{\zeta_0}^{(n+)}$.
\end{proposition}
\begin{proof}
By~\cref{prop:progress_r1}, if the current iterate $\mb q \in \mc S_{\zeta_0}^{(n+)}\setminus \mc{S}_1^{(n+)}$, the next iterate $\mb q_+ \in \mc S_{\zeta_0}^{(n+)}$, provided that $\eta \le 1/\paren{100\sqrt{n}}$. Now if the current $\mb q \in \mc S_1^{(n+)}$, i.e., $q_n^2/q_j^2 \ge 2$ for all $j \in [n-1]$, we can emulate the analysis of the set $\mc I_1$ in proof of~\cref{prop:progress_r1}. Indeed, for any $j \in [n-1]$,
\begin{align}
  \frac{q_{+, n}^2}{q_{+, j}^2}
  \ge \frac{q_n^2\paren{1-\eta g_n/q_n}^2}{q_j^2 + \eta^2 g_j^2}
  \ge \frac{q_n^2\paren{1-2\eta\sqrt{n} }^2}{q_n^2/2 + 4\eta^2}
  \ge \frac{2\paren{1-2\eta\sqrt{n} }^2}{1 + 8n\eta^2},
\end{align}
which is greater than $1+\zeta_0$ provided that
$\eta \le \paren{1-\zeta_0}/\paren{9 \sqrt{n}}$. Combining the two cases finishes the proof.
\end{proof}

\subsection{Proof of~\cref{theorem:opt-one-basis}}
With the choice $\zeta_0 = \frac{1}{5\log n}$, when $\eta^{(k)} \le \frac{1}{100\sqrt{n}}$ and $\mb q^{(0)}\in\mc{S}_{\zeta_0}^{(n+)}$, the entire sequence $\set{\mb q^{(k)}}_{k\ge 0}$ will stay in $\mc S_{\zeta_0}^{(n+)}$ by~\cref{prop:iter_stay_R2}.

For any $\mb q$ and any $\mb v \in \partial_R f\paren{\mb q}$, we have $\innerprod{\mb v}{\mb q}=0$ and therefore
\begin{align}
  \norm{ \mb q - \eta \mb v}{}^2 =  \norm{\mb q}{}^2 + \eta^2 \norm{\mb v}{}^2 \ge 1.
\end{align}
So $\mb q - \eta \mb v$ is not inside $\bb B^n$. Since projection onto $\bb B^{n}$ is a contraction, we have
\begin{align}
  \norm{\mb q_+ - \mb e_n}{}^2
  & = \norm{\frac{\mb q - \eta \mb v}{\norm{\mb q - \eta \mb v}{}} - \mb e_n}{}^2 \le \norm{\mb q - \eta \mb v - \mb e_n}{}^2  \\
  & \le \norm{\mb q - \mb e_n}{}^2 + \eta^2 \norm{\mb v}{}^2 - 2 \eta \innerprod{\mb v}{\mb q - \mb e_n} \\
  & = \norm{\mb q - \mb e_n}{}^2 + \eta^2 \norm{\mb v}{}^2 - 2 \eta \innerprod{\mb v}{q_n\mb q - \mb e_n} \\
  & \le \norm{\mb q - \mb e_n}{}^2 + 4\eta^2 -  \frac{1}{8}\eta \theta\paren{1-\theta} n^{-3/2} \zeta_0 \norm{\mb q - \mb e_n}{} ,
\end{align}
where we have used the bounds in~\cref{prop:grad_norm_concentrate} and~\cref{theorem:emp-grad-inward} to obtain the last inequality. Further applying~\cref{prop:func_lipschitz}, we have
\begin{align}
  \norm{\mb q_+ - \mb e_n}{}^2 \le \norm{\mb q - \mb e_n}{}^2 + 4\eta^2 - \frac{1}{16} \eta \theta\paren{1-\theta}n^{-2} \zeta_0 \paren{f\paren{\mb q} - f\paren{\mb e_n}}.
\end{align}
Summing up the inequalities until step $K$, we have
\begin{align}
  0
  & \le \norm{\mb q^{(0)} - \mb e_n}{}^2 + 4\sum_{j=0}^{K}\paren{\eta^{(j)}}^2 - \frac{1}{16} \theta \paren{1-\theta} n^{-2} \zeta_0\sum_{j=0}^{K} \eta^{(j)}\paren{f\paren{\mb q^{(j)}} - f\paren{\mb e_n}} \\
  & \Longrightarrow \sum_{j=0}^{K} \eta^{(j)} \paren{f\paren{\mb q^{(j)}} - f\paren{\mb e_n}}
    \le \frac{16 \norm{\mb q^{(0)} - \mb e_n}{}^2 + 64\sum_{j=0}^{K} \paren{\eta^{(j)}}^2}{\theta \paren{1-\theta}n^{-2} \zeta_0} \\
   & \Longrightarrow f\paren{\mb q^{\mathrm{best}}} - f\paren{\mb e_n} \le \frac{16 \norm{\mb q^{(0)} - \mb e_n}{}^2 + 64\sum_{j=0}^{K} \paren{\eta^{(j)}}^2}{\theta \paren{1-\theta}n^{-2} \zeta_0 \sum_{j=0}^{K} \eta^{(j)}}.
\end{align}
Substituting the following estimates
\begin{align}
  \sum_{j=0}^{K} \paren{\eta^{(j)}}^2
  & \le \frac{1}{10^4n} \paren{1+ \int_{0}^{K} t^{-2\alpha}\; dt} \le  \frac{1}{10^4n} \frac{1}{1-2\alpha}\paren{K^{1-2\alpha} + 1}, \\
   \sum_{j=0}^{K} \eta^{(j)}
   & \ge \frac{1}{10^2\sqrt{n}} \int_{0}^{K} t^{-\alpha}\; dt
   \ge \frac{1}{10^2 \sqrt{n}}  \frac{K^{1-\alpha}}{1 - \alpha},
\end{align}
and noting $16\norm{\mb q^{(K)} - \mb e_n}{}^2 \le 32$, we have
\begin{align}
  f\paren{\mb q^{\mathrm{best}}} - f\paren{\mb e_n}
  \le \frac{3200n^{5/2} \paren{1-\alpha} + 16/25\cdot n^{3/2} \paren{ \frac{1-\alpha}{1-2\alpha} K^{1-2\alpha} + 1-\alpha}}{\theta \paren{1-\theta} \zeta_0 K^{1-\alpha}}.
\end{align}
Noting that
\begin{align}
  K \ge \paren{\frac{6400n^{5/2} \paren{1-\alpha}}{\theta \paren{1-\theta} \zeta_0 \eps}}^{\frac{1}{1-\alpha}} \Longrightarrow \frac{3200n^{5/2} \paren{1-\alpha}}{\theta \paren{1-\theta} \zeta_0 K^{1-\alpha} } \le \frac{\eps}{2},
\end{align}
and when $K \ge 1$, $K^{1-2\alpha} \ge 1$, yielding that
\begin{multline}
  K \ge \paren{\frac{64n^{3/2} \frac{1-\alpha}{1-2\alpha}}{25\eps \theta \paren{1-\theta} \zeta_0}}^{\frac{1}{\alpha}}
  \Longrightarrow
  \frac{32n^{3/2} \frac{1-\alpha}{1-2\alpha} K^{-\alpha}}{25\theta\paren{1-\theta} \zeta_0 } \le \frac{\eps}{2} \\
  \Longrightarrow
    \frac{16n^{3/2} \cdot 2 \cdot \paren{1-\alpha}\paren{\frac{1}{1-2\alpha} K^{1-2\alpha}}}{25\theta\paren{1-\theta} \zeta_0 K^{1-\alpha}} \le \frac{\eps}{2}
    \Longrightarrow
    \frac{16n^{3/2} \cdot \paren{1-\alpha}\paren{\frac{1}{1-2\alpha} K^{1-2\alpha} + 1}}{25\theta\paren{1-\theta} \zeta_0 K^{1-\alpha}} \le \frac{\eps}{2} .
\end{multline}
So we conclude that when
\begin{align}
  \label{equation:K-bound}
  K \ge \max\set{\paren{\frac{6400n^{5/2} \paren{1-\alpha}}{\theta \paren{1-\theta} \zeta_0 \eps} }^{\frac{1}{1-\alpha}}, \paren{\frac{64n^{3/2} \frac{1-\alpha}{1-2\alpha}}{25\eps \theta \paren{1-\theta} \zeta_0}}^{\frac{1}{\alpha}}},
\end{align}
$f\paren{\mb q^{\mathrm{best}}} - f\paren{\mb e_n} \le \eps$. When this happens, by~\cref{prop:func_sharpness},
\begin{align}
  \norm{\mb q^{\mathrm{best}} - \mb e_n}{}
  \le \frac{16}{\theta\paren{1-\theta}} \eps.
\end{align}
Plugging in the choice $\zeta_0=1/(5\log n)$ in~\cref{equation:K-bound} gives the desired bound on the number of iterations.

%% file: Sections/proof-recovery.tex
\section{Proofs for~\cref{section:recovery}}
\subsection{Proof of~\cref{lemma:random-init-works}}
\begin{lemma} \label{lemma:vol_measure}
  For all $n \ge 3$ and $\zeta \ge 0$, it holds that
  \begin{align}
    \frac{\mathrm{vol}\paren{\mc S_{\zeta}^{(n+)}}}{\mathrm{vol}\paren{\bb S^{n-1}}} \ge \frac{1}{2n}-\frac{9}{8} \frac{\log n}{n} \zeta.
  \end{align}
\end{lemma}
We note that a similar result appears
in~\citep{GilboaEtAl2018Efficient}, but our definition of the region
$\mc{S}_{\zeta}$ is slightly different. For completeness we provide a
proof in~\cref{lemma:vol_measure_app}.

We now prove~\cref{lemma:random-init-works}. Taking $\zeta=1/(5\log n)$ in~\cref{lemma:vol_measure}, we obtain
\begin{equation}
  \frac{\mathrm{vol}\paren{\mc S_{1/(5\log n)}^{(n+)}}}{\mathrm{vol}\paren{\bb S^{n-1}}} \ge \frac{1}{2n}-\frac{9}{8} \frac{\log n}{n} \cdot \frac{1}{5\log n} \ge \frac{1}{4n}.
\end{equation}
By symmetry, all the $2n$ sets $\set{\mc{S}_{1/(5\log n)}^{(i+)}, \mc{S}_{1/(5\log n)}^{(i-)}:i\in[n]}$ have the same volume which is at least $1/(4n)$. As $\mb q^{(0)}\sim{\rm Uniform}(\bb S^{n-1})$, it falls into their union with probability at least $2n\cdot 1/(4n)=1/2$, on which it belongs to a uniformly random one of these $2n$ sets.

\subsection{Proof of~\cref{theorem:opt-all-bases}}
With $\zeta_0$ set to $1/\paren{5\log n}$, the good events in~\cref{prop:grad_norm_concentrate}, ~\cref{theorem:emp-grad-inward}, and~\cref{prop:func_sharpness} happen with probability at least
\begin{multline}
  1 - \exp(-cm\theta^3\zeta_0^2n^{-3}\log^{-1}m) - \exp(-c'm\theta\log^{-1}m) - \exp\paren{-c'' m \theta^3 \log^{-1} m }\\
  \ge 1 - \exp(-c_0m\theta^3n^{-3}\log^{-3} m).
\end{multline}
Moreover, by~\cref{lemma:random-init-works}, random initialization will fall in these $2n$ sets with probability at least $1/2$. When it falls in one of these $2n$ sets, by~\cref{theorem:opt-one-basis}, one run of the algorithm will find a signed standard basis vector up to $\eps$ accuracy. With $R$ independent runs,  at least $S \doteq \frac{1}{4}R$ of them are effective with probability at least $1-\exp\paren{-R/8}$, due to Bernstein's inequality.  After these effective runs, the probability any standard basis vector is missed (up to sign) is bounded by
\begin{align}
  n \paren{1-\frac{1}{n}}^{S} \le \exp\paren{-\frac{S}{n} + \log n}\le \exp\paren{-\frac{S}{2n}},
\end{align}
where the second inequality holds whenever $S \ge 2n \log n$.

For the iteration complexity, since the statement is about distance to the optimizers, we set $\eps$ in~\cref{theorem:opt-one-basis} as $\theta \paren{1-\theta}\eps/16$ to obtain the claimed bound. Finally, noting that
\begin{align}
\frac{3}{100} \cdot \frac{16}{\theta \paren{1-\theta}} \ge \frac{12}{25} \cdot \frac{1}{1/4} \ge 1
\end{align}
completes the proof.

%% file: Sections/aux.tex
\section{Auxiliary calculations}

\begin{lemma} \label{lemma:subgauss_para}
For $x \sim \bgt$, $\norm{x}{\psi_2} \le C_a$. For any vector $\mb u \in \R^n$ and $\mb x \sim_{iid} \bgt$, $\norm{\mb x^\top \mb u}{\psi_2} \le C_b\norm{\mb u}{}$. Here $C_a, C_b \ge 0$ are universal constants.
\end{lemma}
\begin{proof}
  For any $\lambda \in \R$, $x\sim\bgt$ and $z\sim\normal(0,1)$,
  \begin{align}
    \expect{\exp\paren{\lambda x}} = 1-\theta + \theta\expect{\exp\paren{\lambda z}} = 1-\theta + \theta\exp(\lambda^2/2) \le \exp(\lambda^2/2).
  \end{align}
  So $\norm{x}{\psi_2}$ is bounded by a universal constant. Moreover,
  \begin{align}
    \norm{\mb u^\top \mb x}{\psi_2}
    = \norm{\sum_{i} u_i x_i}{\psi_2}
    \le C_1\paren{\sum_i u_i^2 \norm{x_i}{\psi_2}^2}^{1/2}
    \le C_2 \norm{\mb u}{},
  \end{align}
  as claimed.
\end{proof}

\begin{lemma} \label{lemma:vol_measure_app}
  For all $n \ge 3$ and $\zeta \ge 0$, it holds that
  \begin{align}
    \frac{\mathrm{vol}\paren{\mc S_{\zeta}^{(n+)}}}{\mathrm{vol}\paren{\bb S^{n-1}}} \ge \frac{1}{2n}-\frac{9}{8} \frac{\log n}{n} \zeta.
  \end{align}
\end{lemma}
\begin{proof}
  We have
  \begin{align}
    \frac{\mathrm{vol}\paren{\mc S_{\zeta}^{(n+)}}}{\mathrm{vol}\paren{\bb S^{n-1}}}
    & = \bb P_{\mb q \sim \mathrm{uniform}\paren{\bb S^{n-1}}} \brac{q_n^2 \ge \paren{1+\zeta} \norm{\mb q_{-n}}{\infty}^2, q_n \ge 0}\\
   & =  \bb P_{\mb x \sim \normal\paren{\mb 0, \mb I_n}} \brac{x_n \ge 0, x_n^2 \ge \paren{1+\zeta} x_i^2 \; \forall\; i \ne n} \\
   & = \paren{2\pi}^{n/2} \int_{0}^{\infty} e^{-x_n^2/2} \paren{ \prod_{j=1}^{n-1} \int_{-x_n/\sqrt{1+\zeta}}^{x_n/\sqrt{1+\zeta}} e^{-x_j^2/2} \; dx_j }\; dx_n \\
   & = \paren{2\pi}^{1/2} \int_{0}^{\infty} e^{-x_n^2/2} \psi^{n-1}\paren{x_n/\sqrt{1 + \zeta}} \; dx_n \\
   & =  \frac{\sqrt{1+\zeta}}{\sqrt{2\pi}}\int_{0}^{\infty} e^{-\paren{1+\zeta}x^2/2} \psi^{n-1}\paren{x} \; dx \doteq \hbar\paren{\zeta} > 0,
  \end{align}
  where we write $\psi(t) \doteq \frac{1}{\sqrt{2\pi}} \int_{-t}^t \exp\paren{-s^2/2}\; ds$.

  Now we derive a lower bound of the volume ratio by considering a first-order Taylor expansion of $\hbar$ around $\zeta = 0$ (as we are mostly interested in small $\zeta$). By symmetry,  $\hbar\paren{0} = 1/\paren{2n}$. Moreover,
  \begin{align}
    \left.\frac{\partial\hbar(\zeta)}{\partial\zeta}\right|_{\zeta = 0}
    & = \frac{1}{2} \frac{1}{\sqrt{2\pi}} \int_{0}^{\infty} e^{-x^2/2} \psi^{n-1}\paren{x} \; dx - \frac{1}{\sqrt{2\pi}} \int_{0}^{\infty} e^{-x^2/2} x^2\psi^{n-1}\paren{x} \; dx \\
    & = \frac{1}{4n} - \frac{1}{2\sqrt{2\pi}} \int_{0}^{\infty} e^{-x^2/2} x^2\psi^{n-1}\paren{x} \; dx.
  \end{align}
Now we provide an upper bound for the second term of the last equation. Note that
\begin{align}
  \frac{1}{\sqrt{2\pi}} \int_{0}^{\infty} e^{-x^2/2} x^2\psi^{n-1}\paren{x} \; dx
  & = \bb E_{\mb x \sim \normal\paren{\mb 0, \mb I_n}} \brac{x_n^2 \indicator{x_n^2 \ge \norm{\mb x_{-n}}{\infty}^2} \indicator{x_n \ge 0}} \\
  & = \frac{1}{2n} \bb E_{\mb x \sim \normal\paren{\mb 0, \mb I_n}} \norm{\mb x}{\infty}^2.
\end{align}
Now for any $\lambda \in (0, 1/2)$,
\begin{align}
  \exp\paren{\lambda \bb E  \norm{\mb x}{\infty}^2 }
  \le \bb E \exp\paren{\lambda \norm{\mb x}{\infty}^2}
  \le \sum_{j=1}^n \bb E \exp\paren{\lambda x_j^2}
  = n \bb E_{x \sim \normal(0, 1)}\exp\paren{\lambda x^2} \le \frac{n}{\sqrt{1-2\lambda}}.
\end{align}
Taking logarithm on both sides, rearranging the terms, and setting $\lambda = 1/4$, we obtain
\begin{align}
  \bb E \norm{\mb x}{\infty}^2 \le \inf_{\lambda \in (0, 1/2)} \frac{\log n + \frac{1}{2} \log\paren{1-2\lambda}^{-1}}{\lambda} \le 4 \log n + 2\log 2.
\end{align}
So
\begin{align}
  \left. \frac{\partial\hbar \paren{\zeta}}{\partial\zeta}\right|_{\zeta = 0} \ge \frac{1}{4n} - \frac{1}{4n}\paren{4\log n + 2\log 2} \ge -\frac{9}{8} \frac{\log n}{n},
\end{align}
provided that $n \ge 3$.

Now we show that $\hbar\paren{\zeta} \ge \hbar\paren{0} + \hbar'(0) \zeta$ by showing that $\hbar''(\zeta) \ge 0$. We have
\begin{align}
  \frac{\partial^2 \hbar\paren{\zeta}}{\partial \zeta^2}
  = \frac{\sqrt{1+\zeta}}{4\sqrt{2\pi}} \int_0^{\infty} \brac{x^4 - \frac{2x^2}{1+\zeta} - \frac{1}{\paren{1+\zeta}^2}}  e^{-\frac{1+\zeta}{2}x^2} \psi^{n-1}\paren{x}\; dx.
\end{align}
Using integration by part, we have
\begin{align}
  & \int_0^{\infty} \brac{x^4 -\frac{3x^2}{1+\zeta} } e^{-\frac{1+\zeta}{2}x^2} \psi^{n-1}\paren{x}\; dx \\
  =\; & \left. -\frac{1}{1+\zeta} e^{-\frac{1+\zeta}{2} x^2} x^3 \cdot \psi^{n-1}\paren{x} \right|_{0}^{\infty} + \int_0^{\infty} \frac{1}{1+\zeta} e^{-\frac{1+\zeta}{2} x^2} x^3 \paren{n-1} \psi^{n-2}\paren{x} \sqrt{\frac{2}{\pi}} e^{-\frac{x^2}{2}}\; dx \\
  =\; & \int_0^{\infty} \frac{1}{1+\zeta} e^{-\frac{1+\zeta}{2} x^2} x^3 \paren{n-1} \psi^{n-2}\paren{x} \sqrt{\frac{2}{\pi}} e^{-\frac{x^2}{2}}\; dx \ge 0,
\end{align}
and similarly
\begin{align}
  & \int_0^{\infty} \brac{x^2 -\frac{1}{1+\zeta} } e^{-\frac{1+\zeta}{2}x^2} \psi^{n-1}\paren{x}\; dx \\
  =\; & \left. -\frac{1}{1+\zeta} e^{-\frac{1+\zeta}{2} x^2} x \cdot \psi^{n-1}\paren{x} \right|_{0}^{\infty} + \int_0^{\infty} \frac{1}{1+\zeta} e^{-\frac{1+\zeta}{2} x^2} x \paren{n-1} \psi^{n-2}\paren{x} \sqrt{\frac{2}{\pi}} e^{-\frac{x^2}{2}}\; dx \\
  =\; & \int_0^{\infty} \frac{1}{1+\zeta} e^{-\frac{1+\zeta}{2} x^2} x \paren{n-1} \psi^{n-2}\paren{x} \sqrt{\frac{2}{\pi}} e^{-\frac{x^2}{2}}\; dx \ge 0.
\end{align}
Noting that
\begin{align}
x^4 - \frac{2x^2}{1+\zeta} - \frac{1}{\paren{1+\zeta}^2}
 = x^4 - \frac{3x^2}{1+\zeta} + \frac{1}{1+\zeta} \paren{ x^2 - \frac{1}{1+\zeta}}
\end{align}
and combining the above integral results, we conclude that $\hbar''\paren{\zeta} \ge 0$ and complete the proof.
\end{proof}

\begin{lemma}
  \label{lemma:two-dim-angle}
  Let $(x_1, y_1),(x_2,y_2)\in\R_{>0}^2$ be two points in the first quadrant satisfying $y_1\ge x_1$ and $y_2\ge x_2$, and $\frac{y_2/x_2}{y_1/x_1}\in[1, 1+\eta]$ for some $\eta\le 1$, then we have $\angle\paren{(x_1,y_1), (x_2,y_2)} \le \eta$.
\end{lemma}
\begin{proof}
  For $i=1,2$, let $\theta_i$ be the angle between the ray $(x_i,y_i)$ and the $x$-axis. Our assumption implies that $\theta_i\in[\pi/4, \pi/2)$ and $\theta_2\ge \theta_1$, and thus $\angle\paren{(x_1,y_1),(x_2,y_2)}=\theta_2-\theta_1$. We have
\begin{align*}
    \tan\angle\paren{(x_1,y_1), (x_2,y_2)} = \frac{\tan\theta_2 - \tan\theta_1}{1 + \tan\theta_2\tan\theta_1} = \frac{y_2/x_2 - y_1/x_1}{1 + y_2y_1/(x_2x_1)} = \frac{\frac{y_2/x_2}{y_1/x_1} - 1}{y_2/x_2 + x_1/y_1} \le \frac{y_2/x_2}{y_1/x_1} - 1 \le \eta.
\end{align*}
  Thus, $\angle\paren{(x_1,y_1), (x_2,y_2)} \le \arctan(\eta)\le \eta$, completing the proof.
\end{proof}

\begin{lemma}
  \label{lemma:expected-gradient-diff}
  For any admissible pair $\mb p,\mb q\in\bb S^{n-1}$ with $\dexp(\mb p,\mb q)\le\eps < 1$, we have for all $\mb u\in\bb S^{n-1}$ that
  \begin{equation}
    \mathbb{E}_{\mb x\sim\bgt}\left[|\mb u^\top\mb x|\indicator{\sign(\mb p^\top x)\neq \sign(\mb q^\top\mb x)}\right] \le 3\eps\sqrt{\log\frac{1}{\eps}}.
  \end{equation}
\end{lemma}
\begin{proof}
  Fix some threshold $t>0$ to be determined. We have
  \begin{align}
    & \quad \expect{|\mb u^\top\mb x|\indicator{\sign(\mb p^\top x)\neq \sign(\mb q^\top\mb x)}} \\
    & \le \expect{|\mb u^\top\mb x|\indicator{|\mb u^\top\mb x| > t}} + \expect{|\mb u^\top\mb x|\indicator{|\mb u^\top\mb x| \le t,\sign(\mb p^\top\mb x)\neq \sign(\mb q^\top\mb x)}} \\
    & \le \paren{\expect{(\mb u^\top\mb x)^2} \cdot \prob{|\mb u^\top
      \mb x|>t} }^{1/2} + t\expect{\indicator{\sign(\mb p^\top\mb x) \neq \sign(\mb q^\top\mb x)}} \\
    & \le \paren{\theta \cdot 2\exp(-t^2/2)}^{1/2} + \eps t, \label{aux:pq_diff_last_bound}
  \end{align}
  where at the last inequality we use that
  \begin{align}
    \bb E\indicator{\abs{\mb u^\top \mb x} > t}
     = \bb E_\Omega \bb E_{\mb z \sim \normal\paren{\mb 0, \mb I}} \indicator{\abs{\mb u_\Omega^\top \mb z} > t}
     \le \bb E_\Omega\;  2\exp\paren{-\frac{t^2}{2\norm{\mb u_\Omega}{}^2}}
     \le 2\exp\paren{-\frac{t^2}{2}}.
  \end{align}
  Taking $t=\sqrt{2\log \eps^{-2}}$,
the bound in~\cref{aux:pq_diff_last_bound} simplifies to
  \begin{equation}
    \sqrt{2\theta \exp\paren{-\log \eps^{-2}}} +
    \eps\sqrt{2\log \eps^{-2}}
    \le \sqrt{2 \theta}\eps + 2\eps \sqrt{\log \eps^{-1}}
    \le 3\eps \sqrt{\log \eps^{-1}},
  \end{equation}
  where we have used $\theta\le 1/2$.
\end{proof}